\documentclass[10pt]{article}
\RequirePackage{natbib}
\usepackage[top=1in, bottom=1in, left=1in, right=1in]{geometry}

\usepackage[utf8]{inputenc} 
\usepackage[T1]{fontenc}    
\usepackage{hyperref}       
\usepackage{url}            
\usepackage{booktabs}       
\usepackage{amsfonts}       
\usepackage{nicefrac}       
\usepackage{microtype}      
\usepackage{xcolor}         

\usepackage{amsmath,amsthm}
\usepackage{amssymb}
\usepackage{bbm}
\usepackage{float}
\usepackage{xcolor}
\usepackage{enumitem}
\usepackage{booktabs}
\usepackage{etoolbox}
\usepackage{clipboard}
\usepackage{mathtools}
\usepackage{etoolbox}
\usepackage{upgreek}

\newtheorem{theorem}{Theorem}
\newtheorem{lemma}[theorem]{Lemma} 
\newtheorem{proposition}[theorem]{Proposition} 

\newtheorem{corollary}[theorem]{Corollary}
\newtheorem{definition}[theorem]{Definition}

\newtheorem{assumption}[theorem]{Assumption}

\newcommand{\R}{\mathbb{R}}
\newcommand{\N}{\mathbb{N}}
\newcommand{\E}{\mathbb{E}}

\newcommand{\hilb}{\mathcal{H}}
\newcommand{\meas}{\mathcal{P}}

\newcommand{\prob}{\mathbb{P}}

\newcommand{\var}{\operatorname{Var}}

\newcommand{\law}{\operatorname{law}}

\newcommand{\supp}{\operatorname{supp}}
\newcommand{\diam}{\operatorname{diam}}

\newcommand{\cov}{\operatorname{Cov}}
\newcommand{\tr}{\operatorname{Tr}}
\newcommand{\dsm}{{\operatorname{dsm}}}
\newcommand{\sm}{{\operatorname{sm}}}
\newcommand{\nud}{\nu_{\operatorname{data}}}

\newcommand{\emm}{{\operatorname{em}}}
\newcommand{\erm}{{\operatorname{erm}}}
\newcommand{\stab}{\varepsilon_{\text{stab}}}

\newcommand{\tw}{\tau}

\newcommand{\argmin}{{\operatorname{argmin}}}
\newcommand{\clip}{{\operatorname{Clip}}}

\providetoggle{editmode}
\settoggle{editmode}{false}

\iftoggle{editmode}{\setcounter{tocdepth}{3}}{\setcounter{tocdepth}{2}}

\newcommand{\copyenv}[3]{\begin{#1}\label{#2}\Copy{#2}{#3}\end{#1}}
\newcommand{\pasteenv}[2]{\begin{custom#1}{\ref{#2}}\Paste{#2}\end{custom#1}}

\title{Implicit Regularisation in Diffusion Models:\\An Algorithm-Dependent Generalisation Analysis}

\usepackage{authblk}
\author{Tyler Farghly\thanks{farghly@stats.ox.ac.uk}}
\author{Patrick Rebeschini\thanks{rebeschini@stats.ox.ac.uk}}
\author{George Deligiannidis\thanks{deligian@stats.ox.ac.uk}}
\author{Arnaud Doucet\thanks{arnauddoucet@google.com}}

\affil{Department of Statistics, University of Oxford}

\begin{document}

\maketitle

\begin{abstract}
The success of denoising diffusion models raises important questions regarding their generalisation behaviour, particularly in high-dimensional settings. Notably, it has been shown that when training and sampling are performed perfectly, these models memorise training data---implying that some form of regularisation is essential for generalisation. Existing theoretical analyses primarily rely on algorithm-independent techniques such as uniform convergence, heavily utilising model structure to obtain generalisation bounds. In this work, we instead leverage the algorithmic aspects that promote generalisation in diffusion models, developing a general theory of algorithm-dependent generalisation for this setting. Borrowing from the framework of algorithmic stability, we introduce the notion of score stability, which quantifies the sensitivity of score-matching algorithms to dataset perturbations. We derive generalisation bounds in terms of score stability, and apply our framework to several fundamental learning settings, identifying sources of regularisation. In particular, we consider denoising score matching with early stopping (denoising regularisation), sampler-wide coarse discretisation (sampler regularisation) and optimising with SGD (optimisation regularisation). By grounding our analysis in algorithmic properties rather than model structure, we identify multiple sources of implicit regularisation unique to diffusion models that have so far been over-looked in the literature.
\end{abstract}

\section{Introduction}

Diffusion models \citep{sohl2015deep,Ho2020-pq,Song2021-uf} are a class of generative models that have achieved state-of-the-art performance across image, audio, video, and protein synthesis tasks \citep{Rombach2022-kn, Saharia2022-zc, Ramesh2022-ac,watson2023novo,esser2024scaling}. Their ability to generate high-quality samples from complex, high-dimensional distributions with limited data motivates the need for a theoretical understanding of the mechanisms underpinning their strong generalisation capabilities.

The goal of diffusion models is to generate new synthetic samples from a data distribution \(\nud\) using a finite set of $N$ data points \(\{x_i\}_{i=1}^N\). Central to the methodology is a unique approach to generating data, formulating it as the iterative transformation of noise to data or, equivalently, the reversal of a diffusion process \citep{Song2021-uf}. This diffusion process, called the \textit{forward process}, is defined by the stochastic differential equation (SDE),
\begin{equation}\label{eq:forward_sde}
    dX_t = - \alpha X_t \, dt + \sqrt{2} \, dW_t, \qquad X_0 \sim \nud, \qquad t \in [0, T],
\end{equation}
for some \(\alpha \geq 0\), where \(W_t\) denotes the Brownian motion in \(\R^d\) and \(T > 0\) is the terminal time. It can then be shown that the time-reversal of this process, \(Y_t := X_{T-t}\) admits a weak formulation as a solution to the SDE,
\begin{equation}\label{eq:backwards_sde}
    dY_t = \alpha Y_t dt + 2 \nabla \log p_{T-t}(Y_t)dt + \sqrt{2} dW_t, \qquad Y_0 \sim p_T, \qquad t \in [0, T),
\end{equation}
where \(p_t\) denotes the marginal density of \(X_t\) \citep{Haussmann1986-hw}. Therefore, simulating samples from $\nud=p_0$ can be achieved by solving the diffusion process in \eqref{eq:backwards_sde}, which requires an approximation of the score function, \(\nabla \log p_t\). 
%
This is achieved by fitting a time-dependent deep neural network to minimise a weighted \(L^2\) distance called the \textit{(population) score matching loss}:
\begin{equation}\label{eq:sm}
    \ell_{\sm}(s; \tau) := \int \E_{X_t} [\|s(X_t, t) - \nabla \log p_t(X_t)\|^2] \, \tw(dt),
\end{equation}
where \(\tau\) is a probability measure over \((0, T]\) that determines the weighting of the timepoints. Since \(\nabla \log p_t\) is unknown, typically the (population) \textit{denoising score matching loss} \(\ell_{\dsm}\), which differs from \(\ell_\sm(s)\) only by a constant, is used instead (see Section \ref{sec:prelim_overfitting} for definition) and is then approximated using the dataset, forming the \textit{empirical denoising score matching loss} \(\hat{\ell}_{\dsm}\). The score network $s(x, t)$ is trained on this objective using standard stochastic optimisation methods relying on mini-batching. Once an approximation is obtained, samples are generated by numerically solving the reverse-time SDE \eqref{eq:backwards_sde}. Both \textit{score matching} and \textit{backwards sampling} introduce distinct challenges and design choices that impact the quality of model output \citep{Karras2022-us}.

\begin{figure}\label{fig:exp}
\centering
\includegraphics[width=\textwidth]{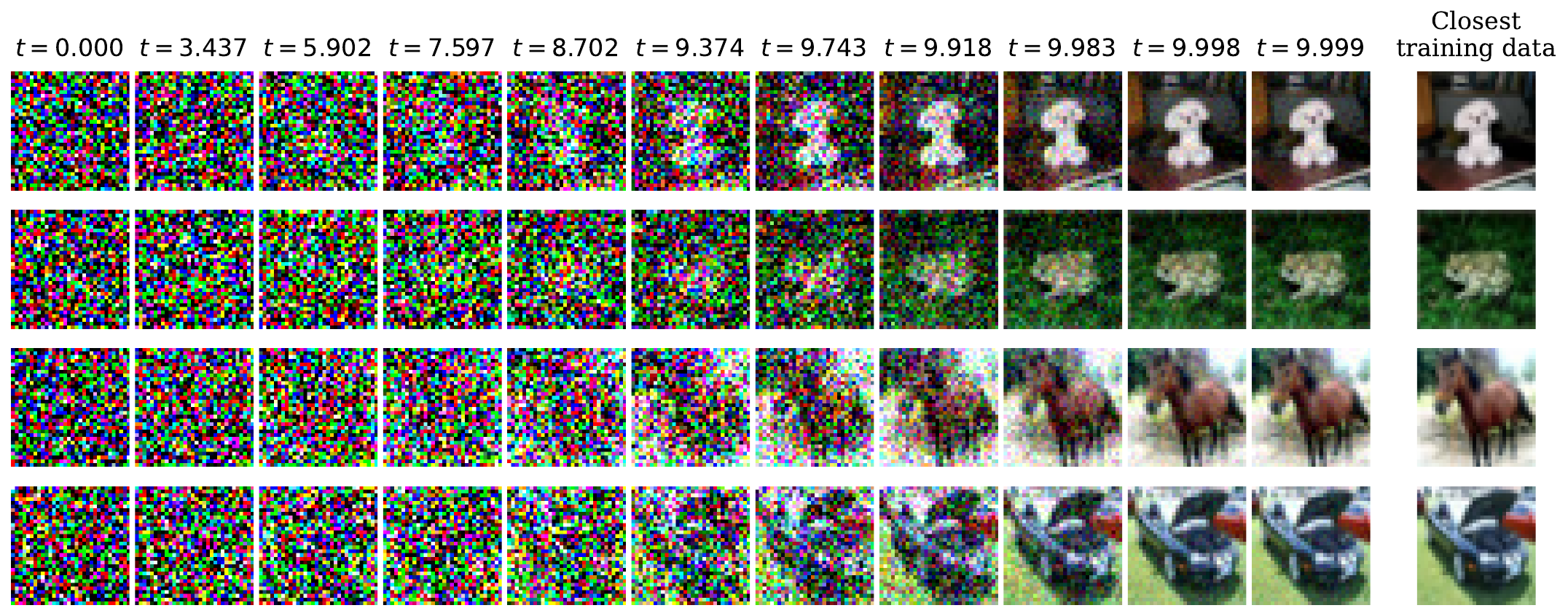}
\caption{Samples generated using the empirical score function on CIFAR-10 compared to the closest image in the dataset, illustrating memorisation of the training data.}
\end{figure}

Score matching presents a key difference from standard supervised learning. In the space of all \(L^2\) score functions, the empirical objective \(\hat{\ell}_\dsm\) possesses a unique minimiser—the empirical score function—as a result of the integration over \(X_t|X_0\) (see Lemma \ref{lem:dsm_sm_emp}). This contrasts with traditional supervised learning where the empirical risk minimisation problem can have infinitely many solutions (for example, overparameterised regression) and regularisation is often required to ensure well-posedness and enable generalisation. As shown in Figure~\ref{fig:exp}, sampling with the empirical score can lead to exact recovery of the training data, as formalised in \cite{Pidstrigach2022-qr}.
This behaviour differs from that of interpolating models in supervised learning, where overfitting does not necessarily prevent generalisation—a phenomenon known as benign overfitting \citep{Bartlett2021-tf, Zhang2021-ly}, which plays a central role in explaining generalisation in deep learning.
This divergence from traditional learning problems suggests that existing theory for deep learning, and supervised learning more broadly, may not be sufficient for understanding the unique success of diffusion models, and highlights the need for new frameworks specifically targetting this setting.

Recently, there has been a drive towards developing theory for better understanding the unique structure of diffusion models. The most developed subset of this work focuses on connecting sample quality to score matching by deriving upper bounds on distribution error (e.g. KL divergence, total variation, or Wasserstein distance) between model samples and the data distribution, controlling it by the population score matching loss \citep{De_Bortoli2021-vi, De-Bortoli2022-tv,Lee2022-yu,Chen2022-wa, Benton2023-ov, Potaptchik2024-ue}. These results, often referred to as \textit{convergence bounds}, typically take the form
\begin{equation*}
    \text{Distribution error} \lesssim \ell_\sm(s) + \Delta,
\end{equation*}
where \(\Delta\) is the discretisation error of the sampling scheme, that can be made small with sufficiently fine discretisation. However, since \(\ell_\sm\) is not computable, these bounds say little about performance under empirical guarantees—that is, their \textit{generalisation} properties.
One line of work, initiated by \cite{Oko2023-sa} and extended in \citep{Azangulov2024-uf, Tang2024-qk}, applies classical uniform convergence theory to bound the generalisation gap from the decomposition,
\begin{equation}\label{eq:sm_decomp}
    \ell_\sm(s) = \hat{\ell}_\sm(s) + \underbrace{\ell_\sm(s) - \hat{\ell}_\sm(s)}_{\text{generalisation gap}},
\end{equation}
where \(\hat{\ell}_\sm\) denotes the empirical counterpart to \(\ell_\sm\). These results rely on covering number bounds for specific classes of neural networks and, while informative, they are limited to carefully chosen model classes and do not account for algorithmic properties.
An alternative approach by \citet{De-Bortoli2022-tv} uses a decomposition of the Wasserstein distance that leverages convergence properties of the empirical measure. Though more model-agnostic, this method overlooks how diffusion models generate novel data. Both lines of work are fundamentally algorithm-independent, in that they lack any utilisation of the algorithmic aspects that uniquely define diffusion models.
Recent efforts aim to incorporate algorithmic effects by restricting the problem. For instance, \citet{Shah2023-ka, Chen2024-pf} consider Gaussian mixture targets, while \citet{Li2023-cf, Yang2022-hq} study random feature models. These settings allow for finer analysis of the role of the score matching algorithm, but remain limited in scope, leaving open the challenge of developing a more general algorithm-dependent theory of generalisation in diffusion models.



As noted earlier, if the empirical score matching loss was completely minimised and sampling was performed perfectly, the diffusion model would simply return training data, failing to generalise. Therefore, the observed success of diffusion models in producing novel data implies that, in practice, they either avoid completely minimising \(\hat{\ell}_\sm\) or must avoid perfectly sampling. This suggests that (implicit) regularisation in the score matching or sampling algorithm is crucial for generalisation, making algorithmic considerations crucial for understanding diffusion models. 



\subsection{Our contributions}
We present here a general-purpose algorithm-dependent framework for analysing diffusion model generalisation. We introduce score stability, based on the classical approach of algorithmic stability \citep{Devroye1979-du,Kearns1999-jq,Bousquet2002-ox}, which quantifies a score matching algorithm's dependence on individual training examples. Under score stability guarantees, we derive expected generalisation gap bounds for the score matching and denoising score matching losses (Theorem \ref{thm:dsm_stability}). Using the score stability framework, we then analyse several examples of score matching algorithms, identifying several sources of implicit regularisation in diffusion model training and sampling. By performing an algorithm-dependent analysis and only using basic properties of the model class, our analysis disentangles and identifies three sources of implicit regularisation in the setting of diffusion models: noising regularisation, sampler regularisation, and optimisation-induced regularisation.



\paragraph{Denoising regularisation} To begin with, we consider the empirical risk minimisation algorithm (ERM) that selects the score function that minimises \(\hat{\ell}_\dsm\) over some hypothesis class \(\mathcal{H}\). Through a score stability analysis, we reveal a regularisation source within this objective when early stopping of the forward process is used—a standard practice in the diffusion model literature. Utilising properties of the noising forward process, we obtain stability guarantees of its minimiser, irrespective of the hypothesis class chosen. In particular, we obtain generalisation gap bounds with near-linear rate, \(\epsilon^{-d^*/4} (\epsilon^{-d^*/2} N^{-2} + \min_{\mathcal{H}} \hat{\ell}_\sm)^{c/2}\) for any \(c < 1\), where \(\epsilon > 0\) is the early stopping constant and \(d^*\) is the dimension of the data manifold.

\paragraph{Sampler regularisation} We then apply this analysis to discrete-time sampling algorithms, deriving statistical guarantees for the expected KL divergence between the true data distribution and samples generated by the diffusion model. The bound we derive is formed of two stages: we obtain generic rates \(\epsilon^{-1/2} (\epsilon^{-d^*/2} N^{-2} + \min_{\mathcal{H}} \hat{\ell}_\sm)^{c/d^*}\) but when \(N^{-2}\) and \(\min_{\mathcal{H}} \hat{\ell}_\sm\) are sufficiently small relative to \(\epsilon\), we obtain bounds with rates \(\epsilon^{-d^*/4} (\epsilon^{-d^*/2} N^{-2} + \min_{\mathcal{H}} \hat{\ell}_\sm)^{c/2}\) that are faster in \(N\) and \(\min_{\mathcal{H}} \hat{\ell}_\sm\). To derive this bound, we utilise regularisation brought about by the coarseness of the discretisation. We find that by increasing discretisation coarseness, we can improve the generalisation gap bound at the expense of worsening the discretisation error term.




\paragraph{Optimisation regularisation} Next, we consider the role of the optimisation schemes often employed to learn the score approximation. As an illustrative example, we consider stochastic gradient descent (SGD) applied to the denoising score matching objective. Furthermore, we employ gradient clipping and weight decay—two techniques that are frequently used in practice when training diffusion models. On the model class, we assume only structural assumptions typical in the optimisation literature, including non-global Lipschitz and smoothness assumptions.
Under the assumption that the learning rate decays sufficiently quickly, we obtain generalisation gap bounds with rate \(K^{\frac{\bar{\eta}\upsilon}{\bar{\eta}\upsilon+1}} N^{-1}\) where \(K\) is the number of iterations, \(c\) depends on the step-size and \(\upsilon = \mathcal{O}(d \overline{M} + \overline{L}^2)\), where \(\overline{L}\) and \(\overline{M}\) are the average Lipschitz and smoothness constants. Then, we more closely inspect the impact of the high variance gradient estimator employed when training diffusion models, investigating the interplay between the optimisation scheme and gradient noise. We consider a modification to the previous scheme that uses a second-order Gaussian approximation of the gradient noise. Using this additional noise, we identify a contractive behaviour in the training dynamics that we harness to obtain stability bounds that do not grow with number of iterations (see Proposition \ref{prop:time_indep_bounds}). Through this analysis, we show that the heightened noise present in diffusion models training dynamics enables tighter generalisation guarantees.

\subsection{Related works}

\citet{De-Bortoli2022-tv} provides guarantees in Wasserstein distance under guarantees on the empirical score matching loss using a uniform bound between the empirical distribution of the training set and the underlying data distribution. In the work of \cite{Oko2023-sa}, the authors obtain statistical rates for an algorithm that chooses from a class of ReLU networks using training data as well as knowledge about the smoothness of the target distribution. They obtain bounds on the expected population total variation that decays with rate \(\tilde{\mathcal{O}}(N^{-\frac{s}{2s + d}})\) as well as bounds in 1-Wasserstein distance. Here, \(s\) captures the smoothness of the target. \cite{Tang2024-qk} and \cite{Azangulov2024-uf} extend the Wasserstein distance analysis to the setting of the manifold hypothesis obtaining bounds that are faster when the target is supported on a low dimensional submanifold. This line of work utilises uniform convergence bounds on a class of feed-forward neural networks that is carefully chosen to regularise for target smoothness. In the works of \cite{Li2023-cf} and \cite{Yang2022-hq} a similar approach is used in the setting of random feature models. They use an algorithm-dependent notion of Rademacher complexity, deriving bounds for gradient descent. However, this type of analysis fundamentally relies on the linear structure of the model class and thus is restricted to the setting of linear models.

Recently, there has been progress in analysing diffusion models in specific settings that lend themselves to theoretical analysis, in particular, the setting where the target is given by a Gaussian mixture model \citep{Wang2024-jc,Shah2023-ka,Gatmiry2024-ds,Chen2024-pf}. These works benefit from the fact that the forward flow from a Gaussian mixture is tractable and that the true score is given by a two-layer neural network.

Concurrent to this work, \cite{Dupuis2025-au} recently posted an algorithm/data-dependent analysis of diffusion models that targets optimisation algorithms used during training. The motivation of their work is to apply information-theoretic bounds in the setting of diffusion models and they do not use algorithmic stability.


\section{Background}\label{sec:background}

\subsection{An algorithmic formulation of diffusion models}
Suppose that the data distribution \(\nud\) is on \(\R^d\) and we are provided a finite dataset of samples \(S = \{x_1, ..., x_N\}\) which, throughout this work, we assume are sampled independently and identically (i.i.d.) from \(\nud\). As discussed in the introduction, diffusion models are formed of two distinct stages. The first stage, \textit{score matching}, consists of learning an approximation to the score function \(\nabla \log p_t\) using the dataset \(S\). In this work, we take a \textit{score function} to be any function belonging to the set \(L^0(\R^d \times [0, T]; \R^d)\), the set of Borel measurable functions of the form \(\R^d \times [0, T] \to \R^d\). Then, a \textit{score matching algorithm} is taken to be any mapping of the form \(A_\sm: (\cup_{N=1}^{\infty} (\R^d)^{\otimes N}) \times \Omega \to \mathcal{H}\) where \(\mathcal{H}\) is a measurable subset of \(L^0(\R^d \times [0, T]; \R^d)\). Here, \(\Omega\) is the event space belonging to some probability space \((\Omega, \mathcal{F}, \prob)\).

The second stage of diffusion models, \textit{backwards sampling}, consists of generating samples with the learned score function. We take a \textit{sampling algorithm} to be a mapping of the form, \(A_{\operatorname{samp}}:~\mathcal{H}~\to~\meas(\R^d)\) where \(\meas(\R^d)\) denotes the set of Borel measures on \(\R^d\). Typically, sampling is performed using an approximation to the backwards process given in \eqref{eq:backwards_sde}, replacing \(\nabla \log p_t\) with the learned score function \(s(\cdot, t)\). To do this, we must make a further approximation: by replacing its initial distribution, \(p_T\), which we do not have access to, with a data-independent prior distribution, \(p_{\text{prior}} = \mathcal{N}(\mathbf{0}, \sigma^2_{\text{prior}} I)\). In the case of \(\alpha > 0\), we choose \(\sigma^2_{\text{prior}} = \alpha^{-1}\) so that the prior coincides with the stationary distribution of the forward process, and when \(\alpha = 0\) we simply set \(\sigma^2_{\text{prior}} = 2T\). With this, we arrive at the SDE,
\begin{equation}\label{eq:sde_sampler}
    d\hat{Y}_t = \alpha \hat{Y}_t dt + 2 s(\hat{Y}_t, T-t) dt + \sqrt{2} dW_t, \quad \hat{Y}_0 \sim p_{\text{prior}}.
\end{equation}
Thus, a sample is generated by sampling from \(\hat{Y}_T\), or more commonly, the process is terminated early, sampling from \(\hat{Y}_{T-\epsilon}\) for some small \(\epsilon > 0\). Therefore a diffusion model is the density estimation algorithm formed of the composition \(A_{\operatorname{samp}} \circ A_\sm\).




\subsection{Denoising score matching and overfitting}\label{sec:prelim_overfitting}
As stated in the introduction, computing \(\ell_\sm\) requires access to the population score function, \(\nabla \log p_t\). So instead,S the \textit{(population) denoising score matching loss} is used in its place:
\begin{equation}\label{eq:dsm}
    \ell_\dsm(s; \tau) := \E_{X_0 \sim \nu} \bigg [ \int \E_{X_t|X_0} [\|s(X_t, t) - \nabla \log p_{t|0}(X_t|X_0)\|^2|X_0] \, \tw(dt) \bigg ],
\end{equation}
which differs from \(\ell_\sm(s)\) only by a constant,
\begin{equation*}
    C_\sm = \int \frac{\mu_t^2}{\sigma_t^4} \E[\tr \cov(X_0|X_t)] \tau(dt) ,
\end{equation*}
(see Lemma \ref{lem:dsm_decomposition}) whilst being easier to approximate without access to \(\nabla \log p_t\) \citep{Hyvarinen2005-zz}. Since \(p_{t|0}\) is a Gaussian kernel, its score function can be directly computed yielding,
\begin{equation}
    \nabla_y \log p_{t|0}(y|x) = \frac{\mu_t x - y}{\sigma_t^2}, \qquad \mu_t = e^{- \alpha t}, \qquad \sigma^2_t = \begin{cases}
        \alpha^{-1} (1 - \mu_t^2), & \text{if } \alpha > 0,\\
        2t, & \text{otherwise.}
    \end{cases}\label{eq:cond_score_form}
\end{equation}
In practice, the objective in \eqref{eq:dsm} is further approximated via Monte Carlo estimation using the dataset which leads to the \textit{empirical denoising score matching loss},
\begin{equation}\label{eq:emp_dsm}
    \hat{\ell}_\dsm(s; S, \tau) := \frac{1}{N} \sum_{i=1}^N \int \E_{X_t|X_0} [\|s(X_t, t) - \nabla \log p_{t|0}(X_t|x_i)\|^2 | X_0 = x_i] \, \tw(dt).
\end{equation}

In the following lemma, we highlight the important property that this can equivalently be defined as the denoising score matching objective for the process \(\hat{X}_t\) which evolves as in \eqref{eq:forward_sde} but with the initial distribution given by the empirical distribution, \(\hat{X}_0 \sim \frac{1}{N} \sum_{i=1}^N \delta_{x_i}(dx)\).

\copyenv{lemma}{lem:dsm_sm_emp}{
The objective \(\hat{\ell}_\dsm(s; S, \tau)\) is identical, up to a constant, to the objective
\begin{equation}\label{eq:emp_sm}
    \hat{\ell}_{\operatorname{sm}}(s; S, \tau) := \int \E[\|s(\hat{X}_t, t) - \nabla \log \hat{p}_t(\hat{X}_t)\|^2 | S] \tw(dt),
\end{equation}
where \(\hat{p}_t\) is the marginal density of \(\hat{X}_t\). Therefore, any minimiser of \(\hat{\ell}_\dsm(\cdot; S, \tau)\) on \(L^0(\R^d \times [0, T]; \R^d)\) is identical to \(\nabla \log \hat{p}_t\) a.e. for any \(t \in \supp(\tau)\).
}


See Appendix \ref{app:prelim_score} for the proof. This lemma shows that, unlike in traditional supervised learning problems, the empirical objective here admits a single unique minimiser, the \textit{empirical score function}, \(\nabla \log \hat{p}_t\). The nature of this score function and the samples it generates has been the focus of several recent studies, notably \citep{Pidstrigach2022-qr} which shows that with perfect sampling, any score function sufficiently close to \(\nabla \log \hat{p}_t\) recovers the training data.


\subsection{Other notation}

When the score matching algorithm \(A_\sm\) is random, we use \(A_\sm(S)\) as shorthand for the random score function \((x, t, \omega) \mapsto A_\sm(S, \omega)(x, t)\). Given two random score functions \(s, s'\), we let \(\Gamma(s, s')\) denote the set of all couplings of these random functions (see Appendix \ref{app:random_sm} for further details). We also use the KL divergence,
\begin{equation*}
    D(p\|q) = \int \log \frac{dp}{dq} \, dp, \qquad p,q \in \meas(\R^d), \qquad p \ll q,
\end{equation*}
as well as the notation \([k] := \{1, ..., k\}\). We use the relation \(\lesssim\) to denote bounds up to multiplicative constants and \(\simeq\) to denote that two variables share the same marginal distribution.

\section{Score stability and generalisation}\label{sec:stability}

Algorithmic stability is a classical method in learning theory used to understand the generalisation properties of a variety of important learning algorithms \citep{Kearns1999-jq, Devroye1979-du, Bousquet2002-ox, Hardt2016-ng}. While there are various formulations, they all share the common aim of connecting properties of a learning algorithm to its robustness under changes in the dataset. Its use has primarily been focused around regression and classification problems—in this section we propose a notion of stability that applies specifically to diffusion models.

\subsection{Algorithmic stability for score matching algorithms}
We introduce the notion of score stability which quantifies how sensitive a score matching algorithm \(A_\sm\) is to individual changes in the dataset. We do this by defining the adjacent dataset \(S^i := \{x_1, ..., x_{i-1}, \tilde{x}, x_{i+1}, ..., x_N\}\) where \(\tilde{x} \sim \nud\), independent from \(S\), and then measuring the similarity between the score functions \(\hat{s} = A_\sm(S)\) and \(\hat{s}^i = A_\sm(S^i)\).

\begin{definition}\label{def:forward_stability}
A score matching algorithm \(A_\sm\) is \textit{score stable} with constant \(\stab > 0\) if for any \(i \in [N]\) it holds that,
\begin{equation*}
    \E_{S, \tilde{x}} \bigg [ \inf_{(\hat{s}, \hat{s}^i) \in \Gamma_i} \int \E[\|\hat{s}(X_t, t) - \hat{s}^i(X_t, t)\|^2|X_0=\tilde{x}, S, \tilde{x}] \ \tw(dt) \bigg ] \leq \stab^2,
\end{equation*}
where \(\Gamma_i = \Gamma(A_\sm(S), A_\sm(S^i))\).
\end{definition}

Since \(A_\sm\) may be random we define score stability in terms of the best-case coupling of the random score functions \(\hat{s}\), \(\hat{s}^i\). We recall that \(\Gamma(\cdot, \cdot)\) denotes the set of couplings between two random score functions and that when the score matching algorithm is not random, it is formed from the singleton \(\Gamma_i = \{(A_\sm(S), A_\sm(S^i))\}\).

\subsection{Generalisation gap bounds under score stability}

In the following theorem, we connect score stability to generalisation by controlling the expected generalisation gap by the score stability constant. We obtain two generalisation gap bounds: one controlling the gap for the denoising score matching loss, and one controlling the gap for the score matching loss.

\copyenv{theorem}{thm:dsm_stability}{
Suppose that the score matching algorithm \(A_\sm\) is score stable with constant \(\stab\). Then, with \(\hat{s} = A_\sm(S)\), it holds that
\begin{equation}\label{eq:dsm_stability_0}
    \Big | \E \big [ \ell_{\operatorname{dsm}}(\hat{s}; \tau) \big ]^{1/2} - \E \big [ \hat{\ell}_{\operatorname{dsm}}(\hat{s}; S, \tau) \big ]^{1/2} \Big | \leq \stab.
\end{equation}
Furthermore, it holds that
\begin{equation}\label{eq:dsm_stability_1}
    \E \big [ \ell_{\operatorname{sm}}(\hat{s}; \tau) \big ] - \E \big [ \hat{\ell}_{\operatorname{sm}}(\hat{s}; S, \tau) \big ] \leq 2 \, \stab \, \E \big [ \hat{\ell}_{\operatorname{dsm}}(\hat{s}; S, \tau) \big ]^{1/2} + \stab^2.
\end{equation}
}
With Theorem \ref{thm:dsm_stability}, we obtain that the generalisation gap for both the denoising score matching loss and the score matching loss decays at the same rate as score stability. We can further simplify the bound for the score matching loss using the fact that \(\hat{\ell}_\dsm\) and \(\hat{\ell}_\sm\) are identical up to a constant to obtain,
\begin{equation*}
    \E \big [ \ell_{\operatorname{sm}}(\hat{s}; \tau) \big ] \lesssim \E \big [ \hat{\ell}_{\operatorname{sm}}(\hat{s}; \tau) \big ] + \stab \, C_\sm^{1/2} + \stab^2.
\end{equation*}

One should expect that if the score matching algorithm is effective, both \(\hat{s}\) and \(\hat{s}^i\) converge to the ground truth as \(N\) grows, and thus \(\stab\) should decrease to \(0\). Ascertaining the rate at which \(N\) decreases requires an analysis of the algorithm at hand, hence the categorisation of algorithmic stability as an algorithm-dependent approach. This contrasts with uniform learning which utilises control over the hypothesis class, providing a worst-case bound that is independent from the algorithm.

To highlight this, one can consider the illustrative example where the algorithm \(A_\sm\) chooses the score function from a highly complex class but does so entirely randomly and independent of the dataset. Despite the nature of the algorithm, the approach of uniform learning would lead to loose bounds due to the high complexity of the hypothesis class, whereas with score stability, the randomness can be coupled in such a way that the stability constant is \(0\), irrespective of the complexity of the hypothesis class, and therefore the generalisation gap is \(0\) also.

In the following sections, we apply the framework of score stability to some common learning settings for diffusion models. We derive estimates of the score stability constant for these algorithms and identify features that promote generalisation.

\section{Empirical score matching and implicit regularisation}\label{sec:score_matching}

We begin our examples by considering the score matching algorithm that minimises the empirical denoising score matching loss. Given a hypothesis class \(\mathcal{H} \subseteq L^0(\R^d \times [0, T]; \R^d)\), we define this algorithm by,
\begin{equation*}
    A_{\operatorname{erm}}(S) = \argmin_{s \in \mathcal{H}} \hat{\ell}_\dsm(s; S, \tau).
\end{equation*}
While this algorithm is not often used in practice, it is the natural analogue to empirical risk minimisation from traditional supervised learning and thus serves as a canonical example. We consider the setting of the manifold hypothesis where the data distribution is supported on a submanifold within \(\R^d\).
\begin{assumption}\label{ass:manifold}
Suppose that \(\nud\) is supported on a smooth submanifold of \(\R^d\) that has dimension \(d^*\) and reach \(\uptau_{\text{reach}} > 0\). Furthermore, its density on the submanifold, \(p_\nu\), satisfies \(c_\nu := \inf p_\nu > 0\).
\end{assumption}

The reach describes the maximum distance where the projection to the manifold is uniquely defined and therefore, it quantifies the maximum curvature of the manifold. We refer to Appendix \ref{sec:manifold} for the full definition. Several recent works have considered the assumption that \(\nud\) lies on a submanifold of \(\R^d\). These works argue that \(d^*\) can often be far smaller than \(d\) and so dependence with respect to \(d^*\) over \(d\) is favourable \cite{De-Bortoli2022-tv, Pidstrigach2022-qr, Loaiza-Ganem2024-eu, Potaptchik2024-ue, Huang2024-ky}. The assumption that the density is bounded from below has appeared in several of these works also \cite{Potaptchik2024-ue, Huang2024-ky}.

We make the following assumption about the class of score networks.

\begin{assumption}\label{ass:hypoth_diam}
Suppose there exists \(D_\mathcal{H} \geq 0\) such that for any \(s, s' \in \mathcal{H}\), it holds that
\begin{equation*}
    \| s(\cdot, t) - s'(\cdot, t) \|_{L^\infty} \leq D_\hilb/\sigma_t^2, \qquad \text{for all } t \in \supp(\tau).
\end{equation*}
\end{assumption}

Under these assumptions, we obtain the following estimate for the stability constant.
\copyenv{proposition}{prop:stability_erm}{
Suppose that assumptions \ref{ass:manifold} and \ref{ass:hypoth_diam} hold and that \(\epsilon = \inf \supp(\tau) \in (0, \uptau_{\text{reach}}^2)\), then for any \(c \in (0, 1)\) and sufficiently large \(N\), the score matching algorithm \(A_{\text{erm}}\) is score stable with constant,
\begin{equation*}
    \stab^2 \lesssim C \Big ( C C_\sm N^{-2} + \E[\hat{\ell}_\sm(\hat{s})] \Big )^c, \qquad C = \frac{D_\hilb^2}{\sigma_\epsilon^4} \vee \frac{1}{c_\nu \sigma_\epsilon^{d^*}}.
\end{equation*}
}

Since we have made only basic assumptions about the structure of the hypothesis class this result suggests that the denoising score matching loss possesses a form of implicit regularisation. This contrasts with algorithmic stability in the setting of traditional supervised learning, where empirical risk minimisation is stable only when restricting the the hypothesis class or with the use of explicit regularisation \cite{Zhang2021-ly, Bousquet2002-ox}. In Proposition \ref{prop:stability_erm}, we show that the denoising score matching loss possesses the unique property that it is stable without the need for additional regularisation.

When \(d^* > 4\), for \(\epsilon\) sufficiently small, we have that \(C = \mathcal{O}(c_\nu^{-1} \epsilon^{-d^*/2}), C_\sm = \mathcal{O}(d^* \epsilon^{-1})\). Since the bound only depends on manifold dimension and not affine dimension, this suggests that diffusion models, via score matching, are automatically manifold-adaptive. The bound also heavily depends on \(\epsilon\), with it being smaller for larger \(\epsilon\) and growing exponentially fast as \(\epsilon\) approaches zero. This suggests that the natural regularisation present in the score matching objective is more prevalent at larger noise scales.

The requirement to have \(\epsilon > 0\) is closely related to the technique of early stopping which is frequently used in the diffusion model literature \citep{Song2021-aj, Karras2022-us}. This is where the backwards process \(\hat{Y}_t\) is terminated early by some small amount of time to avoid irregularity issues of the score function when close to convergence. The relationship between early stopping and the choice of time-weighting is captured in a simple result from \cite{Song2021-rf}: if we let \(\hat{q}_t\) denote the marginal distribution of \(\hat{Y}_t\) (recall the definition in \eqref{eq:backwards_sde}), we obtain that
\begin{equation}\label{eq:early_stop}
    D(p_\epsilon \| \hat{q}_{T-\epsilon}) \leq \ell_\sm(s;\tau_\epsilon) + D(p_T\|p_\infty),
\end{equation}
where \(\tau_\epsilon(dt) = \mathbbm{1}_{t \in [\epsilon, T]} dt\). Other theoretical works have identified the use of early stopping in the generalisation properties of diffusion models \citep{Beyler2025-tt}: in particular, \cite{Oko2023-sa} require that the sampler is stopped early by an amount proportional to \(N^{-1/d}\) which was later improved \cite{Azangulov2024-uf} to \(N^{-1/d^*}\).

\paragraph{Proof summary}
We now provide a brief summary of the proof of Proposition \ref{prop:stability_erm}. The first step of the proof technique utilises a fundamental property of the empirical denoising score matching objective, \(\hat{\ell}_\dsm(s; S, \tau)\): that through its equivalence to the empirical score matching objective (see Lemma \ref{lem:dsm_sm_emp}) it is strongly convex in \(s\) in a data-dependent weighted \(L^2\) space. The use of strong convexity in algorithmic stability analyses is commonplace, frequently used to analyse linear methods—here we borrow a similar approach to analyse stability in function space. With this, we arrive at the following inequality (see Lemma \ref{lem:erm_avg_stability}):
\begin{equation}\label{eq:erm_proof_outline_1}
    \int \E[\|\hat{s}(\hat{X}_t, t) - \hat{s}^{i}(\hat{X}_t, t)\|^2] \tau(dt) \lesssim \E[\hat{\ell}_\sm(\hat{s})] + \frac{\stab}{N} ( C_\sm^{1/2} + \stab),
\end{equation}
where \(\stab\) is the (yet-to-be bounded) score stability constant of \(A_\erm\). Note that this controls the difference between \(\hat{s}\) and \(\hat{s}^{i}\) integrated with respect to \(\hat{X}_t\), whereas score stability requires a bound on the difference integrated with respect to \(X_t|X_0 = \tilde{x}\), which motivates the second step.

The second step of the proof technique utilises a property of the heat kernel—that it smooths out functions. In particular, we utilise the celebrated Harnack inequality of \cite{Wang1997-il} that captures this property by showing that for any positive measurable \(\phi: \R^d \to \R_+\), \(x, y \in \R^d\), it holds that
\begin{equation*}
    \E[\phi(X_t)|X_0 = x] \leq \E[\phi(X_t)^p|X_0 = y]^{1/p} \exp \bigg ( \frac{\mu_t^2 \|x-y\|^2}{2 (p-1) \sigma_t^2} \bigg ),
\end{equation*}
for any \(t > 0, p > 1\). Utilising this bound, we convert the upper bound in \eqref{eq:erm_proof_outline_1} to a bound on the stability constant. The full proof can be found in Appendix \ref{app:erm}.

\section{Stochastic sampling and score stability}\label{sec:sampling}

In practice, the backwards process in \eqref{eq:sde_sampler} cannot be sampled exactly, so we rely on approximations based on numerical integration schemes such as the Euler--Maruyama (EM) scheme (or one of its variations) which approximates the continuous-time dynamics of \(\hat{Y}_t\) with a discrete-time Markov process. In this section, we investigate how algorithmic stability interacts with discrete-time sampling schemes.

We consider the sampling scheme proposed in \citep{Benton2023-ov, Potaptchik2024-ue} which discretises at the timesteps \((t_k)_{k=0}^K\), where we define,
\begin{equation}\label{eq:disc_scheme}
    t_k =\begin{cases}
        \kappa k, & \text{if } k < \frac{T-1}{\kappa},\\
        T - (1 + \kappa)^{\frac{T-1}{\kappa}-k}, & \text{if } \frac{T-1}{\kappa} \leq k \leq K,
    \end{cases} 
\end{equation}
where \(L = \frac{T-1}{\kappa} > 0\), \(\kappa > 0\) and \(K = \lfloor L + \log(\epsilon^{-1})/\log(1+\kappa) \rfloor\) defines the number of steps so that \(t_K \approx \epsilon\) (see Appendix \ref{app:samp} for details). By sampling its terminating iterate \(\hat{y}_K\), we obtain a \textit{sampling algorithm}, \(A_\emm\), that maps a score function \(s\) to the distribution \(\law(\hat{y}_K)\), approximation to the distribution \(\law(\hat{Y}_\epsilon)\).

In the previous section, we identified that early stopping of the backwards process benefits generalisation. In the present section, we will consider how coarseness of the discretisation scheme produces similar benefits. When an EM-type sampling algorithm is utilised, it is often the case that the score function is trained only at those time steps considered by the sampler, i.e. using the time-weighting,
\begin{equation*}
    \hat{\tau}_\kappa(dt) = \frac{1}{K} \sum_{k=0}^{K-1} \delta_{T-t_k}(dt).
\end{equation*}
As a result, the effective stopping time of the algorithm can be much larger than the early stopping time, \(\epsilon\). In the following proposition, we demonstrate how this benefits generalisation.



\begin{proposition}\label{prop:coarse}
Consider the setting of Proposition \ref{prop:stability_erm} with \(\alpha = 1\) and set \(\tau = \hat{\tau}_\kappa\), then for sufficiently large \(N\), \(\kappa \leq \epsilon^{-1}/4\) and any \(c \in (0, 1)\), we have  
\begin{align*}
    &\E[D(p_\epsilon \| A_\emm \circ A_{\operatorname{erm}}(S))] \lesssim \E [ \hat{\ell}_{\sm, \kappa}^\star ] + B_\kappa^{1/2} (1 + \kappa)^{-d^*} + \frac{B_\kappa}{C_\sm} (1 + \kappa)^{-2d^*} + \kappa (1 + \kappa) d^* \log(\epsilon^{-1})^2 + d e^{-2T},
\end{align*}
where \(B_\kappa = \tfrac{C_\sm}{c_\nu} (\tfrac{C_\sm}{c_\nu} N^{-2} + \E [ \inf_{\mathcal{H}} \hat{\ell}_\sm(h; S, \hat{\tau}_\kappa) ])^c \epsilon^{-d^*}\).
\end{proposition}

The second and third terms of the bound in Proposition \ref{prop:coarse} are due to the score stability of the ERM algorithm and decay as \(\kappa\) increases. The fourth term of the bound captures the discretisation error and therefore increases with \(\kappa\). What this result captures is that there is a trade-off between sampler accuracy and generalisation that is managed by the discretisation of the diffusion model. In the following corollary, this trade-off is optimised.

\begin{corollary}\label{cor:coarse_opt}
Consider the setting of Proposition \ref{prop:coarse}, then for any \(c \in (0, 1)\) and sufficiently small \(\epsilon\), there exists \(\kappa > 0\) such that
\begin{equation*}
    \E[D(p_\epsilon \| A_\emm \circ A_{\operatorname{erm}}(S))] \lesssim
    \begin{cases}
        B_{\kappa}^{1/2} + C_\sm^{-1} B_{\kappa}, &\text{ if } B_\kappa \leq \log(\epsilon^{-1})^2,\\
        \log(\epsilon^{-1}) B_{\kappa^*}^{\frac{1}{2(d^* + 1)}} + (C_\sm^{-1} + d^*) \log(\epsilon^{-1})^2 B_{\kappa^*}^{\frac{1}{d^* + 1}} +  d e^{-2T}, &\text{ otherwise.}
    \end{cases}
\end{equation*}
\end{corollary}

The primary strength of this result in comparison with \citep{Oko2023-sa, Azangulov2024-uf} is that we assume little about the hypothesis class optimised over. Their results fundamentally rely on a carefully constrained set of feed-forward neural networks where the number of parameters in each layer is chosen according to the size of the dataset, dimension of the data and the smoothness of the target so that when the uniform learning analysis is performed, the complexity of the hypothesis class is controlled. Furthermore, their approach depends on a carefully chosen early stopping time, requiring that the backwards process is terminated early to obtain their generalisation bounds. Our result holds for any sufficiently small early stopping time, instead relying on the discretisation scheme to be carefully chosen. Since, in practice, the discretisation scheme is often tuned as a hyper-parameter \citep{Karras2022-us, Williams2024-qj}, we believe that this trade-off is preferable. The primary drawback of our bound compared to these results is that by not exploiting the model class considered we fail to adapt to any smoothness properties of the underlying measure. We leave the integration of model class smoothness into the score stability framework as future work.



\section{Stochastic optimisation and implicit regularisation}

In practice, the score function is typically chosen from a parametric hypothesis class \(\{s_\theta: \theta \in \R^n\}\) (e.g. a deep neural network where \(\theta\) represents the weights and biases) and the parameters are chosen by minimising the denoising score matching loss via stochastic optimisation \citep{Karras2024-cw}. In this section, we consider the score stability of this algorithm, focussing on stochastic gradient descent (SGD) with gradient clipping and weight decay.

We consider the standard gradient estimator: given the mini-batch \((x'_i)_{i=1}^{N_B}\) of size \(N_B \ll N\) we define the random estimator,
\begin{equation}\label{eq:stoch_grad}
    G(\theta, (x'_i)_{i=1}^{N_B}) = \frac{1}{N_B P} \sum_{i=1}^{N_B} \sum_{j=1}^P w_{t_{i, j}} \nabla_\theta \|s_\theta(X_{i, j}, t_{i, j}) - \nabla \log p_{t_{i, j}|0}(X_{i, j}| x)\|^2,
\end{equation}
where we define the random variables \(X_{i, j} = \mu_{t_{i, j}} x'_i + \sigma_{t_{i, j}} \xi_{i, j}\), \(t_{i, j} \sim w_t^{-1} \tw(dt), \xi_{i, j} \sim N(0, I_d)\). The additional variance introduced by the random variables \(\xi_{i, j}\) and \(t_{i, j}\) leads to a gradient estimator with significantly higher variance than in standard supervised learning. This presents several challenges during training, and various strategies have been proposed to mitigate this issue \citep{Karras2024-cw, Song2021-aj}. For example, the weighting function \(w: [0, T] \to \R_+\) can be tuned to reduce variance \citep{Karras2022-us} or the number of resamples \(P \in \N\) can be increased. We consider the following iterative scheme, defined for a given weight decay constant \(\lambda > 0\) and clipping value \(C > 0\):
\begin{equation}\label{eq:sgd_iter}
    \theta_{k+1} = (1 - \eta_k \lambda) \theta_k - \eta_k \, \clip_C(G_k(\theta_k, (x_i)_{i \in B_k})),
\end{equation}
where \(\eta_k > 0\) and \(B_k \subset [N]\) denotes the learning rate and mini-batch indices at each iteration \(k\) and we define the clipping operator \(\clip_C(v) = (1 \wedge (C\|v\|^{-1})) v\). Both gradient clipping and weight decay are widely used in diffusion model training and are typically motivated by their stabilising effect on optimisation, minimising the impact of the high variance of the gradient estimator \citep{Song2021-uf,Ho2020-pq}. Throughout this section, we take the mini-batch \(B_k\) to be i.i.d. and uniformly sampled from \([N]\) without replacement. For the sake of simplicity, we suppose that the iterative scheme is terminated after \(K \in \N\) iterations, where \(K\) is fixed and independent of the data.

\subsection{Stability of SGD with weight decay and clipping}\label{sec:sgd_stability}

In our analysis, we avoid restricting the score network to a specific parametric class and instead make structural assumptions based on its smoothness properties. We recall that a function is Lipschitz with constant \(L \geq 0\) if it is differentiable and its directional derivatives are uniformly bounded by \(L\).
\begin{assumption}[Smoothness of the score network]\label{ass:smoothness_score}
There exists \(L: \R^d \times (0, T] \to \R_+\) and \(M: \R^d \times (0, T] \to \R_+\) such that for almost all \(x \in \R^d, t \in (0, T]\), \(s_\theta(x, t)\) is Lipschitz and smooth (gradient-Lipschitz) in \(\theta \in \R^n\) with constants \(L(x, t)\) and \(M(x, t)\), respectively. Furthermore, there exists constants \(\overline{M}, \overline{L} \geq 0\) such that for any \(x \in \supp(\nud)\),
\begin{equation*}
    \int \E [ L(X_t, t)^2 | X_0 = x ] \, \tau(dt) \leq \overline{L}^2, \quad \int \E [ M(X_t, t)^2|X_0=x] \, \tau(dt) \leq \overline{M}^2 .
\end{equation*}
\end{assumption}
The use of Lipschitz and smoothness assumptions is commonplace in the analysis of optimisation schemes \cite{Nesterov2018-yb,Hardt2016-ng}. However, the assumption differs from the usual in that we only require these properties to hold almost everywhere with respect to the input distribution and we allow the Lipschitz and smoothness constants to vary with the input, provided their square averages remain bounded. This relaxation enables us to accommodate for common models that would otherwise violate global smoothness assumptions such as ReLU networks.

\begin{assumption}\label{ass:dsm_stability_prob}
Suppose there exists \(B_\ell > 0\) such that for any \(\theta \in \R^n\), it holds that
\begin{equation}\label{eq:dsm_stability_prob}
    \hat{\ell}_\dsm(s_\theta; \{x\}, \delta_t) \leq B_\ell^2/\sigma_t^4, \qquad \text{ for each } x \in \operatorname{supp}(\nud), t \in \supp(\tw).
\end{equation}
\end{assumption}

This property requires that the supported score functions are made of denoising functions that are concentrated on a compact set. To highlight that this can be achieved quite easily, we note that with the naive estimate \(s(x, t) = - x / \sigma_t^2\), \eqref{eq:dsm_stability_prob} is satisfied with \(B_\ell^2 = \E[\|X_0\|^2]\).

In the following proposition we demonstrate score stability bounds in the case that the step size is decaying with a rate of \(1/k\).
\copyenv{proposition}{prop:sgd_stability}{
Consider the score matching algorithm \(A_\sm: S \mapsto s_{\theta_{K}}\) for some fixed \(K \in \N\) where \((\theta_k)_{k}\) as given in \eqref{eq:sgd_iter}. Suppose that assumptions \ref{ass:smoothness_score} and \ref{ass:dsm_stability_prob} hold and \(\eta_k \leq \bar{\eta} / k\) for all \(k < K\) for some \(\bar{\eta} \in (0, \lambda^{-1})\). Then, we obtain that \(A_\sm\) is score stable with constant,
\begin{equation*}
    \stab^2 \lesssim \bigg ( \frac{C}{\lambda} \vee R \bigg )^{1 + \frac{\bar{\eta}\upsilon}{\bar{\eta}\upsilon + 1}} \frac{\overline{L}^2}{(\bar{\eta}\upsilon) \vee 1} \bigg ( \frac{C}{\bar{\eta}} \bigg )^{\frac{1}{\bar{\eta}\upsilon + 1}} \frac{N_B K^{\frac{\bar{\eta}\upsilon}{\bar{\eta}\upsilon + 1}}}{N},
\end{equation*}
where \(R^2 = \E[\|\theta_0\|^2], \upsilon = (\overline{M} B_\ell C_\tau^{1/2} + \overline{L}^2 - \lambda) \vee 0\) and \(C_\tau = \int \sigma_t^{-4} \tau(dt)\).
}

Since the score matching algorithm is random, to control the stability constant we construct a coupling of the random score functions \(A_\sm(S)\) and \(A_\sm(S^i)\), or equivalently, a coupling of the optimisation trajectories associated with training on \(S\) versus \(S^i\). The construction of this coupling is such that the trajectories are identical for a large portion of the train-time. The proof also heavily relies on the stochastic mini-batching, hence why we consider the setting of \(N_B \ll N\). The proof-technique is a modification of a methodology developed by \cite{Hardt2016-ng}. Several recent works have explored the influence of first-order optimisation methods on generalisation \citep{Neu2021-rh, Pensia2018-md, Hardt2016-ng, Clerico2022-ij, Dupuis2025-mq}.

\subsection{Utilising noise in the gradient estimator}\label{sec:noisy_opt}
In the previous section, we established stability bounds for SGD with gradient clipping and weight decay under basic structural assumptions. The primary drawback of this bound is that it grows quickly with the number of iterations—a limitation that becomes more significant in the context of diffusion models, which typically require far more optimisation steps than the size of the dataset due to the high variance of the gradient estimator. In this section, we improve the dependence on the number of iterations by explicitly leveraging the noise in the gradient estimator. The idea that stochasticity in optimisation can act as a form of implicit regularisation has motivated the development of numerous learning algorithms in recent years \cite{Srivastava2014-eh, Bishop1995-ww, Sietsma1991-zo} and have inspired numerous theoretical directions of research \cite{Mou2018-bz, Pensia2018-md, NEURIPS2020_37693cfc, Farghly2021-xg}. Here, we investigate how the noise intrinsic to the gradient estimator for \(\hat{\ell}_\dsm\) can play a similar role in promoting generalisation in diffusion models.
 
To incorporate the effects of the gradient noise, we consider a simplified model in which the noise from the stochastic gradient estimator is approximated with a second-order Gaussian approximation:
\begin{gather}
    \theta_{k+1} = (1 - \eta \lambda) \theta_p - \eta \E \Big [ \clip_C(G(\theta_k, B_k)) \Big | \theta_k, B_k, S \Big ] + \eta \Sigma(\theta_k, B_k)^{1/2} \xi_k,\label{eq:sgd_gaussian_approx}\\
    \Sigma_S(\theta, B) = \operatorname{Cov} \Big ( \clip_C(G(\theta, B)) \Big | \theta, B, S \Big ),\nonumber
\end{gather}
where \(\xi_k \in \R^d\) is a standard Gaussian. This approximation can be justified by observing that the inner summation in \eqref{eq:stoch_grad} is over conditionally i.i.d. variables, once conditioned on \(\theta, B\) and \(S\). Therefore, the gradient estimator \(G\) becomes approximately Gaussian as \(P\) grows large.

For this analysis, we assume the following lower bound on the gradient noise.
\begin{assumption}\label{ass:covariance}
There exists a positive semi-definite matrix \(\overline{\Sigma} \in \R^{n \times n}\) such that for any \(x \in \supp(\nu)\) and \(\theta \in \R^n\),
\begin{equation*}
    \operatorname{Cov}_{t \sim \tau, X_t|X_0} \big ( \clip_C(\nabla_\theta \|s_\theta(X_t, t) - x\|^2) \big | X_0 = x \big ) \succcurlyeq \overline{\Sigma}.
\end{equation*}
Furthermore, the eigenvalues of \(\overline{\Sigma}\), \((\lambda_i)_{i=1}^n\), possess the spectral gap \(\lambda_{\text{gap}} := \min_{\lambda_i \neq 0} \lambda_i > 0\).
\end{assumption}
We use the matrix \(\overline{\Sigma}\) to dictate the geometry on which we perform our analysis. In particular, we consider the weighted norm \(\|v\|_{\overline{\Sigma}^+} := v^T \overline{\Sigma}^+ v\) where \(\overline{\Sigma}^+\) is the pseudoinverse matrix.

\begin{assumption}\label{ass:covariance_smoothness}
For almost all \(x \in \R^d, t \in (0, T]\), \(s_\theta(x, t)\) is Lipschitz and smooth (gradient-Lipschitz) in \(\theta \in \R^n\) with respect to the seminorm \(\|\cdot\|_{\overline{\Sigma}^+}\) and with constants \(L(x, t)\) and \(M(x, t)\), respectively. Furthermore, there exists constants \(\overline{M}, \overline{L} \geq 0\) such that for any \(x \in \supp(\nud)\),
\begin{equation*}
    \int \E [ L(X_t, t)^4 | X_0 = x ] \, \tau(dt) \leq \overline{L}^4, \quad \int \E [ M(X_t, t)^4|X_0=x] \, \tau(dt) \leq \overline{M}^4 .
\end{equation*}
\end{assumption}
By requiring that the Lipschitz and smoothness properties hold with respect to \(\|\cdot\|_{\overline{\Sigma}^+}\), we effectively require that the gradient estimator adds noise in all directions aside from those that do not change the function (e.g. along symmetries in parameter space).

With this, we arrive at our time-convergent score stability bound for SGD.
\copyenv{proposition}{prop:time_indep_bounds}{
Consider the score matching algorithm \(A_\sm: S \mapsto s_{\theta_K}\) for some fixed \(K \in \N\) where \((\theta_k)_{k}\) as given in \eqref{eq:sgd_gaussian_approx}. Suppose that assumptions \ref{ass:dsm_stability_prob}, \ref{ass:covariance_smoothness} and \ref{ass:covariance} hold, then there exists some \(\bar{\eta} > 0\) such that, if \(\sup_p \eta_p \leq \bar{\eta}\), we obtain that \(A_\sm\) is score stable with constant
\begin{equation*}
    \stab^2 \lesssim \frac{\overline{L}^2 C^2 (P + n)}{\lambda_{\text{gap}} N} \min \bigg \{ \frac{\eta_{\min} \lambda_{\text{gap}} \lambda^2}{P N_B C} \sum_{k=0}^{K-1} \eta_k, \exp \bigg ( \tilde{c} \frac{P N_B C}{\eta_{\min} \lambda_{\text{gap}} \lambda^2} \bigg ) \bigg \}.,
\end{equation*}
where \(\tilde{c} \lesssim (\overline{M}_4 B_\ell C_\tau^{1/2} + \overline{L}_4^2) (P N_B \lambda_{\text{gap}})^{-1/2} \vee 1\), \(\eta_{\min} = \min_k \eta_k\).
}

In this bound, we recover the \(\frac{1}{\sqrt{N}}\) score stability bounds from Proposition \ref{prop:sgd_stability} while also introducing the property that the bound does not grow endlessly with the number of iterations. This property is obtained using the noise in the gradient estimator and is not possible without additional noise. Our proof methodology builds on techniques developed in the literature analysing stochastic gradient Langevin dynamics, a modification of SGD that applies isotropic Gaussian noise at each step. In particular, we draw on the reflection coupling method of \citet{Farghly2021-xg}, which constructs coupled trajectories of the optimisation iterated that contract in expectation under a suitably defined metric (using the technical results of \cite{Eberle2019-ai, Majka2020-lp}). In our setting, this contraction arises naturally from the noise inherent to the gradient estimator for the denoising score matching objective. As \(P\) increases or \(\eta\) decreases, the long-term stability bound increases exponentially fast as a result of the benefit of the noise weakening. Through this analysis, we identify the generalisation benefit of a property unique to diffusion models and how they interact with SGD.

\section{Conclusion and Future Work}
In this paper, we proposed a general-purpose algorithm-dependent framework for analysing the generalisation capabilities of diffusion models. We introduced \textit{score stability}, which quantifies the sensitivity of a score matching algorithm to changes in the dataset, and used it to derive expected generalisation gap bounds. We then applied this framework to some common settings, deriving closed-form bounds on the score stability constant for several score matching algorithms. In the process, we identify several sources of implicit regularisation in diffusion models that have previously been over-looked. We begin with an analysis of empirical risk minimisation, finding that the denoising score matching objective automatically yields score stability guarantees without further regularisation (denoising regularisation). We then analyse how score stability interacts with discrete-time sampler schemes, identifying that coarse discretisation can be used to improve generalisation guarantees (sampler regularisation). Finally, we consider stochastic optimisation schemes in the setting of score matching, obtaining score stability guarantees (optimisation regularisation).

This work opens up several directions for future work. This includes identifying further relationships between score stability and generalisation by developing high probability bounds, or bounds on notions of memorisation or privacy. The analysis of empirical score matching could be tightened by utilising more properties of the data distribution or model class, such as smoothness. The analysis of sampling could be taken further by considering and comparing different sampling algorithms under the score stability framework (e.g. by considering the probability flow ODE).

\section*{Acknowledgements}
Tyler Farghly was supported by Engineering and Physical Sciences Research Council (EPSRC) [grant number EP/T517811/1] and by the DeepMind scholarship. Patrick Rebeschini was funded by UK Research and Innovation (UKRI) under the UK government’s Horizon Europe funding guarantee [grant number EP/Y028333/1]. George Deligiannidis acknowledges support from EPSRC [grant number EP/Y018273/1]. The authors would like to thank Michael Hutchinson, Valentin De Bortoli, Peter Potaptchik, Sam Howard, Iskander Azangulov and Christopher J. Williams for valuable comments and stimulating discussions. We would like to give special thanks to Sam Howard for assisting in creating Figure \ref{fig:exp}.

\bibliography{references}

\begin{thebibliography}{63}
\providecommand{\natexlab}[1]{#1}
\providecommand{\url}[1]{\texttt{#1}}
\expandafter\ifx\csname urlstyle\endcsname\relax
  \providecommand{\doi}[1]{doi: #1}\else
  \providecommand{\doi}{doi: \begingroup \urlstyle{rm}\Url}\fi

\bibitem[Aamari(2017)]{Aamari2017-he}
E.~Aamari.
\newblock \emph{Convergence Rates for Geometric Inference}.
\newblock PhD thesis, Université Paris-Saclay, Sept. 2017.

\bibitem[Aamari et~al.(2019)Aamari, Kim, Chazal, Michel, Rinaldo, and Wasserman]{Aamari2019-gk}
E.~Aamari, J.~Kim, F.~Chazal, B.~Michel, A.~Rinaldo, and L.~Wasserman.
\newblock Estimating the reach of a manifold.
\newblock \emph{Electronic Journal of Statistics}, 13\penalty0 (1):\penalty0 1359--1399, 2019.

\bibitem[Attia and Koren(2022)]{Attia2022-ee}
A.~Attia and T.~Koren.
\newblock Uniform stability for first-order empirical risk minimization.
\newblock In P.-L. Loh and M.~Raginsky, editors, \emph{Proceedings of Thirty Fifth Conference on Learning Theory}, volume 178 of \emph{Proceedings of Machine Learning Research}, pages 3313--3332. PMLR, 2022.

\bibitem[Azangulov et~al.(2024)Azangulov, Deligiannidis, and Rousseau]{Azangulov2024-uf}
I.~Azangulov, G.~Deligiannidis, and J.~Rousseau.
\newblock Convergence of diffusion models under the manifold hypothesis in high-dimensions.
\newblock \emph{arXiv preprint arXiv:2409.18804}, 2024.

\bibitem[Bakry et~al.(2014)Bakry, Gentil, and Ledoux]{Bakry2014-ut}
D.~Bakry, I.~Gentil, and M.~Ledoux.
\newblock \emph{Analysis and Geometry of Markov Diffusion Operators}.
\newblock Springer International Publishing, 2014.

\bibitem[Bartlett et~al.(2021)Bartlett, Montanari, and Rakhlin]{Bartlett2021-tf}
P.~L. Bartlett, A.~Montanari, and A.~Rakhlin.
\newblock Deep learning: a statistical viewpoint.
\newblock \emph{Acta Numerica}, 30:\penalty0 87–201, 2021.

\bibitem[Benton et~al.(2024)Benton, De~Bortoli, Doucet, and Deligiannidis]{Benton2023-ov}
J.~Benton, V.~De~Bortoli, A.~Doucet, and G.~Deligiannidis.
\newblock Nearly $d$-linear convergence bounds for diffusion models via stochastic localization.
\newblock In \emph{International Conference on Learning Representations}, 2024.

\bibitem[Beyler and Bach(2025)]{Beyler2025-tt}
E.~Beyler and F.~Bach.
\newblock Optimal denoising in score-based generative models: The role of data regularity.
\newblock \emph{arXiv [cs.LG]}, Mar. 2025.

\bibitem[Bishop(1995)]{Bishop1995-ww}
C.~M. Bishop.
\newblock Training with noise is equivalent to tikhonov regularization.
\newblock \emph{Neural Computation}, 7\penalty0 (1):\penalty0 108--116, 1995.

\bibitem[Bousquet and Elisseeff(2002)]{Bousquet2002-ox}
O.~Bousquet and A.~Elisseeff.
\newblock Stability and generalization.
\newblock \emph{The Journal of Machine Learning Research}, 2:\penalty0 499--526, 2002.

\bibitem[Charles and Papailiopoulos(2018)]{Charles2018-gl}
Z.~Charles and D.~Papailiopoulos.
\newblock Stability and generalization of learning algorithms that converge to global optima.
\newblock In J.~Dy and A.~Krause, editors, \emph{Proceedings of the 35th International Conference on Machine Learning}, volume~80 of \emph{Proceedings of Machine Learning Research}, pages 745--754. PMLR, 2018.

\bibitem[Chen et~al.(2023)Chen, Chewi, Li, Li, Salim, and Zhang]{Chen2022-wa}
S.~Chen, S.~Chewi, J.~Li, Y.~Li, A.~Salim, and A.~R. Zhang.
\newblock Sampling is as easy as learning the score: theory for diffusion models with minimal data assumptions.
\newblock In \emph{International Conference on Learning Representations}, 2023.

\bibitem[Chen et~al.(2024)Chen, Kontonis, and Shah]{Chen2024-pf}
S.~Chen, V.~Kontonis, and K.~Shah.
\newblock Learning general {G}aussian mixtures with efficient score matching.
\newblock \emph{arXiv preprint arXiv:2404.18893}, 2024.

\bibitem[Clerico et~al.(2022)Clerico, Farghly, Deligiannidis, Guedj, and Doucet]{Clerico2022-ij}
E.~Clerico, T.~Farghly, G.~Deligiannidis, B.~Guedj, and A.~Doucet.
\newblock Generalisation under gradient descent via deterministic {PAC}-bayes.
\newblock \emph{arXiv [stat.ML]}, Sept. 2022.

\bibitem[De~Bortoli(2022)]{De-Bortoli2022-tv}
V.~De~Bortoli.
\newblock Convergence of denoising diffusion models under the manifold hypothesis.
\newblock \emph{Transactions on Machine Learning Research}, 2022.

\bibitem[De~Bortoli et~al.(2021)De~Bortoli, Thornton, Heng, and Doucet]{De_Bortoli2021-vi}
V.~De~Bortoli, J.~Thornton, J.~Heng, and A.~Doucet.
\newblock Diffusion {S}chr{\"o}dinger bridge with applications to score-based generative modeling.
\newblock In \emph{Advances in Neural Information Processing Systems}, 2021.

\bibitem[Devroye and Wagner(1979)]{Devroye1979-du}
L.~Devroye and T.~Wagner.
\newblock Distribution-free performance bounds for potential function rules.
\newblock \emph{IEEE Transactions on Information Theory}, 25\penalty0 (5):\penalty0 601--604, Sept. 1979.

\bibitem[Dupuis et~al.(2025{\natexlab{a}})Dupuis, Haddouche, Deligiannidis, and Simsekli]{Dupuis2025-mq}
B.~Dupuis, M.~Haddouche, G.~Deligiannidis, and U.~Simsekli.
\newblock Understanding the generalization error of markov algorithms through poissonization.
\newblock \emph{arXiv [stat.ML]}, Feb. 2025{\natexlab{a}}.

\bibitem[Dupuis et~al.(2025{\natexlab{b}})Dupuis, Shariatian, Haddouche, Durmus, and Simsekli]{Dupuis2025-au}
B.~Dupuis, D.~Shariatian, M.~Haddouche, A.~Durmus, and U.~Simsekli.
\newblock Algorithm- and data-dependent generalization bounds for score-based generative models.
\newblock \emph{arXiv [stat.ML]}, June 2025{\natexlab{b}}.

\bibitem[Eberle(2016)]{Eberle2016-te}
A.~Eberle.
\newblock Reflection couplings and contraction rates for diffusions.
\newblock \emph{Probability Theory and Related Fields}, 166\penalty0 (3):\penalty0 851--886, Dec. 2016.

\bibitem[Eberle and Majka(2019)]{Eberle2019-ai}
A.~Eberle and M.~B. Majka.
\newblock Quantitative contraction rates for {M}arkov chains on general state spaces.
\newblock \emph{Electronic Journal of Probability}, 24:\penalty0 1--36, Jan. 2019.

\bibitem[Esser et~al.(2024)Esser, Kulal, Blattmann, Entezari, M{\"u}ller, Saini, Levi, Lorenz, Sauer, Boesel, et~al.]{esser2024scaling}
P.~Esser, S.~Kulal, A.~Blattmann, R.~Entezari, J.~M{\"u}ller, H.~Saini, Y.~Levi, D.~Lorenz, A.~Sauer, F.~Boesel, et~al.
\newblock Scaling rectified flow transformers for high-resolution image synthesis.
\newblock In \emph{Forty-first international conference on machine learning}, 2024.

\bibitem[Farghly and Rebeschini(2021)]{Farghly2021-xg}
T.~Farghly and P.~Rebeschini.
\newblock Time-independent generalization bounds for {SGLD} in non-convex settings.
\newblock In \emph{Advances in Neural Information Processing Systems}, 2021.

\bibitem[Gatmiry et~al.(2024)Gatmiry, Kelner, and Lee]{Gatmiry2024-ds}
K.~Gatmiry, J.~Kelner, and H.~Lee.
\newblock Learning mixtures of gaussians using diffusion models.
\newblock \emph{arXiv [cs.LG]}, Apr. 2024.

\bibitem[Hardt et~al.(2016)Hardt, Recht, and Singer]{Hardt2016-ng}
M.~Hardt, B.~Recht, and Y.~Singer.
\newblock Train faster, generalize better: Stability of stochastic gradient descent.
\newblock In \emph{International Conference on Machine Learning}, pages 1225--1234, 2016.

\bibitem[Haussmann and Pardoux(1986)]{Haussmann1986-hw}
U.~G. Haussmann and E.~Pardoux.
\newblock Time reversal of diffusions.
\newblock \emph{The Annals of Probability}, 14\penalty0 (4):\penalty0 1188--1205, 1986.

\bibitem[Ho et~al.(2020)Ho, Jain, and Abbeel]{Ho2020-pq}
J.~Ho, A.~Jain, and P.~Abbeel.
\newblock Denoising diffusion probabilistic models.
\newblock In \emph{Advances in Neural Information Processing Systems}, 2020.

\bibitem[Huang et~al.(2024)Huang, Wei, and Chen]{Huang2024-ky}
Z.~Huang, Y.~Wei, and Y.~Chen.
\newblock Denoising diffusion probabilistic models are optimally adaptive to unknown low dimensionality.
\newblock \emph{arXiv [cs.LG]}, Oct. 2024.

\bibitem[Hyv{\"a}rinen(2005)]{Hyvarinen2005-zz}
A.~Hyv{\"a}rinen.
\newblock Estimation of {Non-Normalized} statistical models by score matching.
\newblock \emph{Journal of Machine Learning Research}, 6\penalty0 (24):\penalty0 695--709, 2005.

\bibitem[Karras et~al.(2022)Karras, Aittala, Aila, and Laine]{Karras2022-us}
T.~Karras, M.~Aittala, T.~Aila, and S.~Laine.
\newblock Elucidating the design space of diffusion-based generative models.
\newblock In {S. Koyejo and S. Mohamed and A. Agarwal and D. Belgrave and K. Cho and A. Oh}, editor, \emph{Advances in Neural Information Processing Systems}, 2022.

\bibitem[Karras et~al.(2024)Karras, Aittala, Lehtinen, Hellsten, Aila, and Laine]{Karras2024-cw}
T.~Karras, M.~Aittala, J.~Lehtinen, J.~Hellsten, T.~Aila, and S.~Laine.
\newblock Analyzing and improving the training dynamics of diffusion models.
\newblock In \emph{Proceedings of the IEEE/CVF Conference on Computer Vision and Pattern Recognition (CVPR)}, pages 24174--24184, June 2024.

\bibitem[Kearns and Ron(1999)]{Kearns1999-jq}
M.~Kearns and D.~Ron.
\newblock Algorithmic stability and sanity-check bounds for leave-one-out cross-validation.
\newblock \emph{Neural Computation}, 11\penalty0 (6):\penalty0 1427--1453, 1999.

\bibitem[Lee et~al.(2022)Lee, Lu, and Tan]{Lee2022-yu}
H.~Lee, J.~Lu, and Y.~Tan.
\newblock Convergence for score-based generative modeling with polynomial complexity.
\newblock \emph{Advances in Neural Information Processing Systems}, 35:\penalty0 22870--22882, 2022.

\bibitem[Li et~al.(2023)Li, Li, Zhang, and Bian]{Li2023-cf}
P.~Li, Z.~Li, H.~Zhang, and J.~Bian.
\newblock On the generalization properties of diffusion models.
\newblock In \emph{Advances in Neural Information Processing Systems}, 2023.

\bibitem[Loaiza-Ganem et~al.(2024)Loaiza-Ganem, Ross, Hosseinzadeh, Caterini, and Cresswell]{Loaiza-Ganem2024-eu}
G.~Loaiza-Ganem, B.~L. Ross, R.~Hosseinzadeh, A.~L. Caterini, and J.~C. Cresswell.
\newblock Deep generative models through the lens of the manifold hypothesis: A survey and new connections.
\newblock \emph{arXiv [cs.LG]}, Apr. 2024.

\bibitem[Majka et~al.(2020)Majka, Mijatovic, and Szpruch]{Majka2020-lp}
M.~B. Majka, A.~Mijatovic, and L.~Szpruch.
\newblock Nonasymptotic bounds for sampling algorithms without log-concavity.
\newblock \emph{Ann. Appl. Probab.}, 30\penalty0 (4):\penalty0 1534--1581, Aug. 2020.

\bibitem[Mou et~al.(2018)Mou, Wang, Zhai, and Zheng]{Mou2018-bz}
W.~Mou, L.~Wang, X.~Zhai, and K.~Zheng.
\newblock Generalization bounds of sgld for non-convex learning: Two theoretical viewpoints.
\newblock In {Bubeck, Sébastien and Perchet, Vianney and Rigollet, Philippe}, editor, \emph{Proceedings of the 31st Conference On Learning Theory}, volume~75 of \emph{Proceedings of Machine Learning Research}, pages 605--638. PMLR, 2018.

\bibitem[Nesterov(2018)]{Nesterov2018-yb}
Y.~Nesterov.
\newblock \emph{Lectures on convex optimization}.
\newblock Springer optimization and its applications. Springer International Publishing, Cham, Switzerland, 2 edition, Dec. 2018.

\bibitem[Neu et~al.(2021)Neu, Dziugaite, Haghifam, and Roy]{Neu2021-rh}
G.~Neu, G.~K. Dziugaite, M.~Haghifam, and D.~M. Roy.
\newblock Information-theoretic generalization bounds for stochastic gradient descent.
\newblock In \emph{Proceedings of Thirty Fourth Conference on Learning Theory}, volume 134 of \emph{Proceedings of Machine Learning Research}, pages 3526--3545. PMLR, 2021.

\bibitem[Oko et~al.(2023)Oko, Akiyama, and Suzuki]{Oko2023-sa}
K.~Oko, S.~Akiyama, and T.~Suzuki.
\newblock Diffusion models are minimax optimal distribution estimators.
\newblock In \emph{International Conference on Machine Learning}, 2023.

\bibitem[Pensia et~al.(2018)Pensia, Jog, and Loh]{Pensia2018-md}
A.~Pensia, V.~Jog, and P.-L. Loh.
\newblock Generalization error bounds for noisy, iterative algorithms.
\newblock In \emph{2018 IEEE International Symposium on Information Theory (ISIT)}, pages 546--550, June 2018.

\bibitem[Pidstrigach(2022)]{Pidstrigach2022-qr}
J.~Pidstrigach.
\newblock Score-based generative models detect manifolds.
\newblock In \emph{Advances in Neural Information Processing Systems}, 2022.

\bibitem[Potaptchik et~al.(2024)Potaptchik, Azangulov, and Deligiannidis]{Potaptchik2024-ue}
P.~Potaptchik, I.~Azangulov, and G.~Deligiannidis.
\newblock Linear convergence of diffusion models under the manifold hypothesis.
\newblock \emph{arXiv preprint arXiv:2410.09046}, 2024.

\bibitem[Ramesh et~al.(2022)Ramesh, Dhariwal, Nichol, Chu, and Chen]{Ramesh2022-ac}
A.~Ramesh, P.~Dhariwal, A.~Nichol, C.~Chu, and M.~Chen.
\newblock Hierarchical text-conditional image generation with clip latents.
\newblock \emph{arXiv preprint arXiv:2204.06125}, 2022.

\bibitem[Rombach et~al.(2022)Rombach, Blattmann, Lorenz, Esser, and Ommer]{Rombach2022-kn}
R.~Rombach, A.~Blattmann, D.~Lorenz, P.~Esser, and B.~Ommer.
\newblock High-resolution image synthesis with latent diffusion models.
\newblock In \emph{Proceedings of the {IEEE/CVF} Conference on Computer Vision and Pattern Recognition ({CVPR})}. IEEE, June 2022.

\bibitem[Saharia et~al.(2022)Saharia, Chan, Saxena, Li, Whang, Denton, Ghasemipour, Gontijo~Lopes, Karagol~Ayan, Salimans, Ho, Fleet, and Norouzi]{Saharia2022-zc}
C.~Saharia, W.~Chan, S.~Saxena, L.~Li, J.~Whang, E.~L. Denton, K.~Ghasemipour, R.~Gontijo~Lopes, B.~Karagol~Ayan, T.~Salimans, J.~Ho, D.~J. Fleet, and M.~Norouzi.
\newblock Photorealistic {Text-to-Image} diffusion models with deep language understanding.
\newblock In \emph{Advances in Neural Information Processing Systems}, 2022.

\bibitem[Shah et~al.(2023)Shah, Chen, and Klivans]{Shah2023-ka}
K.~Shah, S.~Chen, and A.~Klivans.
\newblock Learning mixtures of {G}aussians using the {DDPM} objective.
\newblock In \emph{Advances in Neural Information Processing Systems}, 2023.

\bibitem[Sietsma and Dow(1991)]{Sietsma1991-zo}
J.~Sietsma and R.~J.~F. Dow.
\newblock Creating artificial neural networks that generalize.
\newblock \emph{Neural Networks}, 4\penalty0 (1):\penalty0 67--79, 1991.

\bibitem[Simsekli et~al.(2020)Simsekli, Sener, Deligiannidis, and Erdogdu]{NEURIPS2020_37693cfc}
U.~Simsekli, O.~Sener, G.~Deligiannidis, and M.~A. Erdogdu.
\newblock Hausdorff dimension, heavy tails, and generalization in neural networks.
\newblock In H.~Larochelle, M.~Ranzato, R.~Hadsell, M.~Balcan, and H.~Lin, editors, \emph{Advances in Neural Information Processing Systems}, volume~33, pages 5138--5151. Curran Associates, Inc., 2020.
\newblock URL \url{https://proceedings.neurips.cc/paper_files/paper/2020/file/37693cfc748049e45d87b8c7d8b9aacd-Paper.pdf}.

\bibitem[Sohl-Dickstein et~al.(2015)Sohl-Dickstein, Weiss, Maheswaranathan, and Ganguli]{sohl2015deep}
J.~Sohl-Dickstein, E.~Weiss, N.~Maheswaranathan, and S.~Ganguli.
\newblock Deep unsupervised learning using nonequilibrium thermodynamics.
\newblock In \emph{International Conference on Machine Learning}, 2015.

\bibitem[Song and Kingma(2021)]{Song2021-aj}
Y.~Song and D.~P. Kingma.
\newblock How to train your energy-based models.
\newblock \emph{arXiv preprint arXiv:2101. 03288}, 2021.

\bibitem[Song et~al.(2021{\natexlab{a}})Song, Durkan, Murray, and Ermon]{Song2021-rf}
Y.~Song, C.~Durkan, I.~Murray, and S.~Ermon.
\newblock Maximum likelihood training of {Score-Based} diffusion models.
\newblock In \emph{Advances in Neural Information Processing Systems}, 2021{\natexlab{a}}.

\bibitem[Song et~al.(2021{\natexlab{b}})Song, Sohl-Dickstein, Kingma, Kumar, Ermon, and Poole]{Song2021-uf}
Y.~Song, J.~Sohl-Dickstein, D.~P. Kingma, A.~Kumar, S.~Ermon, and B.~Poole.
\newblock {Score-Based} generative modeling through stochastic differential equations.
\newblock In \emph{9International Conference on Learning Representations}, 2021{\natexlab{b}}.

\bibitem[Srivastava et~al.(2014)Srivastava, Hinton, Krizhevsky, Sutskever, and Salakhutdinov]{Srivastava2014-eh}
N.~Srivastava, G.~Hinton, A.~Krizhevsky, I.~Sutskever, and R.~Salakhutdinov.
\newblock Dropout: A simple way to prevent neural networks from overfitting.
\newblock \emph{Journal of Machine Learning Research}, 15\penalty0 (56):\penalty0 1929--1958, 2014.

\bibitem[Tang and Yang(2024)]{Tang2024-qk}
R.~Tang and Y.~Yang.
\newblock Adaptivity of diffusion models to manifold structures.
\newblock In \emph{International Conference on Artificial Intelligence and Statistics}, 2024.

\bibitem[Vary et~al.(2024)Vary, Martínez-Rubio, and Rebeschini]{Vary2024-ao}
S.~Vary, D.~Martínez-Rubio, and P.~Rebeschini.
\newblock Black-box uniform stability for non-euclidean empirical risk minimization.
\newblock \emph{arXiv [cs.LG]}, Dec. 2024.

\bibitem[Wainwright(2019)]{Wainwright2019-rz}
M.~J. Wainwright.
\newblock \emph{High-Dimensional Statistics: A Non-Asymptotic Viewpoint}.
\newblock Cambridge Series in Statistical and Probabilistic Mathematics. Cambridge University Press, 2019.

\bibitem[Wang(1997)]{Wang1997-il}
F.-Y. Wang.
\newblock Logarithmic sobolev inequalities on noncompact riemannian manifolds.
\newblock \emph{Probability Theory and Related Fields}, 109\penalty0 (3):\penalty0 417--424, Nov. 1997.

\bibitem[Wang et~al.(2024)Wang, Zhang, Zhang, Chen, Ma, and Qu]{Wang2024-jc}
P.~Wang, H.~Zhang, Z.~Zhang, S.~Chen, Y.~Ma, and Q.~Qu.
\newblock Diffusion models learn low-dimensional distributions via subspace clustering.
\newblock \emph{arXiv [cs.LG]}, Sept. 2024.

\bibitem[Watson et~al.(2023)Watson, Juergens, Bennett, Trippe, Yim, Eisenach, Ahern, Borst, Ragotte, Milles, et~al.]{watson2023novo}
J.~L. Watson, D.~Juergens, N.~R. Bennett, B.~L. Trippe, J.~Yim, H.~E. Eisenach, W.~Ahern, A.~J. Borst, R.~J. Ragotte, L.~F. Milles, et~al.
\newblock De novo design of protein structure and function with {RFdiffusion}.
\newblock \emph{Nature}, 620\penalty0 (7976):\penalty0 1089--1100, 2023.

\bibitem[Williams et~al.(2024)Williams, Campbell, Doucet, and Syed]{Williams2024-qj}
C.~Williams, A.~Campbell, A.~Doucet, and S.~Syed.
\newblock Score-optimal diffusion schedules.
\newblock In \emph{The Thirty-eighth Annual Conference on Neural Information Processing Systems}, 2024.

\bibitem[Yang(2022)]{Yang2022-hq}
H.~Yang.
\newblock A mathematical framework for learning probability distributions.
\newblock \emph{arXiv preprint arXiv:2212.11481}, 2022.

\bibitem[Zhang et~al.(2021)Zhang, Bengio, Hardt, Brain, Recht, and Vinyals]{Zhang2021-ly}
C.~Zhang, S.~Bengio, M.~Hardt, G.~Brain, B.~Recht, and O.~Vinyals.
\newblock Understanding deep learning (still) requires rethinking generalization.
\newblock \emph{Communications of the ACM}, 64\penalty0 (3):\penalty0 107--115, 2021.

\end{thebibliography}




\newpage
\appendix

\section{Further background}\label{app:prelim}
We begin with some further details on notation and lemmas used throughout this work and provide proofs for the lemmas in Section \ref{sec:background}.

\subsection{Random score matching algorithms}\label{app:random_sm}
We begin with some additional details on how random score matching algorithms are defined in this work. Recalling the probability space \((\Omega, \mathcal{F}, \prob)\), we define the set of random score functions,
\begin{equation*}
    \mathcal{S} := \Big \{ s: \R^d \times [0, T] \times \Omega: s(\cdot, \cdot, \omega) \in L^0(\R^d \times [0, T]; \R^d) \Big \}.
\end{equation*}
For any random score matching algorithm \(A_\sm: (\cup_{N=1}^{\infty} (\R^d)^{\otimes N}) \times \Omega \to L^0(\R^d \times [0, T]; \R^d)\), we use \(A_\sm(S)\) as shorthand for the random score function \((\omega, x, t) \mapsto A_\sm(S, \omega)(x, t)\) belonging to \(\mathcal{S}\).

Given two random score functions \(s, s'\), let \(\Gamma(s, s')\) denote the set of all couplings of these functions which we define as,
\begin{equation*}
    \Gamma(s, s') := \Big \{ (\tilde{s}, \tilde{s}') \in \mathcal{S} \times \mathcal{S}: \tilde{s} \simeq s, \tilde{s}' \simeq s' \Big \},
\end{equation*}
where \(\tilde{s} \simeq s\) denotes the fact that for any bounded measurable test function \(\phi: L^0(\R^d \times [0, T]; \R^d) \to \R\), it holds that,
\begin{equation*}
    \int \phi(s(\cdot, \cdot, \omega)) d \prob = \int \phi(\tilde{s}(\cdot, \cdot, \omega)) d \prob.
\end{equation*}

\subsection{Preliminary lemmas}\label{app:prelim_score}
For the score matching loss bound, we begin with the fact that the score matching loss is equivalent to the denoising score matching loss up to an added constant \cite{Song2021-uf, Hyvarinen2005-zz}.

\begin{lemma}\label{lem:score_fn_decomp}
For any \(t > 0\), \(y \in \R^d\), we have
\begin{equation}
    \nabla \log p_t(y) = \frac{\mu_t \E[X_0|X_t=y] - y}{\sigma_t^2}, \qquad \nabla \log \hat{p}_t(y) = \frac{\mu_t \E[\hat{X}_0|\hat{X}_t=y, S] - y}{\sigma_t^2}.\label{eq:score_fn_decomp_2}
\end{equation}
\end{lemma}
\begin{proof}
 We begin by showing that the conditional score is an unbiased estimate of \(\nabla \log p_t\). For any \(x \in \R^d, t > 0\), we have
\begin{align*}
    \E[\nabla \log p_{t|0}(X_t|X_0) | X_t = x] &= \int \nabla_x \log p_{t|0}(x|y) \, p_{0|t}(y|x) dy\\
    &= \int \nabla \log p_{t|0}(x|y) \, \frac{p_{t|0}(x|y) p_0(y)}{p_{t}(x)} dy\\
    &= \int \nabla p_{t|0}(x|y) \, \frac{p_0(y)}{p_{t}(x)} dy.
\end{align*}
Therefore, using the exchangeability of gradients and integrals (note that \(p_{t|0}\) is \(C^\infty\)), we arrive at
\begin{align}
    \E[\nabla \log p_{t|0}(X_t|X_0) | X_t = x] &= \frac{\nabla p_t(x)}{p_{t}(x)}\\
    &= \nabla \log p_t(x). \label{eq:dsm_decomposition_2}
\end{align}
Alternatively, using \eqref{eq:cond_score_form}, we obtain that the left-hand side takes the form,
\begin{equation*}
    \E[\nabla \log p_{t|0}(X_t|X_0) | X_t = x] = \frac{\mu_t \E[X_0|X_t = x] - x}{\sigma_t^2},
\end{equation*}
completing the proof of the first equality in \eqref{eq:score_fn_decomp_2}. For the second equality, concerning that empirical score function, the proof follows similarly once the empirical measure \(\frac{1}{N} \sum_{i=1}^N \delta_{x_i}\) is considered in place of \(\nud\).
\end{proof}

\begin{lemma}\label{lem:dsm_decomposition}
For any integrable score function \(s\), it holds that
\begin{equation*}
    \ell_\dsm(s; \tau) = \ell_\sm(s; \tau) + C_\sm,
\end{equation*}
where, given \(s^\star(x, t) := \nabla \log p_t(x)\), we define
\begin{equation}
    C_\sm := \int \frac{\mu_t^2}{\sigma_t^4} \E[\tr \cov(X_0|X_t)] \tau(dt) = \ell_\dsm(s^\star; \tau).\label{eq:dsm_min}
\end{equation}
\end{lemma}
\begin{proof}
Let \(s\) be any score function. Using the equality in \eqref{eq:dsm_decomposition_2}, we obtain the following bias-variance decomposition of \(\ell_\dsm(s; \tau)\):
\begin{align*}
    \ell_\dsm(s; \tau) &= \int \E \Big [ \| s(X_t, t) - \nabla \log p_{t|0}(X_t|X_0) \|^2 \Big ] \tau(dt)\\
    &= \int \E \Big [ \| s(X_t, t) - \nabla \log p_{t}(X_t) \|^2 \Big ] \tau(dt) + \int \E \Big [ \| \nabla \log p_{t|0}(X_t|X_0) - \nabla \log p_{t}(X_t) \|^2 \Big ] \tau(dt)\\
    &= \ell_\sm(s; \tau) + \int \E \Big [ \tr \cov \Big ( \nabla \log p_{t|0}(X_t|X_0) \Big | X_t \Big ) \Big ] \tau(dt).
\end{align*}
Once we note that,
\begin{align*}
    \tr \cov \Big ( \nabla \log p_{t|0}(X_t|X_0) \Big | X_t \Big ) &= \tr \cov \bigg ( \frac{\mu_t X_0 - x}{\sigma_t^2} \bigg | X_t \bigg )\\
    &=\frac{\mu_t^2}{\sigma_t^4} \tr \cov(X_0|X_t),
\end{align*}
we obtain the bound \(\ell_\dsm(s; \tau) = \ell_\sm(s; \tau) + C_\sm\) from the statement. To derive the equality \(C_\sm = \ell_\dsm(s^\star; \tau)\), we use that \(\ell_\sm(s^\star; \tau) = 0\) and so we obtain \(\ell_\dsm(s^\star; \tau) = 0 + C_\sm\).
\end{proof}

Similarly, there is an equivalence between the empirical forms of the denoising score matching loss and the score matching loss,
\begin{equation}\label{eq:dsm_sm_emp_2}
    \hat{\ell}_\dsm(s; S, \tau) = \hat{\ell}_\sm(s; S, \tau) + \hat{C}_\sm,
\end{equation}
where
\begin{equation}
    \hat{C}_\sm := \int \frac{\mu_t^2}{\sigma_t^4} \E[\tr \cov(\hat{X}_0|\hat{X}_t, S)|S] \tau(dt) = \hat{\ell}_\dsm(\hat{s}^\star; S, \tau),\label{eq:emp_dsm_min}
\end{equation}
and \(\hat{s}^\star(x, t) = \nabla \hat{p}_t(x)\). This follows immediately from the above proof once the empirical measure \(\frac{1}{N} \sum_{i=1}^N \delta_{x_i}\) is considered in place of \(\nud\). This effectively completes the proof of Lemma \ref{lem:dsm_sm_emp} in Section \ref{sec:prelim_overfitting}.

\pasteenv{lemma}{lem:dsm_sm_emp}
\begin{proof}
The proof follows nearly immediately from \eqref{eq:dsm_sm_emp_2}. Since \(p_{t|0}\) is \(C^\infty\), \(\nabla \log p_{t|0}\) is measurable and thus its empirical average \(\nabla \log \hat{p}_t\) must be also. Therefore,  the score function \(s^\star(x, t) = \nabla \log \hat{p}_t(x)\) satisfies \(\hat{s}^\star \in L^0(\R^d \times [0, T]; \R^d)\) as well as,
\begin{equation*}
    \hat{\ell}_\sm(\hat{s}^\star; S, \tau) = 0.
\end{equation*}
Now let \(s \in L^0(\R^d \times [0, T]; \R^d)\) be any minimiser of \(\hat{\ell}_\dsm(\cdot; S, \tau)\). Through the equivalence of \(\hat{\ell}_\dsm\) and \(\hat{\ell}_\sm\) up to a constant, it follows that \(s\) must also be a minimiser of \(\hat{\ell}_\sm(\cdot; S, \tau)\) and, due to the existence of \(\hat{s}^\star\), must satisfy \(\hat{\ell}_\sm(s; S, \tau) = 0\) also. Letting \(t \in \supp(t)\), we note that since \(t > 0\), we must have that \(p_{t|0}\) has full support and thus, \(s(\cdot, t) = s^\star(\cdot, t)\) almost everywhere.
\end{proof}

\subsection{Manifolds}\label{sec:manifold}
We also introduce some basic properties of smooth manifolds, primarily referencing \cite{Aamari2019-gk}. We define the manifold reach and include a known property of this quantity.

\begin{definition}
The reach of a set \(A \subset \R^d\), is defined by \(\tau_A = \inf_{p \in A} d(p, Med(A))\), where we define the set,
\begin{equation*}
    Med(A) = \Big \{z \in \R^d: \exists p, q \in A \text{ s.t. } p \neq q, \|p - z\| = \|q - z\| \Big \}.
\end{equation*}
\end{definition}

\begin{lemma}\label{lem:manifold_ball_bound}
Suppose that the measure \(\mu\) is supported on a manifold \(M\) with reach \(\tau_M > 0\) and dimension \(d^*\). Then, for any \(r \leq \tau_M\), we have
\begin{equation*}
    \mu(B_r(x)) \geq \Big | \inf_{B_r(x)} p_\mu \Big | \, r^{d^*},
\end{equation*}
where \(p_\mu\) denotes the density of \(\mu\) with respect to the volume measure on \(M\).
\end{lemma}

For the proof of this lemma, we refer to the proof of Proposition 4.3 in \cite{Aamari2019-gk} or Lemma III.23 in \cite{Aamari2017-he}.
\section{Proofs for the generalisation gap bounds}\label{app:gen_gap_proofs}
\label{sec:proof_sm_bounds}
We now provide provide the proof of theorem \ref{thm:dsm_stability} that bound the generalisation gap under score stability guarantees. For the sake of brevity, throughout this section we suppress the notation for the time weighting, for example, using the shorthand \(\hat{\ell}_\sm(s; S)\) in place of \(\hat{\ell}_\sm(s; S, \tau)\).

\pasteenv{theorem}{thm:dsm_stability}
\begin{proof}
Setting \(\hat{s} = A_\sm(S)\) and \(\hat{s}^i = A_\sm(S^i)\), we use the property that \((\hat{s}, \tilde{x})\) and \((\hat{s}^i, x_i)\) are distributed identically to obtain that,
\begin{align*}
    \E [\ell_{\dsm}(\hat{s}; \tau)] &= \E[\hat{\ell}_\dsm(\hat{s}; \{\tilde{x}\})]\\
    &= \E \Big [ \frac{1}{N} \sum_{i=1}^N \hat{\ell}_\dsm(\hat{s}^i; \{x_i\}) \Big ]\\
    &= \E \Big [ \frac{1}{N} \sum_{i=1}^N \int \E_{X_t} [\|\hat{s}^i(X_t, t, \omega) - \nabla \log p_{t|0}(X_t|x_i)\|^2 | X_0 = x_i, S] \, \tau(dt) \, \bigg ].
\end{align*}
Therefore, it follows from the triangle inequality in \(L^2\)-norm that
\begin{align*}
    \Big | \E [\ell_{\dsm}(\hat{s}; \tau)]^{1/2} - \E [\hat{\ell}_{\dsm}(\hat{s}; S)]^{1/2} \Big | &\leq \E \bigg [ \frac{1}{N} \sum_{i=1}^N \int \E [\|\hat{s}(X_t, t) - \hat{s}^i(X_t, t)\|^2 | X_0 = x_i, S] \, \tau(dt) \, \bigg ]^{1/2}
\end{align*}
Note that if the algorithm \(A_\sm\) is stochastic, the right-hand side would hold regardless of how \(\hat{s}|S, \tilde{x}\) and \(\hat{s}^i|S, \tilde{x}\) were coupled. Therefore the most efficient coupling can be chosen, leading to the bound,
\begin{align}
    \Big | \E [\ell_{\dsm}(\hat{s}; \tau)]^{1/2} - \E [\hat{\ell}_{\dsm}(\hat{s}; S)]^{1/2} \Big | &\leq \E \bigg [ \inf_{(\hat{s}, \hat{s}^i) \in \Gamma_i} \frac{1}{N} \sum_{i=1}^N \int \E [\|\hat{s}(X_t, t) - \hat{s}^i(X_t, t)\|^2 | X_0 = x_i, S] \, \tau(dt) \, \bigg ]^{1/2} \nonumber\\
    &\leq \stab, \label{eq:dsm_stability_4}
\end{align}
completing the proof of the bound in \eqref{eq:dsm_stability_0}.

To obtain the bound in \eqref{eq:dsm_stability_1}, we use Lemma \ref{lem:dsm_decomposition} to derive
\begin{align}
    \E[\ell_{\sm}(\hat{s}; \tau)] &= \E[\hat{\ell}_{\sm}(\hat{s}; S)] + \E[\ell_{\dsm}(\hat{s}; \tau) - \hat{\ell}_{\dsm}(\hat{s}; S)] + \E \big [ \hat{\ell}_{\dsm}(\nabla \log \hat{p}_t; S) \big ]\nonumber\\
    & \qquad - \ell_{\dsm}(\nabla \log p_t; \tau).\label{eq:dsm_stability_7}
\end{align}
Since \(\hat{\ell}_\dsm(\cdot; S)\) is a unbiased estimator of \(\ell_\sm(\cdot; \tau)\), we have that
\begin{equation}\label{eq:dsm_stability_5}
    \ell_{\dsm}(\nabla \log p_t; \tau) = \E[\hat{\ell}_{\dsm}(\nabla \log p_t; S)] \geq \E[\hat{\ell}_{\dsm}(\nabla \log \hat{p}_t; S)],
\end{equation}
where the inequality follows from the fact that \(\nabla \log \hat{p}_t\) minimises \(\hat{\ell}_\dsm\). Furthermore, using \eqref{eq:dsm_stability_4}, we deduce the bound,
\begin{align}
    |\E[\ell_{\dsm}(\hat{s}; \tau)& - \hat{\ell}_{\dsm}(\hat{s}; S)]| \nonumber\\
    &= \Big ( \E[\ell_{\dsm}(\hat{s}; \tau)]^{1/2} + \E[\hat{\ell}_{\dsm}(\hat{s}; S)]^{1/2} \Big ) \Big | \E[\ell_{\dsm}(\hat{s}; S)]^{1/2} - \E[\hat{\ell}_{\dsm}(\hat{s}; S)]^{1/2} \Big | \nonumber\\
    &\leq \Big ( 2 \E[\hat{\ell}_{\dsm}(\hat{s}; S)]^{1/2} + \stab \Big ) \stab \nonumber\\
    &= 2 \stab \E[\hat{\ell}_{\dsm}(\hat{s}; S)]^{1/2} + \stab^2.\label{eq:dsm_stability_6}
\end{align}
Thus, substituting \eqref{eq:dsm_stability_5} and \eqref{eq:dsm_stability_6} in to \eqref{eq:dsm_stability_7} recovers the bound in \eqref{eq:dsm_stability_1} in the statement.
\end{proof}

We obtain upper bounds relying on the fact that the constant separating the score matching loss from the denoising score matching loss is larger on average in the empirical case. One could obtain lower bounds through our techniques but this would require an analysis of the rate of convergence of this constant which is beyond the scope of this paper.
\section{Proofs for stability of empirical denoising score matching}\label{app:erm}
In this section, we provide the proof for Theorem \ref{thm:dsm_stability}, where the algorithm that minimises \(\hat{\ell}_\dsm(\cdot; S, \tau)\) over some class of score functions \(\mathcal{H}\) is shown to be score stable.

\subsection{On-average stability of the ERM algorithm}
We begin with an important lemma that shows that under minimal assumptions, \(\hat{s} = A_\erm(S)\) and \(\hat{s}^i = A_\erm(S)\) are close in \(L^2\) space, averaged over the full dataset. The first half of this proof utilises the fact that \(\hat{\ell}_\dsm\) is 1-strongly convex in a weighted \(L^2\) space, exploiting a well-known relationship between strong-convexity and algorithmic stability (e.g. see \citep{Bousquet2002-ox, Charles2018-gl, Vary2024-ao, Attia2022-ee}).

\begin{lemma}\label{lem:erm_avg_stability}
Suppose that \(A_\erm\) is score stable with constant \(\stab\), then for any \(i \in [N]\), we obtain,
\begin{equation}
    \E \bigg [ \int \int \|\hat{s}^i(y, t) - \hat{s}(y, t)\|^2 \, \hat{p}_t(dy) \, \tw(dt) \bigg ] \leq 8 \E[\hat{\ell}_\sm(\hat{s})] + \frac{8}{N} \stab ( C_\sm^{1/2} + \stab)\label{eq:erm_stability_4}
\end{equation}
where \(\hat{s} = A_{\operatorname{erm}}(S), \hat{s}^i = A_{\operatorname{erm}}(S)\).
\end{lemma}
\begin{proof}
Choose \(i \in [N]\) and let \(\hat{s} = A_{\operatorname{erm}}(S), \hat{s}^i = A_{\operatorname{erm}}(S^i)\) so that \(\hat{s} \in \operatorname{argmin}_{\mathcal{H}} \hat{\ell}_\dsm(\cdot; S, \tau), \hat{s}^i \in \operatorname{argmin}_{\mathcal{H}} \hat{\ell}_\dsm(\cdot; S^i, \tau)\).
The proof begins with the following simple expression, that holds for all \(j \in [N]\):
\begin{align*}
    & 2 \int \big \langle \hat{s}^i(y, t) - \hat{s}(y, t), \hat{s}^i - \nabla \log p_{t|0}(y|x_j) \big \rangle \, p_{t|0}(dy|x_j)\\
    & \qquad = \int \|\hat{s}^i(y, t) - \nabla \log p_{t|0}(y|x_j)\|^2 \, p_{t|0}(dy|x_j) - \int \|\hat{s}(y, t) - \nabla \log p_{t|0}(y|x_j)\|^2 \, p_{t|0}(dy|x_j)\\
    & \qquad \qquad + \int \|\hat{s}^i(y, t) - \hat{s}(y, t)\|^2 \, p_{t|0}(dy|x_j).
\end{align*}
By averaging over \(j \in [N]\) and integrating with respect to \(\tau(dt)\), we arrive at the upper bound,
\begin{align}
    &\frac{2}{N} \sum_{j \in [N]} \int \int \big \langle \hat{s}^i(y, t) - \hat{s}(y, t), \hat{s}^i - \nabla \log p_{t|0}(y|x_j) \big \rangle \, p_{t|0}(dy|x_j) \, \tw(dt)\nonumber\\
    & \qquad = \hat{\ell}_\dsm(\hat{s}^i; S, \tau) - \hat{\ell}_\dsm(\hat{s}; S, \tau) + \int \int \|\hat{s}^i(y, t) - \hat{s}(y, t)\|^2 \, \hat{p}_t(dy) \, \tw(dt)\nonumber\\
    & \qquad \geq \int \int \|\hat{s}^i(y, t) - \hat{s}(y, t)\|^2 \, \hat{p}_t(dy) \, \tw(dt),\label{eq:erm_stability_2}
\end{align}
where the inequality follows from the fact that \(\hat{\ell}_\dsm(\hat{s}; S, \tau) \leq \hat{\ell}_\dsm(s; S, \tau)\) for any score function \(s \in \mathcal{H}\). Additionally, the left-hand side is upper bounded using the Cauchy-Schwarz inequality to obtain,
\begin{align}
    &\frac{2}{N} \sum_{x \in S} \int \int \big \langle \hat{s}^i(y, t) - \hat{s}(y, t), \hat{s}^i - \nabla \log p_{t|0}(y|x) \big \rangle \, p_{t|0}(dy|x) \, \tw(dt)\nonumber\\
    & \qquad = \frac{2}{N} \sum_{x \in S^i} \int \int \big \langle \hat{s}^i(y, t) - \hat{s}(y, t), \hat{s}^i(y, t) - \nabla \log p_{t|0}(y|x) \big \rangle \, p_{t|0}(dy|x) \, \tw(dt)\nonumber\\
    & \qquad \qquad + \frac{2}{N} \int \int \big \langle \hat{s}^i(y, t) - \hat{s}(y, t), \hat{s}^i(y, t) - \nabla \log p_{t|0}(y|x_i) \big \rangle \, p_{t|0}(dy|x_i) \, \tw(dt)\nonumber\\
    & \qquad \qquad  - \frac{2}{N} \int \int \big \langle \hat{s}^i(y, t) - \hat{s}(y, t), \hat{s}^i(y, t) - \nabla \log p_{t|0}(y|\tilde{x}) \big \rangle \, p_{t|0}(dy|\tilde{x}) \, \tw(dt)\nonumber\\
    & \qquad \leq 2 \hat{\ell}_\sm(\hat{s}^i; S^i, \tau)^{1/2} \bigg ( \int \int \|\hat{s}^i(y, t) - \hat{s}(y, t)\|^2 \, \hat{p}^i_t(dy) \, \tw(dt) \bigg )^{1/2}\nonumber\\
    & \qquad \qquad + \frac{2}{N} \hat{\ell}_\dsm(\hat{s}^i; \{x_i\}, \tau)^{1/2} \bigg ( \int \int \| \hat{s}^i(y, t) - \hat{s}(y, t) \|^2 \, p_{t|0}(dy|x_i) \, \tw(dt) \bigg )^{1/2}\nonumber\\
    & \qquad \qquad + \frac{2}{N} \hat{\ell}_\dsm(\hat{s}^i; \{\tilde{x}\}, \tau)^{1/2} \bigg ( \int \int \| \hat{s}^i(y, t) - \hat{s}(y, t) \|^2 \, p_{t|0}(dy|\tilde{x}) \, \tw(dt) \bigg )^{1/2},\label{eq:erm_stability_3}
\end{align}
where \(\hat{p}^i_t(dy) = \frac{1}{N} \sum_{x \in S^i} p_{t|0}(dy|x)\). Combining the expressions in \eqref{eq:erm_stability_2} and \eqref{eq:erm_stability_3} and taking the expectation, we derive the bound,
\begin{align*}
    &\E \bigg [ \int \int \|\hat{s}^i(y, t) - \hat{s}(y, t)\|^2 \, \hat{p}_t(dy) \, \tw(dt) \bigg ]\\
     & \qquad \leq 2 \E[\hat{\ell}_\sm(\hat{s}^i; S^i, \tau)]^{1/2} \E \bigg [ \int \int \|\hat{s}^i(y, t) - \hat{s}(y, t)\|^2 \, \hat{p}^i_t(dy) \, \tw(dt) \bigg ]^{1/2}\nonumber\\
    & \qquad \qquad + \frac{2}{N} \E[\hat{\ell}_\dsm(\hat{s}^i; \{x_i\}, \tau)]^{1/2} \E \bigg [ \int \int \| \hat{s}^i(y, t) - \hat{s}(y, t) \|^2 \, p_{t|0}(dy|x_i) \, \tw(dt) \bigg ]^{1/2}\nonumber\\
    & \qquad \qquad + \frac{2}{N} \E[\hat{\ell}_\dsm(\hat{s}^i; \{\tilde{x}\}, \tau)]^{1/2} \E \bigg [ \int \int \| \hat{s}^i(y, t) - \hat{s}(y, t) \|^2 \, p_{t|0}(dy|\tilde{x}) \, \tw(dt) \bigg ]^{1/2}\\
    & \qquad \leq 2 \E[\hat{\ell}_\sm(\hat{s}; S, \tau)]^{1/2} \E \bigg [ \int \int \|\hat{s}^i(y, t) - \hat{s}(y, t)\|^2 \, \hat{p}_t(dy) \, \tw(dt) \bigg ]^{1/2}\\
    & \qquad \qquad + \frac{2}{N} \stab \Big ( \E[\hat{\ell}_\dsm(\hat{s}; S, \tau)]^{1/2} + \E[\ell_\dsm(\hat{s}, )]^{1/2} \Big ),
\end{align*}
where we recall that \(\stab\) is the stability constant for \(A_\erm\). Here, we have used the fact that \((\hat{s}, S)\) has the same law as \((\hat{s}^i, S^i)\) and also \(\E[\hat{\ell}_\dsm(\hat{s}^i; \{\tilde{x}\})] = \E[\hat{\ell}_\dsm(\hat{s}; S)]\) and \(\E[\hat{\ell}_\dsm(\hat{s}^i; \{x_i\})] = \E[\ell_\dsm(\hat{s})]\). By solving the quadratic equation, we deduce that the above inequality implies that,
\begin{align*}
    &\E \bigg [ \int \int \|\hat{s}^i(y, t) - \hat{s}(y, t)\|^2 \, \hat{p}_t(dy) \, \tw(dt) \bigg ]\\
    & \qquad \qquad \leq \bigg ( \E[\hat{\ell}_\sm(\hat{s}; S, \tau)]^{1/2} + \sqrt{\E[\hat{\ell}_\sm(\hat{s}; S, \tau)] + \frac{2}{N} \stab (\E[\ell_\dsm(\hat{s}; \tau)]^{1/2} + \E[\hat{\ell}_\dsm(\hat{s}; S, \tau)]^{1/2})} \bigg )^2\\
    & \qquad \qquad \leq 4 \E[\hat{\ell}_\sm(\hat{s}; S, \tau)] + \frac{4}{N} \stab (\E[\ell_\dsm(\hat{s}; \tau)]^{1/2} + \E[\hat{\ell}_\dsm(\hat{s}; S, \tau)]^{1/2}).
\end{align*}
We simplify the above expression further using Theorem \ref{thm:dsm_stability}. Using the stability assumption, it follows from \eqref{eq:dsm_stability_0} that \(\E[\ell_\dsm(\hat{s})]^{1/2} \leq \E[\hat{\ell}_\dsm(\hat{s})]^{1/2} + \varepsilon\). Furthermore, from Lemma \ref{lem:dsm_decomposition}, we have
\begin{align*}
    \E[\hat{\ell}_\dsm(\hat{s})] &= \E[\hat{\ell}_\sm(\hat{s})] + \E[\hat{C}_\sm]\\
    &\leq \E[\hat{\ell}_\sm(\hat{s})] + C_\sm,
\end{align*}
where we recall the definitions of \(\hat{C}_\sm\) and \(C_\sm\) from \eqref{eq:emp_dsm_min} and \eqref{eq:dsm_min} and recall that \(\E[\hat{C}_\sm] \leq C_\sm\) from \eqref{eq:dsm_stability_5}. Thus, from Young's inequality, we obtain the bound
\begin{align*}
    &\E \bigg [ \int \int \|\hat{s}^i(y, t) - \hat{s}(y, t)\|^2 \, \hat{p}_t(dy) \, \tw(dt) \bigg ] \nonumber\\
    & \qquad \qquad \leq 4 \E[\hat{\ell}_\sm(\hat{s})] + \frac{4}{N} \stab (2\E[\hat{\ell}_\sm(\hat{s})]^{1/2} + 2 C_\sm^{1/2} + \stab)\nonumber\\
    & \qquad \qquad \leq 8 \E[\hat{\ell}_\sm(\hat{s})] + \frac{4}{N} \stab ( \stab/N + 2 C_\sm^{1/2} + \stab)\nonumber\\
    & \qquad \qquad \leq 8 \E[\hat{\ell}_\sm(\hat{s})] + \frac{8}{N} \stab ( C_\sm^{1/2} + \stab).
\end{align*}
\end{proof}

\subsection{Proof of Proposition \ref{prop:stability_erm}}

To obtain the stability bound in Proposition \ref{prop:stability_erm}, we convert the result in Lemma \ref{lem:erm_avg_stability}, which is a bound in \(L^2(\hat{p}_t)\), to a bound in \(L^2(p_{t|0}(\cdot|\tilde{x})\) which is required of score stability. For this, we rely on two further lemmas, the first of which is a fundamental property of the Ornstein-Uhlenbeck process, captured by the Harnack inequality of \cite{Wang1997-il} (see Theorem 5.6.1 \cite{Bakry2014-ut}).

\begin{lemma}[Wang's Harnack inequality]\label{lem:wang}
For each positive measurable function \(\phi: \R^d \to \R\), every \(t > 0, p > 1\) and every \(x, y \in \R^d\), it holds that
\begin{equation*}
    \E[\phi(X_t)|X_0 = x] \leq \E[\phi(X_t)^p|X_0 = y]^{1/p} \exp \bigg ( \frac{\mu_t^2 \|x-y\|^2}{2 (p-1) \sigma_t^2} \bigg ).
\end{equation*}
\end{lemma}

This result describes the stability of the diffusion semigroup under changes in initial position and shows that as \(t\) grows, the distribution of \(X_t\) depends less on \(X_0\). The second lemma, for which we provide a proof, controls the empirical measure, 
\begin{equation*}
    \hat{\nu}(dx) = \frac{1}{N} \sum_{i=1}^N \delta_{x_i}(dx),
\end{equation*}
on balls around training examples.

\begin{lemma}\label{lem:emp_prob_bound}
Suppose that Assumption \ref{ass:manifold} is satisfied, then for any \(i \in [N], r \in (0, \uptau_{\text{reach}}]\) and any decreasing function \(\phi: (0, \infty) \to \R_+\), we have the bound
\begin{align*}
    \E \Big [ \phi \Big ( \hat{\nu}(B_r(x_i)) \Big ) \Big ] \leq \phi(N^{-1}) \exp( - c_\nu N^2 r^{d^*} ) + \phi ( c_\nu r^{d^*} / 2 ),
\end{align*}
whenever \(N \geq 4 c_\nu^{-1} r^{-d^*}\), where \(c_\nu = \inf p_\nu\).
\end{lemma}
\begin{proof}
We rewrite the object \(\hat{\nu}(B_r(x_i))\) as an empirical average of Bernoulli random variables
\begin{equation*}
    \hat{\nu}(B_r(x_i)) = \frac{1}{N} \sum_{j=1}^N \mathbbm{1}_{x_j \in B_r(x_i)} = \frac{1}{N} + \frac{1}{N} \sum_{j \neq i} \mathbbm{1}_{x_j \in B_r(x_i)}.
\end{equation*}
When conditioned on \(x_i\), the random variables \((\mathbbm{1}_{x_j \in B_r(x_i)})_{j \neq i}\) are independently and identically distributed Bernoulli random variable with probability \(\mu = \nud(B_r(x_i))\). To utilise concentration of the empirical process, we first rewrite the probability
\begin{align*}
    \prob \Big ( \hat{\nu}(B_r(x_i)) \leq \mu / 2 \Big | x_i \Big ) &\leq \prob \Big ( S_{N-1} \leq \frac{N\mu}{2} - 1 \Big | x_i \Big ), \qquad S_{N-1} = \sum_{j \neq i} \mathbbm{1}_{x_j \in B_r(x_i)}.
\end{align*}
Therefore, by Chernoff's inequality we obtain
\begin{align*}
    \prob \Big ( \hat{\nu}(B_r(x_i)) \leq \mu / 2 \Big | x_i \Big ) &\leq \exp \Big ( - \mu^{-1} (N \mu / 2 - 1)^2 \Big )\\
    &\leq \exp( - N^2 \mu / 16 ),
\end{align*}
where the last bound holds when \(N \geq 4 \mu^{-1}\). Therefore, using the above bound as well as the trivial bound \(\hat{\nu}(B_r(x_i)) \geq N^{-1}\) we apply the law of total expectation to obtain,
\begin{align*}
    \E \Big [ \phi \Big ( \hat{\nu}(B_r(x_i)) \Big ) \Big | x_i \Big ] &= \E \Big [ \phi \Big ( \hat{\nu}(B_r(x_i)) \Big ) \Big | \hat{\nu}(B_r(x_i)) > \mu / 2 \Big ] + \prob \Big ( \hat{\nu}(B_r(x_i)) \leq \mu / 2 \Big | x_i \Big ) \phi(N^{-1})\\
    &\leq \phi ( \mu / 2 ) + \exp( - N^2 \mu / 16 ) \phi(N^{-1}).
\end{align*}
To control \(\mu\), we use Lemma \ref{lem:manifold_ball_bound} which asserts that \(\mu \geq c_\nu r^{d^*}\).
\end{proof}


This now brings us to the proof of the proposition, which we first restate.
\pasteenv{proposition}{prop:stability_erm}
\begin{proof}
We use the shorthand \(\hat{\ell}_\sm(s) = \hat{\ell}_\sm(s; S, \tau), \hat{\ell}_\dsm(s) = \hat{\ell}_\dsm(s; S, \tau), \ell_\sm(s) = \ell_\sm(s; \tau)\) for the sake of brevity. We start from Lemma \ref{lem:erm_avg_stability} which provides a bound on the difference between \(\hat{s}^i\) and \(\hat{s}\) in \(L^2(\hat{p}_t)\) and use it to develop a bound in \(L^2(\hat{p}_{t|0}(\cdot|\tilde{x}))\), as required by score stability. In particular, we define the quantity
\begin{align*}
    \varepsilon^2 = \E \bigg [ \int \int \| \hat{s}^i(y, t) - \hat{s}(y, t) \|^2 \, p_{t|0}(dy|x_i) \, \tw(dt) \bigg ],
\end{align*}
so that, by the symmetric of the algorithm, \(A_\erm\) is score stable with constant \(\varepsilon\) (we have that \(\varepsilon < \infty\) from Assumption \ref{ass:hypoth_diam}. Therefore, from Lemma \ref{lem:erm_avg_stability}, we have
\begin{equation*}
    \E \bigg [ \int \int \|\hat{s}^i(y, t) - \hat{s}(y, t)\|^2 \, \hat{p}_t(dy) \, \tw(dt) \bigg ] \leq 8 \E[\hat{\ell}_\sm(\hat{s})] + \frac{8}{N} \varepsilon ( C_\sm^{1/2} + \varepsilon).
\end{equation*}
We proceed using Lemma \ref{lem:wang} with \(\phi(y)~=~\| \hat{s}^i(y, t) - \hat{s}(y, t) \|^2\) to obtain that for any \(j \in [N]\), \(p > 1\),
\begin{align*}
    &\int \| \hat{s}^i(y, t) - \hat{s}(y, t) \|^2 \, p_{t|0}(dy|x_i)\\
    & \qquad \leq \bigg ( \int \| \hat{s}^i(y, t) - \hat{s}(y, t) \|^{2p} \, p_{t|0}(dy|x_j) \bigg )^{1/p} \exp \bigg ( \frac{\mu_t^2 \|x_i - x_j\|^2}{2(p-1)\sigma_t^2} \bigg )
\end{align*}
Given any subset of the dataset \(B \subset S\) with \(x_i \in B\) we can average over the above bound to obtain,
\begin{align*}
    \int \| \hat{s}^i(y, t)& - \hat{s}(y, t) \|^2 \, p_{t|0}(dy|x_i)\\
    &\leq \frac{1}{|B|} \sum_{x \in B} \bigg ( \int \| \hat{s}^i(y, t) - \hat{s}(y, t) \|^{2p} \, p_{t|0}(dy|x) \bigg )^{1/p} \exp \bigg ( \frac{\mu_t^2 \diam(B)^2}{2(p-1)\sigma_t^2} \bigg )\\
    &\leq \bigg ( \frac{1}{|B|} \sum_{x \in B} \int \| \hat{s}^i(y, t) - \hat{s}(y, t) \|^{2p} \, p_{t|0}(dy|x) \bigg )^{1/p} \exp \bigg ( \frac{\mu_t^2 \diam(B)^2}{2(p-1)\sigma_t^2} \bigg )\\
    &\leq \hat{\nu}(B)^{-1/p} \bigg ( \int \| \hat{s}^i(y, t) - \hat{s}(y, t) \|^{2p} \, \hat{p}_t(dy) \bigg )^{1/p} \exp \bigg ( \frac{\mu_t^2 \diam(B)^2}{2(p-1)\sigma_t^2} \bigg )\\
    &\leq (D_\hilb/\sigma_t^2)^{2(1-1/p)} \hat{\nu}(B)^{-1/p} \bigg ( \int \| \hat{s}^i(y, t) - \hat{s}(y, t) \|^2 \, \hat{p}_t(dy) \bigg )^{1/p} \exp \bigg ( \frac{\mu_t^2 \diam(B)^2}{2(p-1)\sigma_t^2} \bigg ),
\end{align*}
where in the final inequality we use the \(L^\infty\) bound in Assumption \ref{ass:hypoth_diam}. Integrating with respect to \(\tau\) and taking the expectation, we obtain,
\begin{align*}
    \varepsilon^2 &\leq (D_\hilb/\sigma_\epsilon^2)^{2/q} \E \Big [ \hat{\nu}(B)^{-1/p} \Big ] \E \bigg [ \int \int \| \hat{s}^i(y, t) - \hat{s}(y, t) \|^2 \, \hat{p}_t(dy) \tau(dt) \bigg ] \exp \bigg ( \frac{\mu_\epsilon^2 \diam(B)^2}{2(p-1)\sigma_\epsilon^2} \bigg )\\
    &\leq (D_\hilb/\sigma_\epsilon^2)^{2/q} \E \Big [ \hat{\nu}(B)^{-q/p} \Big ]^{1/q} \bigg ( 8 \E[\hat{\ell}_\sm(\hat{s})] + \frac{8}{N} \varepsilon ( C_\sm^{1/2} + \varepsilon) \bigg )^{1/p} \exp \bigg ( \frac{\mu_\epsilon^2 \diam(B)^2}{2(p-1)\sigma_\epsilon^2} \bigg ),
\end{align*}
where we define \(q := (1 - 1/p)^{-1}\). Using Young's inequality, it follows that for any \(\lambda > 0\),
\begin{align*}
    \varepsilon^2 &\leq \frac{D_\hilb^2}{\sigma_\epsilon^4 \lambda^q q} \E \Big [ \hat{\nu}(B)^{-q/p} \Big ] \exp \bigg ( \frac{q \mu_\epsilon^2 \diam(B)^2}{2(p-1)\sigma_\epsilon^2} \bigg ) + \frac{\lambda^p}{p} \bigg ( 8 \E[\hat{\ell}_\sm(\hat{s})] + \frac{8}{N} \varepsilon ( C_\sm^{1/2} + \varepsilon) \bigg ).
\end{align*}
Setting \(\kappa := 8 \lambda^p/pN\), we can rearrange this to obtain the quadratic inequality,
\begin{equation*}
    (1 - \kappa) \varepsilon^2 - C_\sm^{1/2} \kappa \varepsilon \leq \bigg ( \frac{8}{N p \kappa} \bigg )^{q/p} \frac{D_\hilb^2}{\sigma_\epsilon^4 q} \E \Big [ \hat{\nu}(B)^{-q/p} \Big ] \exp \bigg ( \frac{q \mu_\epsilon^2 \diam(B)^2}{2(p-1)\sigma_\epsilon^2} \bigg ) + N \kappa \E[\hat{\ell}_\sm(\hat{s})].
\end{equation*}
Requiring that \(\kappa \leq 1/2\), we solve the quadratic to obtain the inequality,
\begin{equation}
    \frac{\varepsilon^2}{4} \leq C_\sm \kappa^2 + \bigg ( \frac{8}{N p \kappa} \bigg )^{q/p} \frac{D_\hilb^2}{\sigma_\epsilon^4 q} \E \Big [ \hat{\nu}(B)^{-q/p} \Big ] \exp \bigg ( \frac{q \mu_\epsilon^2 \diam(B)^2}{2(p-1)\sigma_\epsilon^2} \bigg ) + N \kappa \E[\hat{\ell}_\sm(\hat{s})].\label{eq:erm_stability_5}
\end{equation}
Next, we optimise \(B\) by setting \(B = B_{\sigma_\epsilon}(x_i) \cap S\). We apply Lemma \ref{lem:emp_prob_bound} with \(\phi(r) = r^{-q/p}\) to obtain that whenever \(\sigma_\epsilon \leq \uptau_{\text{reach}}\) we obtain,
\begin{align*}
    \E \Big [ \hat{\nu}(B)^{-q/p} \Big ] &\leq N^{q/p} \exp( - c_\nu N^2 r^{d^*} ) + \bigg ( \frac{2}{c_\nu r^{d^*}} \bigg )^{q/p}\\
    &\leq 2 \bigg ( \frac{2}{c_\nu \sigma_\epsilon^{d^*}} \bigg )^{q/p},
\end{align*}
where the second inequality holds whenever \(N \geq q/2p\). Returning to \eqref{eq:erm_stability_5}, it follows from the above that
\begin{equation}
    \frac{\varepsilon^2}{4} \leq C_\sm \kappa^2 + \bigg ( \frac{16}{N p c_\nu \sigma_\epsilon^{d^*} \kappa} \bigg )^{q/p} \frac{2 D_\hilb^2}{\sigma_\epsilon^4 q} \exp \bigg ( \frac{2q}{p-1} \bigg ) + N \kappa \E[\hat{\ell}_\sm(\hat{s})].\label{eq:erm_stability_6}
\end{equation}
We now choose \(\kappa\) by optimising the second two terms of this bound, by which we arrive at the choice
\begin{equation*}
    \kappa^{q/p + 1} = \frac{2 D_\hilb^2}{\sigma_\epsilon^4 p N \gamma} \exp \bigg ( \frac{2q}{p-1} \bigg ) \bigg ( \frac{16}{N p c_\nu \sigma_\epsilon^{d^*}} \bigg )^{q/p},
\end{equation*}
for some \(\gamma > 0\).
Substituting this in to \eqref{eq:erm_stability_6}, we arrive at the bound
\begin{align*}
    \frac{\varepsilon^2}{4} &\leq C_\sm (Np)^{-2} \bigg ( \frac{2 D_\hilb^2}{\sigma_\epsilon^4} \bigg )^{2/q} \exp \bigg ( \frac{4}{p-1} \bigg ) \bigg ( \frac{16}{c_\nu \sigma_\epsilon^{d^*}} \bigg )^{2/p} \gamma^{-1/q}\\
    & \qquad + \bigg ( \frac{2 D_\hilb^2}{\sigma_\epsilon^4} \bigg )^{1/q} \exp \bigg ( \frac{2}{p-1} \bigg ) \bigg ( \frac{16}{c_\nu \sigma_\epsilon^{d^*}} \bigg )^{1/p} \bigg ( \frac{\gamma^{1/p}}{q} + \frac{1}{p \gamma^{1/q}} \E[\hat{\ell}_\sm(\hat{s})] \bigg ).
\end{align*}
Optimising \(\gamma\) leads to the bound,
\begin{align*}
    \frac{\varepsilon^2}{4} &\leq \bigg ( \frac{2 D_\hilb^2}{\sigma_\epsilon^4} \bigg )^{1/q} \exp \bigg ( \frac{2}{p-1} \bigg ) \bigg ( \frac{16}{c_\nu \sigma_\epsilon^{d^*}} \bigg )^{1/p} \bigg ( C_\sm N^{-2} \bigg ( \frac{2 D_\hilb^2}{\sigma_\epsilon^4} \bigg )^{1/q} \exp \bigg ( \frac{2}{p-1} \bigg ) \bigg ( \frac{16}{c_\nu \sigma_\epsilon^{d^*}} \bigg )^{1/p} + \E[\hat{\ell}_\sm(\hat{s})] \bigg )^{1/p}\\
    &\leq \bigg ( \frac{2 D_\hilb^2}{\sigma_\epsilon^4} \vee \frac{16}{c_\nu \sigma_\epsilon^{d^*}} \bigg ) \exp \bigg ( \frac{4}{p-1} \bigg ) \bigg ( \bigg ( \frac{2 D_\hilb^2}{\sigma_\epsilon^4} \vee \frac{16}{c_\nu \sigma_\epsilon^{d^*}} \bigg ) C_\sm N^{-2} + \E[\hat{\ell}_\sm(\hat{s})] \bigg )^{1/p}
\end{align*}
Optimising \(p\),  we obtain
\begin{equation*}
    \frac{\varepsilon^2}{4} \lesssim \bigg ( \frac{2 D_\hilb^2}{\sigma_\epsilon^4} \vee \frac{16}{c_\nu \sigma_\epsilon^{d^*}} \bigg ) \exp \bigg ( \frac{5}{2\sqrt{2}} \log(\alpha^{-1})^{1/2} - 2 \bigg ) \alpha, \quad \alpha = \bigg ( \frac{2 D_\hilb^2}{\sigma_\epsilon^4} \vee \frac{16}{c_\nu \sigma_\epsilon^{d^*}} \bigg ) C_\sm N^{-2} + \E[\hat{\ell}_\sm(\hat{s})],
\end{equation*}
from which the bound in the statement follows. To obtain that \(\kappa \leq 1/2\) and \(N \geq q/2q\), it is sufficient to require that \(N\) is sufficiently large.
\end{proof}
\section{Proofs for sampling and score stability}\label{app:samp}
In this section, we provide details for the discretisation scheme considered in Section \ref{sec:sampling} and give the proof for Proposition \ref{prop:coarse} and Corollary \ref{cor:coarse_opt}. In the work of \cite{Potaptchik2024-ue}, they consider the following discretisation scheme, based on the scheme of \citep{Benton2023-ov}:
\begin{gather*}
    \hat{y}_{k+1} = \mu_{t_{k+1} - t_k}^{-1} \hat{y}_k + \frac{\sigma^2_{t_{k+1} - t_k}}{\mu_{t_{k+1} - t_k}} s(\hat{y}_k, T - t_k) + \sigma_{t_{k+1} - t_k} \frac{\sigma_{T - t_{k+1}}}{\sigma_{T - t_k}} \zeta_k, \qquad k \in \{0, ..., K-1\},
\end{gather*}
where \(\zeta_k \sim N(0, I_d)\) and we recall that the timesteps \((t_k)_{k=0}^K\) are given by,
\begin{equation*}
    t_k =\begin{cases}
        \kappa k, & \text{if } k < \frac{T-1}{\kappa},\\
        T - (1 + \kappa)^{\frac{T-1}{\kappa}-k}, & \text{if } \frac{T-1}{\kappa} \leq k \leq K,
    \end{cases} 
\end{equation*}
where \(L = \frac{T-1}{\kappa} > 0\), \(K = \lfloor L + \log(\epsilon^{-1})/\log(1+\kappa) \rfloor\) and \(\kappa > 0, T \geq 1\) is chosen freely. We recall the following result from \cite{Potaptchik2024-ue}.
\begin{lemma}\label{lem:disc_err}
Suppose that \(\alpha = 1\) and Assumption \ref{ass:manifold} holds with \(\diam \, \supp(\nud) \leq 1\). Then, it holds that,
\begin{gather*}
    D(p_\epsilon\|A_\emm(s)) \lesssim \ell_\sm(s; \hat{\tau}) + D(p_T \| p_\infty) + \Delta_{\kappa, K},\\
    \Delta_{\kappa, K} = \kappa + d^* \kappa^2 (K-L) (\log(\epsilon^{-1}) + \sup |\log(p_\nu)|
    ),
\end{gather*}
where we define the measure,
\begin{equation*}
    \hat{\tau}(dt) = \frac{1}{K} \sum_{k=0}^{K-1} \delta_{T-t_k}(dt).
\end{equation*}
\end{lemma}
Thus, taking \(\kappa\) small recovers the upper bound for the backwards process in \eqref{eq:early_stop}.

\subsection{Coarse discretisation and regularisation}

Fix \(\epsilon > 0\) and suppose that \(\kappa\) is such that \(\log(\epsilon^{-1})/\log(1 + \kappa)\) is an integer. Set \(K = L + \log(\epsilon^{-1})/\log(1+\kappa)\) so that, according to the discretisation scheme,
\begin{equation*}
    t_K = T - (1 + \kappa)^{- \log(\epsilon^{-1})/\log(1+\kappa)} = T - \epsilon.
\end{equation*}

\begin{proof}[Proof of Proposition \ref{prop:coarse}]
Let \(\hat{s} = A_\erm(S)\). We begin with Lemma \ref{lem:disc_err}, which provides the bound,
\begin{equation*}
    \E[D(p_\epsilon\|A_\emm(\hat{s}))] \lesssim \E[\ell_\sm(\hat{s}; S, \hat{\tau})] + D(p_T \| p_\infty) + \Delta_{\kappa, K}.
\end{equation*}
For \(\epsilon\) sufficiently small we have the bound,
\begin{align*}
    \Delta_{\kappa, K} &= \kappa + d^* \kappa^2 \frac{\log(\epsilon^{-1})}{\log(1 + \kappa)} (\log(\epsilon^{-1}) + \sup |\log(p_\nu)|)\\
    &\lesssim \kappa (1 + \kappa) d^* \log(\epsilon^{-1})^2.
\end{align*}
Using Theorem \ref{thm:dsm_stability}, we obtain that if the algorithm is \(\stab\)-score stable, we have
\begin{align*}
    \E[\ell_\sm(\hat{s}; \hat{\tau})] &\lesssim \E[\hat{\ell}_\sm(\hat{s}; S, \hat{\tau})] + \stab \E[\hat{\ell}_\dsm(\hat{s}; S, \hat{\tau})]^{1/2} + \stab^2\\
    &\lesssim \E[\hat{\ell}_\sm(\hat{s}; \hat{\tau})] + \stab C_\sm^{1/2} + \stab^2
\end{align*}
Using Proposition \ref{prop:stability_erm} we obtain that with \(\tau = \hat{\tau}\), \(A_\erm\) is score stable, with constant,
\begin{align*}
    \stab^2 &\lesssim C \Big ( C C_\sm N^{-2} + \E[\hat{\ell}_\sm(\hat{s})] \Big )^c\\
    &\lesssim c_\nu^{-1} \sigma_{T - t_{K-1}}^{-d^*} \Big ( c_\nu^{-1} \sigma_{T - t_{K-1}}^{-d^*} C_\sm N^{-2} + \E[\hat{\ell}_\sm(\hat{s})] \Big )^c.
\end{align*}
Now by definition, we have that
\begin{equation*}
    T - t_{K-1} = (1 + \kappa)^{L - K + 1} = \epsilon (1 + \kappa),
\end{equation*}
so if we take \(\epsilon, \kappa\) sufficiently small so that \(\epsilon (1 + \kappa) \leq \frac{1}{2}\), we also have \(\sigma_{\epsilon (1 + \kappa)}^2 \geq \epsilon (1 + \kappa)\) and thus we obtain,
\begin{align*}
    \stab^2 &\lesssim c_\nu^{-1} \epsilon^{-d^*/2} (1 + \kappa)^{-d^*/2} \Big ( c_\nu^{-1} \epsilon^{-d^*/2} (1 + \kappa)^{-d^*/2} C_\sm N^{-2} + \E[\hat{\ell}_\sm(\hat{s})] \Big )^c.\\
    &\lesssim c_\nu^{-1} \epsilon^{-d^*} (1 + \kappa)^{-d^*} \Big ( c_\nu^{-1} C_\sm N^{-2} + \E[\hat{\ell}_\sm(\hat{s})] \Big )^c,
\end{align*}
where in the last inequality, we use that \(\epsilon (1 + \kappa) \leq 1/2\).
\end{proof}

We now proceed by proving Corollary \ref{cor:coarse_opt} in which the bound in Proposition \ref{prop:coarse} is optimised.

\begin{proof}[Proof of Corollary \ref{cor:coarse_opt}]
Let \(\tilde{\tau}_\epsilon\) denote the weak limit of the measure \(\tau_\kappa\) as \(\kappa \to 0^+\). Since \(\supp(\tilde{\tau}_\epsilon) \subseteq [\epsilon, T]\) and \(\epsilon > 0\), we know that \(\inf_{\mathcal{H}} \hat{\ell}_\sm(\cdot; S, \tilde{\tau}_\epsilon) < \infty\). From this, we deduce that \(\lim_{\kappa \to 0^+} B_\kappa < \infty\).

With this there exists \(\kappa^* \geq 1\) which is the smallest quantity satisfying,
\begin{equation*}
    (1 + \kappa^*)^{2d^* + 2} = \frac{B_{\kappa^*}}{\log(\epsilon^{-1})^2} \vee 1.
\end{equation*}
In the case that \(B_{\kappa^*} > \log(\epsilon^{-1})\), we have that
\begin{align*}
    B_{\kappa^*}^{1/2} (1 + \kappa^*)^{-d^*} + \frac{B_{\kappa^*}}{C_\sm} (1 + \kappa^*)^{-2d^*} + \kappa^* (1 + \kappa^*) d^* \log(\epsilon^{-1})^2 &= B_{\kappa^*}^{\frac{1}{2(d^* + 1)}} \log(\epsilon^{-1})^{\frac{d^*}{d^* + 1}} + (C_\sm^{-1} + d^*) B_{\kappa^*}^{\frac{1}{d^* + 1}} \\
    &\leq B_{\kappa^*}^{\frac{1}{2(d^* + 1)}} \log(\epsilon^{-1}) + (C_\sm^{-1} + d^*) B_{\kappa^*}^{\frac{1}{d^* + 1}} \log(\epsilon^{-1})^2.
\end{align*}
Plus, if \(B_{\kappa^*} \leq \log(\epsilon^{-1})\) and therefore \(\kappa^* = 1\), then there exists \(\kappa\) such that,
\begin{equation*}
    B_{\kappa}^{1/2} (1 + \kappa)^{-d^*} + \frac{B_{\kappa}}{C_\sm} (1 + \kappa)^{-2d^*} + \kappa (1 + \kappa) d^* \log(\epsilon^{-1})^2 \lesssim B_{\kappa}^{1/2} + \frac{B_{\kappa}}{C_\sm} + d e^{-T}.
\end{equation*}
Combining these leads to the bound in the statement.
\end{proof}
\section{Proofs for stability of SGD}
In this section, we analyse the stochastic optimisation scheme in \eqref{eq:sgd_iter}, deriving the score stability bounds given in Proposition \ref{prop:sgd_stability}. We begin with a basic lemma that follows from weight decay and gradient clipping.

\begin{lemma}\label{lem:sgd_compact_support}
Suppose that \(\eta_k < \lambda^{-1}\) for all \(k \in \N\), then for any \(K \in \N\), it holds that
\begin{equation*}
    \|\theta_K\| \leq \frac{C e}{\lambda} \vee \|\theta_0\|.
\end{equation*}
\end{lemma}
\begin{proof}
We begin with the bound,
\begin{align*}
    \|\theta_{k+1}\| &\leq (1 - \eta_k \lambda) \|\theta_k\| + \eta_k \|\clip_C(G_k(\theta_k, \{x_i\}_{i \in B_k}))\|\\
    &\leq (1 - \eta_k \lambda) \|\theta_k\| + \eta_k C.
\end{align*}
By comparison, this leads to the bound
\begin{align*}
    \|\theta_k\| &\leq C \sum_{k=0}^{K-1} \eta_k \prod_{i=k+1}^{K-1} (1 - \eta_k \lambda) + \prod_{k=0}^{K-1} (1 - \eta_k \lambda) \|\theta_0\|\\
    &\leq C \sum_{k=0}^{K-1} \eta_k \exp \bigg ( \lambda \sum_{i=0}^k \eta_k \bigg ) + \exp \bigg ( -\lambda \sum_{i=0}^{K-1} \eta_k \bigg ) \|\theta_0\|\\
    &\leq C \exp \bigg ( - \lambda \sum_{i=0}^{K-1} \eta_k + \lambda \max_k \eta_k \bigg ) \sum_{k=0}^{K-1} \eta_k \exp \bigg ( \lambda \sum_{i=0}^{k-1} \eta_k \bigg ) + \exp \bigg ( -\lambda \sum_{i=0}^{K-1} \eta_k \bigg ) \|\theta_0\|
\end{align*}
Since the sum forms a left Riemann sum, approximating an integral of an increasing function, we can upper bound it by the integral over \(\exp(\lambda t)\). Furthermore, we have that \(\lambda \max_k \eta_k \leq 1\), which leads to the bound,
\begin{align*}
    \|\theta_k\| &\leq C e \exp \bigg ( - \lambda \sum_{i=0}^{K-1} \eta_k \bigg ) \int_0^{\sum_{k=0}^{K-1} \eta_k} \exp ( \lambda t ) dt + \exp \bigg ( -\lambda \sum_{i=0}^{K-1} \eta_k \bigg ) \|\theta_0\|\\
    &\leq \frac{C e}{\lambda} \bigg ( 1 - \exp \Big ( -\lambda \sum_{k=0}^{K-1} \eta_k \Big ) \bigg ) + \exp \Big ( -\lambda \sum_{k=0}^{K-1} \eta_k \Big ) \|\theta_0\|\\
    &\leq \frac{C e}{\lambda} \vee \|\theta_0\|.
\end{align*}
\end{proof}

We are now ready to prove Proposition \ref{prop:sgd_stability}.
\pasteenv{proposition}{prop:sgd_stability}
\begin{proof}
Since the stochastic mini-batch scheme, and therefore the resulting score matching algorithm, is symmetric to dataset permutations, we consider stability under changes in the \(N^{th}\) entry of the dataset, without loss of generality. Let \(\theta_k\) be the process given in \eqref{eq:sgd_iter}, using the dataset \(S\) and let \(\tilde{\theta}_k\) be the same process using \(S^{N}\) instead of \(S\):
\begin{equation*}
    \tilde{\theta}_{k+1} = (1 - \eta \lambda) \tilde{\theta}_k - \eta_k \, \clip_C(G_k(\tilde{\theta}_p, \{\tilde{x}_i\}_{i \in B_k})), \qquad \tilde{\theta}_0 = \theta_0,
\end{equation*}
where \(\tilde{x}_i = x_i\) for \(i \neq N\), \(\tilde{x}_N = \tilde{x}\). By having the processes share the same mini-batch indices \(B_k\) and gradient approximation \(G_k\) (i.e. sharing the same random time variables \(t_{i, j}\) and noise \(\xi_{i, j}\)), we couple the processes \(\theta_k\) and \(\tilde{\theta}_k\).

We proceed by first controlling the stability of the gradient estimator, computing the bound,
\begin{align*}
    \|G_k(\theta_k, &(x_i)_{i \in B_k}) - G_k(\tilde{\theta}_k, (x_i)_{i \in B_k})\|\\
    &\leq \frac{1}{N_B P} \sum_{i \in B_k} \sum_{j=1}^P w_{t_{i, j}} \| \nabla s_{\theta_k}(X_{i, j}, t_{i, j})^T (s_{\theta_k}(X_{i, j}, t_{i, j}) - \nabla \log p_{t_{i, j}|0}(X_{t_{i, j}}|x_i)) \\
    & \qquad \qquad - \nabla s_{\tilde{\theta}_k}(X_{t_{i, j}}, t_{i, j})^T (s_{\tilde{\theta}_k}(X_{t_{i, j}}, t_{i, j}) - \nabla \log p_{t_{i, j}|0}(X_{t_{i, j}}|x_i)))\|\\
    &\leq \frac{1}{N_P P} \sum_{i \in B_k} \sum_{j=1}^P w_{t_{i, j}} \Big ( \| \nabla s_{\theta_k}(X_{t_{i, j}}, t_{i, j}) - \nabla s_{\tilde{\theta}_k}(X_{t_{i, j}}, t_{i, j}) \| \| s_{\theta_k}(X_{t_{i, j}}, t_{i, j}) - \nabla \log p_{t_{i, j}|0}(X_{t_{i, j}}|x_i)\| \\
    & \qquad \qquad + \| \nabla s_{\tilde{\theta}_k}(X_{t_{i, j}}, t_{i, j}) \| \| s_{\theta_k}(X_{t_{i, j}}, t_{i, j}) - s_{\tilde{\theta}_k}(X_{t_{i, j}}, t_{i, j}) \| \Big ) \\
    &\leq \frac{1}{N_P P} \sum_{i \in B_k} \sum_{j=1}^P w_{t_{i, j}} \Big ( M(X_{t_{i, j}}, t_{i, j}) \|s_{\theta_k}(X_{t_{i, j}}, t_{i, j}) - \nabla \log p_{t_{i, j}|0}(X_{t_{i, j}}|x_j)\|\\
    & \qquad \qquad + L(X_{t_{i, j}}, t_{i, j})^2 \Big ) \| \theta_k - \tilde{\theta}_k \|.
\end{align*}
We control the expectation of this by first noting that,
\begin{align*}
    &\E \Big [ w_{t_{i, j}} \Big ( M(X_{t_{i, j}}, t_{i, j}) \|s_{\theta_k}(X_{t_{i, j}}, t_{i, j}) - \nabla \log p_{t_{i, j}|0}(X_{t_{i, j}}|x_j)\|  + L(X_{t_{i, j}}, t_{i, j})^2 \Big ) \Big | \theta_k, \tilde{\theta}_k, S, \tilde{x} \Big ]\\
    & \qquad \leq \bigg ( \int \E[M(X_t, t)^2|X_0=x_i] \tau(dt) \bigg )^{1/2} \bigg ( \int \hat{\ell}_\dsm(s_{\theta_k};\{x_i\}, \delta_{t}) \tau(dt) \bigg )^{1/2}\\
    & \qquad\qquad + \int \E[L(X_t, t)^2|X_0=x_i] \tau(dt)\\
    & \qquad \leq \overline{M} B_\ell C_\tau^{1/2} + \overline{L}^2,
\end{align*}
where we define the quantity \(C_\tau := \int \sigma_t^{-4} \tau(dt)\). From this, it follows that
\begin{align*}
    \E \Big [ \|G_k(\theta_k, (x_i)_{i \in B_k}) - G_k(\tilde{\theta}_k, (x_i)_{i \in B_k})\| \Big | \theta_k, \tilde{\theta}_k, S, \tilde{x} \Big ] &\leq \Big ( \overline{M} B_\ell C_\tau^{1/2} + \overline{L}^2 \Big ) \, \| \theta_k - \tilde{\theta}_k \|.
\end{align*}
Furthermore, we can control the difference between \(G_k(\tilde{\theta}_k, (x_i)_{i \in B_k})\) and \(G_k(\tilde{\theta}_k, (\tilde{x}_i)_{i \in B_k})\), using the fact that they are identical whenever \(N \not \in B_k\). Thus, obtaining,
\begin{align*}
    \E \Big [ \|\clip_C(G(\theta_k, &(x_i)_{i \in B_k})) - \clip_C(G(\tilde{\theta}_k, (\tilde{x}_i)_{i \in B_k}))\| \Big | \theta_k, \tilde{\theta}_k, S, \tilde{x} \Big ]\\
    & \qquad \leq \E \Big [ \|G(\theta_k, (x_i)_{i \in B_k}) - G(\tilde{\theta}_k, (x_i)_{i \in B_k})\| \Big | \theta_k, \tilde{\theta}_k, S, \tilde{x} \Big ] \\
    & \qquad \qquad + \E \Big [ \| \clip_C(G(\tilde{\theta}_k, (x_i)_{i \in B_k})) - \clip_C(G(\tilde{\theta}_k, (\tilde{x}_i)_{i \in B_k}))\| \Big | \theta_k, \tilde{\theta}_k, S, \tilde{x} \Big ]\\
    &\qquad \leq \Big ( \overline{M} B_\ell C_\tau^{1/2} + \overline{L}^2 \Big ) \, \| \theta_k - \tilde{\theta}_k \| + 2C\frac{N_B}{N},
\end{align*}
where we have used the fact that \(\prob(N \in B_k) = \frac{N_B}{N}\). Thus, using \eqref{eq:sgd_iter}, we obtain that for any \(k_0 \leq k\),
\begin{align*}
    \E \Big [ \| \theta_{k+1} - \tilde{\theta}_{k+1} \| \Big | \theta_{k_0}, \tilde{\theta}_{k_0}, S, \tilde{x} \Big ] &\leq \Big ( 1 + \eta_k \Big ( \overline{M} B_\ell C_\tau^{1/2} + \overline{L}^2 - \lambda \Big ) \Big ) \, \E \Big [ \| \theta_k - \tilde{\theta}_k \| \Big | \theta_{k_0}, \tilde{\theta}_{k_0}, S, \tilde{x} \Big ] + 2 \eta_k C \frac{N_B}{N}.\\
    &\leq (1 + \eta_k \upsilon ) \, \E \Big [ \| \theta_k - \tilde{\theta}_k \| \Big | \theta_{k_0}, \tilde{\theta}_{k_0}, S, \tilde{x} \Big ] + 2 \eta_k C \frac{N_B}{N},
\end{align*}
where \(\upsilon = \overline{M} B_\ell C_\tau^{1/2} + \overline{L}^2 - \lambda\). Thus, by comparison, we obtain,
\begin{align*}
     \E \Big [ \| \theta_K - \tilde{\theta}_K \| \Big | \theta_{k_0}, \tilde{\theta}_{k_0}, S, \tilde{x} \Big ] &\leq \sum_{i=k_0}^{K-1} 2 \eta_i C \frac{N_B}{N} \prod_{j=i+1}^{K-1} (1 + \eta_j \upsilon ) + \| \theta_{k_0} - \tilde{\theta}_{k_0}\| \prod_{j=k_0}^{K-1} (1 + \eta_j \upsilon ).
\end{align*}
From this we obtain the following:
\begin{align*}
    \E [\| \theta_K - \tilde{\theta}_K \| | \theta_{k_0} = \tilde{\theta}_{k_0}, S, \tilde{x}] &\leq 2C\frac{N_B}{N} \sum_{i=k_0}^{K-1} \eta_i \exp \bigg ( \sum_{j=i+1}^{K-1} \eta_j \upsilon \bigg )\\
    &\leq \frac{2 C N_B \bar{\eta}}{N} \sum_{i=k_0}^{K-1} \frac{1}{i} \bigg ( \frac{K}{i} \bigg )^{\bar{\eta}\upsilon}\\
    &\lesssim \frac{C N_B}{N \upsilon} \bigg ( \frac{K}{k_0} \bigg )^{\bar{\eta} \upsilon},
\end{align*}
where we use the fact that \(\sum_{j=i+1}^{K-1} \frac{1}{j} \leq \log(K) - \log(i)\). By the law of total probability, we have
\begin{align*}
    \E [\| \theta_{K} - \tilde{\theta}_{K} \| |\theta_0 ] &= \E [\| \theta_K - \tilde{\theta}_K \| | \theta_{k_0} = \tilde{\theta}_{k_0}] \prob(\theta_{k_0} = \tilde{\theta}_{k_0}|\theta_0) + \E [\| \theta_K - \tilde{\theta}_K \| | \theta_{k_0} \neq \tilde{\theta}_{k_0}, \theta_0] \prob(\theta_{k_0} \neq \tilde{\theta}_{k_0}|\theta_0)\\
    &\lesssim \frac{C N_B}{N \upsilon} \bigg ( \frac{K}{k_0} \bigg )^{\bar{\eta} \upsilon} + \bigg ( \frac{Ce}{\lambda} \vee \|\theta_0\| \bigg ) \frac{k_0 N_B}{N},
\end{align*}
where in the second inequality, we use Lemma \ref{lem:sgd_compact_support}. Thus, optimising \(k_0\) leads to the bound,
\begin{equation*}
    \E [\| \theta_{K} - \tilde{\theta}_{K} \| | \theta_0 ] \lesssim \bigg ( \frac{C}{c} \bigg )^{\frac{1}{\upsilon + 1}} (1 + 1/c\upsilon) \bigg ( \frac{Ce}{\lambda} \vee \|\theta_0\| \bigg )^{\frac{c\upsilon}{c\upsilon + 1}} \frac{N_B}{N} K^{\frac{c\upsilon}{c\upsilon + 1}}.
\end{equation*}
Finally, we obtain score stability using the fact that
\begin{align*}
    \int \E[\|s_{\theta_K}(X_t, t) - s_{\tilde{\theta}_K}(X_t, t) \|^2|X_0=\tilde{x}, S] \tau(dt) &\leq \E \Big [ \bar{L}^2 \| \theta_K -\tilde{\theta}_K\|^2 \Big ]\\
    &\leq 2 \E \Big [ \bar{L}^2 \bigg ( \frac{Ce}{\lambda} \vee \|\theta_0\| \bigg )  \| \theta_K -\tilde{\theta}_K\| \Big ]\\
    &\lesssim \bar{L}^2 \bigg ( \frac{Ce}{\lambda} \vee R \bigg )^{1 + \frac{c\upsilon}{c\upsilon + 1}} \bigg ( \frac{C}{c} \bigg )^{\frac{1}{c\upsilon + 1}} (1 + 1/c\upsilon)\frac{N_B}{N} K^{\frac{c\upsilon}{c\upsilon + 1}},
\end{align*}
where \(R^2 = \E\|\theta_0\|^2\).
\end{proof}
\section{Wasserstein contractions}
In this section, we derive the Wasserstein contraction result used in the proof of Proposition \ref{prop:time_indep_bounds}. We begin with the more abstract problem of deriving Wasserstein contractions for a discrete time diffusion process with anisotropic non-constant volatility. We consider stochastic processes given by the discrete-time update,
\begin{gather}
    x_{k+1} = (1 - \eta \lambda) x_k + \eta b(x_k) + \sqrt{\eta} \sigma(x_k) \xi_k, \label{eq:stoch_proc_1}\\
    y_{k+1} = (1 - \eta \lambda) y_k + \eta \tilde{b}(y_k) + \sqrt{\eta} \tilde{\sigma}(y_k) \xi_k, \label{eq:stoch_proc_2}
\end{gather}
for some \(b, \tilde{b}: \R^d \to \R^d, \sigma, \tilde{\sigma}: \R^d \to \R^{d \times d}\) where \(\xi_k \sim N(0, I_d)\), and we show that the laws of \(x_k\) and \(y_k\) contract in Wasserstein distance. We borrow the strategy developed by \cite{Eberle2016-te} and extended in \citep{Eberle2019-ai, Majka2020-lp}, constructing a coupling and a metric for which exponential contractions of the coupling can be obtained. However, these works are restricted to the setting of isotropic noise with constant volatility (i.e. \(\sigma(x) = c I_d\)) and so some careful modification to the strategy is required. In particular, we analyse this process with respect to the seminorm \(\|\cdot\|_{G^+}\) given by \(\|x\|_{G^+}^2 = x^T G^+ x\), where \(G^+\) denotes the Moore-Penrose pseudoinverse of the matrix \(G\). Furthermore, we allow for \(x_k\) and \(y_k\) to have different bias and volatility terms and so controlling for this will also require some modifications to the proof technique.

To define our coupling we first suppose that there exists a symmetric positive semi-definite matrix \(G \in \R^{d \times d}\) such that \(\sigma(x), \tilde{\sigma}(y) \succcurlyeq G^{1/2}\) for all \(x \in \R^d\), and to couple an update from the above process starting at \(x, y \in \R^d\), we first define the update,
\begin{gather*}
    \tilde{x} = (1 - \eta \lambda) x + \eta b(x), \qquad \tilde{y} = (1 - \eta \lambda) y + \eta \tilde{b}(y),\\
    \hat{x} = \tilde{x} + \sqrt{\eta} (\sigma(x) - G^{1/2}) Z', \qquad \hat{y} = \tilde{y} + \sqrt{\eta} (\tilde{\sigma}(y) - G^{1/2}) Z',
\end{gather*}
where \(Z' \sim N(0, I_d)\). We then define the \textit{synchronous coupled} processes,
\begin{gather*}
    X' = \hat{x} + \sqrt{\eta} G^{1/2} Z\\
    Y'_s = \hat{y} + \sqrt{\eta} G^{1/2} Z,
\end{gather*}
with \(Z \sim N(0, I_d)\). We also consider the reflection coupling,
\begin{equation}\label{eq:contraction_reflection}
    Y'_r = \hat{y} + \sqrt{\eta} G^{1/2} \Big ( I - 2 (G^{1/2})^+ e e^T (G^{1/2})^+ \Big ) Z, \qquad \text{ with } e = (\hat{x} - \hat{y}) / \|\hat{x} - \hat{y}\|_{G^+}
\end{equation}
which has the noise act in the mirrored direction. We combine these couplings to arrive at the final coupling \((X', Y')\):
\begin{equation}\label{eq:contraction_coupling}
    Y' = \begin{cases}
        X', & \text{ if } \zeta \leq \phi_{\hat{y}, \eta G}(X')/\phi_{\hat{x}, \eta G}(X'), |\langle e, Z \rangle|^2 < m^2/\eta \text{ and } \hat{r} \leq r_1\\
        Y'_r, & \text{ if } \zeta > \phi_{\hat{y}, \eta G}(X')/\phi_{\hat{x}, \eta G}(X'), \ |\langle e, Z \rangle|^2 < m^2/\eta \text{ and } \hat{r} \leq r_1\\
        Y'_s, & \text{ otherwise,}
    \end{cases}
\end{equation}
for some fixed \(m > 0\).

We assume the following regularity properties.
\begin{assumption}\label{ass:contraction}
Suppose that \(b\) is bounded, satisfying \(B := \sup_{x \in \R^n} \|b(x)\|_{G^+} < \infty\) and we have the Lipschitz property, \(\|b(x) - b(y)\|_{G^+} \leq L_b \|x - y\|_{G^+}\) and \(\|\sigma(x) - \sigma(y)\|_{op, G^+} \leq L_\sigma \|x - y\|_{G^+}\) for all \(x, y \in \R^n\) and for some \(L_b, L_\sigma \geq 0\).
\end{assumption}

We also allow for \(b \neq \tilde{b}\) and \(\sigma \neq \tilde{\sigma}\), making the following assumption.
\begin{assumption}\label{ass:contraction_inaccurate}
Suppose that \(b, \tilde{b}\) satisfy \(\|b(x) - \tilde{b}(x)\|_{G^+} \leq \tilde{B}_b, \|\sigma(x) - \tilde{\sigma}(x)\|_{op, G^+} \leq \tilde{B}_\sigma\) for all \(x \in \R^n\) and for some \(\tilde{B}_b, \tilde{B}_\sigma \geq 0\).
\end{assumption}

We define the objects,
\begin{equation*}
    R = \|x - y\|_{G^+}, \qquad \tilde{r} = \|\tilde{x} - \tilde{y}\|_{G^+}, \qquad \hat{r} = \|\hat{x} - \hat{y}\|_{G^+}, \qquad R' = \|X' - Y'\|_{G^+}.
\end{equation*}
We wish to show that \(R'\) contracts in expectation, i.e. it is less than \(R\) on average. We modify the metric to guarantee this is possible. We define the function,
\begin{equation*}
    f(r) = \begin{cases}
        \frac{1}{a} (1 - e^{-ar}), & \text{ if } r \leq r_2,\\
        \frac{1}{a} (1 - e^{-ar_2}) + \frac{1}{2 r_2} e^{-ar_2} (r^2 - r_2^2), & \text{ otherwise,}
    \end{cases}
\end{equation*}
where \(a = 6L_br_1/c_0\), \(r_1 = 4 (1 + \eta_0 L_b) B / \lambda\), \(r_2 = r_1 + \sqrt{\eta_0}\) and \(c_0, \eta_0\) are defined below. The coupling and the strategy for proving contractions is closely based on an analysis in \cite{Majka2020-lp} and for the sake of comparison, we rely on similar notation. We will also heavily borrow properties of the function \(f\) that are proven in this work.

By allowing \(\sigma\) to be non-constant, we run in to additional complications that are controlled by making the following assumption about the scale of \(L_\sigma\).
\begin{assumption}\label{ass:contraction_l_sigma}
Suppose that the following three inequalities hold:
\begin{equation*}
    n -1, (\lambda^2/16 L_\sigma^2 - 1)^2 \geq 32 \log \bigg ( \frac{8 L_\sigma (6 \vee (4a)) \kappa_0^{1/2}}{\sqrt{\eta} (1 - e^{-ar_2}) c} \bigg ), \qquad L_\sigma^2 \leq \lambda/8n,
\end{equation*}
for some universal constant \(\kappa_0\).
\end{assumption}

Under these assumptions, we obtain exponential contractions.

\begin{proposition}\label{prop:contraction}
Suppose that assumptions \ref{ass:contraction}, \ref{ass:contraction_inaccurate} and \ref{ass:contraction_l_sigma} hold and \(m = \sqrt{\eta_0}/2\), then for any \(\eta \leq \eta_0\) and \(x, y \in \R^d\), it holds that
\begin{equation*}
    \E[f(R')] \leq ( 1 - \eta c/4 ) f(r) + \frac{3}{2 r_2} e^{-ar_2}(\eta^2 \tilde{B}^2 + \eta n \tilde{B}_\sigma^2),
\end{equation*}
where
\begin{gather*}
    c := \min \bigg \{ e^{-ar_2} \frac{\lambda}{16}, \frac{\frac{1}{2} e^{-ar_2} r_2}{\frac{1}{a}(1-e^{-ar_2})} \frac{\lambda}{16}, \frac{9 L^2 r_1^2}{2 c_0} e^{-6L r_1^2/c_0}, \frac{3 L r_1}{16 \sqrt{\eta_0}} \bigg \},\\
    \eta_0 := \min \bigg \{ \frac{\lambda}{4 L^2}, \frac{16}{\lambda}, \frac{1}{2L}, \frac{2c_0 \log(3/2) \lambda^2}{432 L^2 B^2}, \frac{4 B^2}{\lambda^2}, \frac{c_0^2 (\log(2))^2 \lambda^2}{2304 L^2 B^2} \bigg \},
\end{gather*}
for some universal \(c_0\) and \(L = 2 (L_b - \lambda)_+ + 4\eta^{-1/2} L_\sigma \sqrt{2(n-1)}\).
\end{proposition}

\subsection{The coupling}
Before we provide the proof of Proposition \ref{prop:contraction}, we provide an explanation of how the coupling is arrived at. We begin by discussing the one-dimensional coupling of the Gaussian distribution that the construction is ultimately based on. Consider the following coupling of \(\mathcal{N}(t, \eta)\) and \(\mathcal{N}(s, \eta)\) for \(t, s \in \R\): with \(z \sim \mathcal{N}(0, 1)\),
\begin{gather}
    t' = t + \sqrt{\eta} z,\label{eq:coupling_oned_1}\\
    s' = \begin{cases}
        t', & \text{ if } \zeta \leq \phi_{s, \eta}(t')/\phi_{t, \eta}(t'), |\sqrt{\eta} z| < \tilde{m}, \text{ and } |t - s| \leq r_1,\\
        s - \sqrt{\eta} z, & \text{ if } \zeta > \phi_{s, \eta}(t')/\phi_{t, \eta}(t'), |\sqrt{\eta} z| < \tilde{m}, \text{ and } |t - s| \leq r_1,\\
        s + \sqrt{\eta} z, & \text{ otherwise.}
    \end{cases}\label{eq:coupling_oned_2}
\end{gather}
This coupling has the following property given in lemmas 3.1 and 3.2 of \cite{Majka2020-lp}.
\begin{lemma}\label{lem:coupling_oned}
For the coupling defined in \eqref{eq:coupling_oned_1} and \eqref{eq:coupling_oned_2}, we have
\begin{equation*}
    \E[|t' - s'|] = |t - s|,
\end{equation*}
and if \(\eta \leq 4 \tilde{m}^2\), we have
\begin{gather*}
    \E \Big [ (|t' - s'| - |t-s|)^2 \mathbbm{1}_{|t' - s'| \in I_{|t-s|}} \Big ] \geq \frac{1}{2} c_0 \min(\sqrt{\eta}, |t-s|) \sqrt{\eta},\\
    \text{where} \qquad I_r = \begin{cases}
        (0, r + \sqrt{\eta}), & \text{ if } r \leq \sqrt{\eta},\\
        (r - \sqrt{\eta}, r), & \text{ otherwise,}
    \end{cases}
\end{gather*}
for some universal constant \(c_0 > 0\).
\end{lemma}

Thus, through the second bound, we have control of the probability that \(|t'-s'|\) contracts below \(|t-s|\). The coupling proposed in \eqref{eq:contraction_coupling} is a multivariate analogue of this that also accounts for the diffusion coefficient \(G^{1/2}\). Let the vector \(e \in \R^d\) be as defined in \eqref{eq:contraction_reflection}, then we obtain that,
\begin{gather*}
    \langle e, G^+ X' \rangle = \langle e, G^+ \hat{x} \rangle + \sqrt{h} \langle (G^{1/2})^+ e, Z \rangle,\\
    \langle e, G^+ Y_s' \rangle = \langle e, G^+ \hat{y} \rangle + \sqrt{h} \langle (G^{1/2})^+ e, Z \rangle.
\end{gather*}
Therefore, \(\langle e, G^+ X' \rangle, \langle e, G^+ Y_s' \rangle\) are a synchronous coupling of \(\mathcal{N}(\langle e, G^+ \hat{x} \rangle, h)\) and \(\mathcal{N}(\langle e, G^+ \hat{y} \rangle, h)\). Furthermore, we have
\begin{align*}
    \langle e, G^+ Y_r' \rangle &= \langle e, G^+ \hat{y} \rangle + \sqrt{h} \langle (G^{1/2})^+ e, (I - 2 (G^{1/2})^+ e e^T (G^{1/2})^+) Z \rangle\\
    &= \langle e, G^+ \hat{y} \rangle + \sqrt{h} \langle (G^{1/2})^+ e, Z \rangle - 2 \sqrt{h} \langle e, G^+ e \rangle \langle (G^{1/2})^+ e, Z \rangle\\
    &= \langle e, G^+ \hat{y} \rangle - \sqrt{h} \langle (G^{1/2})^+ e,  Z \rangle,
\end{align*}
and so \(\langle e, G^+ X' \rangle, \langle e, G^+ Y_r' \rangle\) is the one-dimensional reflection coupling. Finally we obtain,
\begin{align*}
    \frac{\phi_{\hat{y}, \eta G}(X')}{\phi_{\hat{x}, \eta G}(X')} &= \frac{\phi_{(G^{1/2})^+(\hat{y} - \hat{x}), \eta (G^{1/2})^+ G^{1/2}}(\sqrt{\eta} Z)}{\phi_{\mathbf{0}, \eta (G^{1/2})^+ G^{1/2}}(\sqrt{\eta} Z)}\\
    &= \exp \bigg ( - \frac{1}{2 \eta}\| \sqrt{\eta} Z - (G^{1/2})^+(\hat{y} - \hat{x}) \|_{(G^{1/2})^+ G^{1/2}}^2 + \frac{1}{2 \eta}\| \sqrt{\eta} Z \|_{(G^{1/2})^+ G^{1/2}}^2 \bigg )\\
    &= \exp \bigg ( - \frac{1}{2 \eta} \|\hat{y} - \hat{x}\|_{G^+}^2 + \frac{1}{\eta} \sqrt{\eta} \langle (G^{1/2})^+(\hat{y} - \hat{x}), Z \rangle \bigg )\\
    &= \exp \bigg ( - \frac{1}{2 \eta} |\langle e, G^+ (\hat{y} - \hat{x}) \rangle|^2 + \frac{1}{\eta} \sqrt{\eta} \langle e, G^+ (\hat{y} - \hat{x}) \rangle \langle (G^{1/2})^+ e, Z \rangle \bigg )\\
    &= \exp \bigg ( - \frac{1}{2 \eta} ( \sqrt{\eta} \langle e, G^+ Z \rangle - \langle e, G^+ (\hat{y} - \hat{x}) \rangle)^2 + \frac{|\sqrt{\eta} \langle (G^{1/2})^+ e, Z \rangle|^2}{2 \eta} \bigg )\\
    &= \frac{\phi_{\langle e, G^+ (\hat{y} - \hat{x}) \rangle, \eta }(\sqrt{\eta} \langle (G^{1/2})^+ e , Z \rangle)}{\phi_{0, \eta}(\sqrt{\eta} \langle (G^{1/2})^+ e , Z \rangle)}\\
    &= \frac{\phi_{\langle e, G^+ \hat{y} \rangle, \eta}(\langle e, G^+ X' \rangle)}{\phi_{\langle e, G^+ \hat{x} \rangle, \eta}(\langle e, G^+ X' \rangle)}.
\end{align*}
From this, we deduce that \(\langle e, G^+ X' \rangle, \langle e, G^+ Y' \rangle\) are coupled as in \eqref{eq:coupling_oned_1}, \eqref{eq:coupling_oned_2}. The equivalence follows by setting
\begin{gather}
    t' = \langle e, G^+ X' \rangle, \qquad s' = \langle e, G^+ Y' \rangle \label{eq:coupling_one_to_multi_1}\\
    t = \langle e, G^+ \hat{x} \rangle, \qquad s = \langle e, G^+ \hat{y} \rangle, \qquad z = \langle (G^{1/2})^+ e, Z \rangle.
\end{gather}
Through this equivalence, we can extend the previous lemma to obtain the following result about the high dimensional coupling.

\begin{lemma}\label{lem:coupling_property}
For the coupling defined in \eqref{eq:contraction_coupling}, we obtain that for \(\eta \leq 4 m^2\), we have the following:
\begin{equation*}
    \E[R'] = \hat{r}, \qquad \E \Big [ (R' - \hat{r})^2 \mathbbm{1}_{R' \in I_{\hat{r}}} \Big ] \geq \frac{1}{2} c_0 \min(\sqrt{\eta}, \hat{r}) \sqrt{\eta},
\end{equation*}
where \(c_0\) and \(I_r\) is as in Lemma \ref{lem:coupling_oned}.
\end{lemma}
\begin{proof}
Let \(\{e_i\}_{i=1}^n\) be a basis of \(\R^n\) with respect to the inner product \(\langle \cdot, \cdot \rangle_{G^+}\) with \(e_1 = e\). Then, we have that
\begin{align}
    (R')^2 &= \sum_{i=1}^n |\langle e_i, G^+ (X' - Y') \rangle|^2 \nonumber\\
    &= |t' - s'|^2 + \sum_{i=2}^n |\langle e_i, G^+ (X' - Y') \rangle|^2,\label{eq:coupling_property_2}
\end{align}
where \(t', s'\) are as defined in \eqref{eq:coupling_one_to_multi_1}. For any \(i \neq 1\), we can use that \(e_i \perp e\), to obtain that
\begin{align*}
    \langle e_i, G^+ Y_r' \rangle &= \langle e_i, G^+ \hat{y} \rangle + \sqrt{h} \langle e_i, e \rangle + 2 \sqrt{h} \langle e_i, Z \rangle\\
    &= \langle e_i, G^+ \hat{y} \rangle + 2 \sqrt{h} z.
\end{align*}
From this, we obtain that,
\begin{equation*}
    \langle e_i, G^+ (X' - Y_r') \rangle = \langle e_i, G^+ (\hat{x} - \hat{y}) \rangle = 0.
\end{equation*}
This also holds for the synchronous coupling and hence we obtain \(\langle e_i, G^+ (X' - Y') \rangle = 0\). Combined with \eqref{eq:coupling_property_2}, we obtain that \(R' = |t' - s'|\). Similarly it can be shown that \(\hat{r} = |t - s|\) and thus, from Lemma \ref{lem:coupling_oned}, the statement of the lemma follows.
\end{proof}

\subsection{Proof of Proposition \ref{prop:contraction}}

We begin by considering the setting where \(Z'\) is truncated Gaussian noise and that \(b = \tilde{b}\), \(\sigma = \tilde{\sigma}\). We will then extend this to the more general setting in Section \ref{sec:contraction_full_noise}. We begin by decomposing \(\xi \sim N(0, I_d)\) in to directions parallel and perpendicular to the radial vector,
\begin{equation*}
    \xi_1 = v v^T \xi, \qquad \xi_2 = (I - v v^T) \xi, \qquad v = \frac{\tilde{x} - \tilde{y}}{\|\tilde{x} - \tilde{y}\|_{G^+}}.
\end{equation*}
We then clip each direction according to constants \(\bar{z}_1, \bar{z}_2 > 0\) and add them together:
\begin{equation}\label{eq:contraction_gauss_trunc}
    Z' = (1 \wedge \bar{z}_1 \|\xi_1\|_{G^+}^{-1}) \xi_1 + (1 \wedge \bar{z}_2 \|\xi_2\|_{G^+}^{-1}) \xi_2.
\end{equation}

To prove that the process is contractive, we consider two cases based on the initial distance \(r\).

\subsubsection{The case of \(r \geq r_1\)}

When \(r\) is large, we can rely on contractive properties following from the weight decay. For this, we obtain the following.
\begin{lemma}\label{lem:contraction_rhat}
Suppose that Assumption \ref{ass:contraction} holds and that \(4 \bar{z}_1 \leq \lambda L_\sigma^{-1} \sqrt{\eta}, 2 \bar{z}_2 \leq \sqrt{\lambda} L_\sigma, \eta \leq \lambda^{-1}\). Then whenever \(r \geq 4 B / \lambda\), we have
\begin{equation}\label{contraction_rhat_0}
    \hat{r} \leq \bigg ( 1 - \frac{\eta \lambda}{8} \bigg ) r,
\end{equation}
and when \(r < 4B/\lambda\),
\begin{equation}\label{contraction_rhat_1}
    \hat{r} \leq (1 + \eta L) r,
\end{equation}
where \(L = 2 (L_b - \lambda)_+ + 4 \eta^{-1/2} L_\sigma \bar{z}_1\).
\end{lemma}
\begin{proof}
From the triangle inequality and the Lipschitz property of \(b\), we obtain
\begin{align*}
    \tilde{r} &\leq (1 - \eta \lambda) \|x - y\|_{G^+} + \eta \|b(x) - b(y)\|_{G^+}\\
    &\leq (1 + \eta (L_b - \lambda)_+) r.
\end{align*}
Alternatively, we can use the fact that \(\|b\|_{G^+} \leq B\) to obtain, \(\tilde{r} \leq (1 - \eta \lambda) r + 2 \eta B\). In particular, if \(r \geq 4B/\lambda\), we obtain \(\tilde{r} \leq (1 - \eta \lambda / 2) r\). Next we bound \(\hat{r}\) using the decomposition,
\begin{align}
    \hat{r}^2 &= \|\tilde{x} - \tilde{y} + \sqrt{\eta}(\sigma(x) - \sigma(y)) Z'\|_{G^+}^2 \nonumber\\
    &= \|\tilde{x} - \tilde{y} + \sqrt{\eta}(\sigma(x) - \sigma(y)) (1 \wedge \bar{z}_1 \|\xi_1\|_{G^+}^{-1})\xi_1\|_{G^+}^2 + \|\sqrt{\eta}(\sigma(x) - \sigma(y)) (1 \wedge \bar{z}_2 \|\xi_2\|_{G^+}^{-1})\xi_2\|_{G^+}^2 \nonumber\\
    &\leq \|\tilde{x} - \tilde{y} + \sqrt{\eta}(\sigma(x) - \sigma(y)) (1 \wedge \bar{z}_1 \|\xi_1\|_{G^+}^{-1})\xi_1\|_{G^+}^2 \label{eq:contraction_rhat_2}
\end{align}
The second term is then bounded by,
\begin{align}
    \|\sqrt{\eta}(\sigma(x) - \sigma(y)) (1 \wedge \bar{z}_2 \|\xi_2\|_{G^+}^{-1})\xi_2\|_{G^+}^2 &\leq \eta \|\sigma(x) - \sigma(y)\|_{op, G^+} (1 \wedge \bar{z}_2 \|\xi_2\|_{G^+}^{-1})^2 \|\xi_2\|_{G^+}^2 \nonumber\\
    &\leq \eta L_\sigma^2 \bar{z}_2^2 r^2 \label{eq:contraction_rhat_3},
\end{align}
and the first term is bounded by,
\begin{align}
    \|\tilde{x} - \tilde{y} + \sqrt{\eta}(\sigma(x) - \sigma(y)) (1 \wedge \bar{z}_1 \|\xi_1\|_{G^+}^{-1})\xi_1\|_{G^+}^2 &\leq \tilde{r}^2 + \eta L_\sigma^2 \bar{z}_1^2 r^2 + 2 \sqrt{\eta} L_\sigma \langle v, G^+ \xi_1 \rangle \tilde{r}^2 \nonumber\\
    &\leq (1 + 2 \sqrt{\eta} L_\sigma \bar{z}_1) \tilde{r}^2 + \eta L_\sigma^2 \bar{z}_1^2 r^2. \label{eq:contraction_rhat_4}
\end{align}
We substitute \eqref{eq:contraction_rhat_3} and \eqref{eq:contraction_rhat_4} in to \eqref{eq:contraction_rhat_2} to obtain
\begin{align*}
    \hat{r}^2 &\leq (1 + 2 \sqrt{\eta} L_\sigma \bar{z}_1) \tilde{r}^2 + \eta L_\sigma^2 (\bar{z}_1^2 + \bar{z}_2^2) r^2\\
    &\leq (1 + \eta (L_b - \lambda)_+)^2 (1 + 2 \sqrt{\eta} L_\sigma \bar{z}_1) r^2 + \eta L_\sigma^2 (\bar{z}_1^2 + \bar{z}_2^2) r^2\\
    &\leq (1 + \eta (L_b - \lambda)_+ + 2 \eta^{3/2} (L_b - \lambda)_+ L_\sigma \bar{z}_1 + 2 \sqrt{\eta} L_\sigma \bar{z}_1 + \eta L_\sigma^2 (\bar{z}_1^2 + \bar{z}_2^2) ) r^2\\
    &\leq (1 + 2\eta (L_b - \lambda)_+ + 4 \sqrt{\eta} L_\sigma \bar{z}_1) r^2,
\end{align*}
where we have used that \(2 \eta^{1/2} L_\sigma \bar{z}_1 \leq 1, \eta^{1/2} L_\sigma (z_1^2 + z_2^2) \leq \bar{z}_1\), producing the bound in \eqref{contraction_rhat_1}. In the case that \(r \leq 4 B / \lambda\), we can use the fact that \(2 \eta^{1/2} L_\sigma \bar{z}_1 \leq \eta \lambda / 2\) and \(L_\sigma^2 (\bar{z}_1^2 + \bar{z}_2^2) \leq \lambda / 4\) to refine this bound:
\begin{align*}
    \hat{r}^2 &\leq (1 - \eta \lambda / 2)^2 (1 + 2 \sqrt{\eta} L_\sigma \bar{z}_1) r^2 + \eta L_\sigma^2 (z_1^2 + z_2^2) r^2\\
    &\leq (1 - \eta \lambda / 2) (1 - \eta^2 \lambda^2 / 4)^2 r^2 + \eta L_\sigma^2 (z_1^2 + z_2^2) r^2\\
    &\leq (1 - \eta \lambda / 4) r^2.
\end{align*}
Using the fact that \((1 - \eta \lambda / 4)^{1/2} \leq 1 - \eta \lambda / 8\), we obtain the bound in \eqref{contraction_rhat_0}.
\end{proof}

We will also need a property of \(f\) given in \cite{Majka2020-lp}.
\begin{lemma}\label{lem:contraction_concavity}
The function \(f\) satisfies the property that for all \(r \geq r_2\),
\begin{equation*}
    f \Big ( \Big ( 1 - \tfrac{\eta K}{2} \Big ) r \Big ) - f(r) \leq - \eta c f(r).
\end{equation*}
\end{lemma}

Using the fact that \(f\) is increasing, it follows from lemmas \ref{lem:contraction_rhat} and \ref{lem:contraction_concavity} that,
\begin{equation*}
    f(\hat{r}) \leq f \Big ( \Big ( 1 - \tfrac{\eta K}{2} \Big ) r \Big ) \leq (1 - \eta c) f(r).
\end{equation*}

Thus, to obtain contractions of \(\E[f(R')]\) when \(r \geq r_1\), it is sufficient to show that \(\E[f(R')|Z'] \leq f(\hat{r})\). Note that when \(\hat{r} \geq r_1\) or \(\|\sqrt{\eta} Z\| \geq m\), the synchronous coupling is used and so \(R' = \hat{r}\). Furthermore, if \(\hat{r} < r_1\) and \(\|\sqrt{\eta} Z\| < m\), we have that \(R' \leq r_2\) and thus, using the concavity of \(f\), we deduce that
\begin{equation*}
    \E[f(R')|Z'] - f(\hat{r}) \leq f'(\hat{r}) (\E[R'|Z'] - \hat{r}) = 0.
\end{equation*}
Thus, we have shown that whenever \(r \geq r_1\), \(\E[f(R')|Z'] \leq f(\hat{r})\).

\subsubsection{The case of \(r < r_1\)}
When \(r\) is small we no longer have contractions due to weight decay and must instead rely on properties of the coupling and function. From Taylor's theorem we have the following:
\begin{equation*}
    f(R') - f(\hat{r}) = f'(\hat{r}) (R' - \hat{r}) + \frac{1}{2} \sup_{\theta} f''(\theta) (R' - \hat{r})^2.
\end{equation*}
where the supremum is between all \(\theta \geq 0\) between \(R'\) and \(\hat{r}\). We note that in the present setting, \(\hat{r} \leq r_1 \) also (this follows from Lemma \ref{lem:contraction_rhat}) and furthermore \(R' - \hat{r} \leq 2 m \leq r_2\). Therefore, we can use that \(f\) is concave between \(R'\) and \(\hat{r}\) and so \(f''\) is negative. Using this fact, as well as the fact that \(\E[R'|Z'] = \hat{r}\), we obtain the bound,
\begin{align*}
    \E[f(R')|Z'] - f(\hat{r}) &\leq \frac{1}{2} \E \Big [ \sup_{\theta} f''(\theta) (R' - \hat{r})^2 \mathbbm{1}_{R' \in I_{\hat{r}}} \Big ]\\
    &\leq \frac{1}{2} \sup_{\theta \in I_{\hat{r}}} f''(\theta) \E \Big [ (R' - \hat{r})^2 \mathbbm{1}_{R' \in I_{\hat{r}}} \Big | Z' \Big ]\\
    &\leq \frac{1}{4} \sup_{\theta \in I_{\hat{r}}} f''(\theta) c_0 \min(\sqrt{\eta}, \hat{r}) \sqrt{\eta}.
\end{align*}
Furthermore, we analyse the contractions between \(\hat{r}\) and \(r\) using the fact that the function is concave between these values, obtaining,
\begin{equation*}
    f(\hat{r}) - f(r) \leq f'(r) (\hat{r} - r) \leq f'(r) \eta L r.
\end{equation*}
Since we have the derivative \(f'(r) = e^{-ar} = f'(\hat{r}) e^{-a(r-\hat{r})} \leq f'(\hat{r}) e^{a \eta L r_1}\), it holds that
\begin{equation}\label{eq:f_grad_bound}
    f(\hat{r}) - f(r) \leq f'(\hat{r}) e^{a \eta L r_1} \eta L \hat{r},
\end{equation}
where we have used that \(f(\hat{r}) - f(r) \leq 0\) holds trivially whenever \(r \geq \hat{r}\). Putting these together, we obtain the bound,
\begin{equation*}
    \E[f(R')|Z'] - f(r) \leq f'(\hat{r}) e^{a \eta L r_1} \eta L \hat{r} + \frac{1}{4} \sup_{\theta \in I_{\hat{r}}} f''(\theta) c_0 \min(\sqrt{\eta}, \hat{r}) \sqrt{\eta}.
\end{equation*}
To complete the analysis of this case, we borrow a result from \cite{Majka2020-lp}.
\begin{lemma}
The function \(f\), satisfies the property that for all \(\hat{r} \in [0, r_1]\),
\begin{equation*}
    f'(\hat{r}) e^{a \eta L r_1} \eta L \hat{r} + \frac{1}{4} c_0 \min(\sqrt{\eta}, \hat{r}) \sqrt{\eta} \sup_{I_{\hat{r}}} f''(\hat{r}) \leq - c h f(\hat{r}).
\end{equation*}
\end{lemma}

Between this section and the previous, we have shown that for any \(x, y \in \R^n\),
\begin{equation*}
    \E[f(R')] \leq (1 - \eta c /2) f(r),
\end{equation*}
in the setting where \(Z'\) is the truncated Gaussian defined in \eqref{eq:contraction_gauss_trunc} and \(b = \tilde{b}, \sigma = \tilde{\sigma}\).

\subsubsection{Full noise and inaccurate drift}\label{sec:contraction_full_noise}

We now consider the more general case of \(b \neq \tilde{b}\), \(\sigma \neq \tilde{\sigma}\) necessarily and also set \(Z' = \xi\), so that it is Gaussian distributed. We do this by borrowing the contraction analysis above. We use the notation \(R'' = \|X' - Y'\|_{G^+}\) to not confuse it with \(R'\) used above. We obtain,
\begin{align*}
    R'' &\leq R' + \eta \|b(y) - \tilde{b}(y)\|_{G^+} + \sqrt{\eta} \|(\sigma(y) - \tilde{\sigma}(y)) \xi\|_{G^+} + \sqrt{\eta}\|(\sigma(x) - \sigma(y))(Z' - \xi_1 - \xi_2)\|_{G^+}\\
    &\leq R' + \eta \tilde{B} + \sqrt{\eta} \tilde{B}_\sigma \|\xi\|_{G^+} + \sqrt{\eta} \|\sigma(x) - \sigma(y)\|_{op, G^+} \|Z' - (1 \wedge \bar{z}_1 \|\xi_1\|_{G^+}^{-1}) \xi_1 - (1 \wedge \bar{z}_2 \|\xi_2\|^{-1}) \xi_2\|_{G^+}\\
    &\leq R' + \eta \tilde{B} + \sqrt{\eta} \tilde{B}_\sigma \|\xi\|_{G^+} + \sqrt{\eta} L_\sigma r ( \|\xi_1\|_{G^+} \mathbbm{1}_{\|\xi_1\|_{G^+} \geq \bar{z}_1} + \|\xi_2\|_{G^+} \mathbbm{1}_{\|\xi_2\|_{G^+} \geq \bar{z}_2}).
\end{align*}
We use the following stability bound for the function \(f\) given in the proof of Theorem 2.5 in \cite{Majka2020-lp}.

\begin{lemma}\label{lem:contraction_f_bound}
For any \(t, s \geq 0\), we have
\begin{equation*}
    f(t) - f(s) \leq (r_2^{-1} e^{-ar_2} (t \vee s) + 1) |t - s|.
\end{equation*}
\end{lemma}
Thus, the difference between \(f(R'')\) and \(f(R')\) is given by,
\begin{align}
    f(R'') - f(R') &\leq f(R' + \eta \tilde{B} + \sqrt{\eta} \tilde{B}_\sigma \|\xi\|_{G^+} + \sqrt{\eta} L_\sigma r ( \|\xi_1\|_{G^+} \mathbbm{1}_{\|\xi_1\|_{G^+} \geq \bar{z}_1} + \|\xi_2\|_{G^+} \mathbbm{1}_{\|\xi_2\|_{G^+} \geq \bar{z}_2})) - f(R') \nonumber\\
    &\leq (r_2^{-1} e^{-ar_2} (R' + \eta \tilde{B} + \sqrt{\eta} \tilde{B}_\sigma \|\xi\|_{G^+} + \sqrt{\eta} L_\sigma r ( \|\xi_1\|_{G^+} \mathbbm{1}_{\|\xi_1\|_{G^+} \geq \bar{z}_1} + \|\xi_2\|_{G^+} \mathbbm{1}_{\|\xi_2\|_{G^+} \geq \bar{z}_2})) \nonumber\\
    & \qquad + 1)(\eta \tilde{B} + \sqrt{\eta} \tilde{B}_\sigma \|\xi\|_{G^+} + \sqrt{\eta} L_\sigma r ( \|\xi_1\|_{G^+} \mathbbm{1}_{\|\xi_1\|_{G^+} \geq \bar{z}_1} + \|\xi_2\|_{G^+} \mathbbm{1}_{\|\xi_2\|_{G^+} \geq \bar{z}_2})).\label{eq:contraction_w_err}
\end{align}

We now control the expected value of this. Using concentration of the \(\chi^2\) distribution (see Example 2.11 of \citet{Wainwright2019-rz}), we obtain that for any \(\bar{z}_1 = \sqrt{2\lambda_{\text{gap}}(G)^{-1}(n-1)}\),
\begin{align*}
    \E[\|\xi_2\|_{G^+}^2 \mathbbm{1}_{\|\xi_2\|_{G^+} \geq \bar{z}_2}] &\leq \lambda_{\text{gap}}(G)^{-1} \E[\|\xi_2\|^2 \mathbbm{1}_{\|\xi_2\| \geq \lambda_{\text{gap}}(G)^{1/2}\bar{z}_2}]\\
    &\leq \lambda_{\text{gap}}(G)^{-1} \int_{\lambda_{\text{gap}}(G)\bar{z}_2^2}^\infty \prob(\|\xi_2\|^2 \geq r) \, dr + \lambda_{\text{gap}}(G)^{-1} \int^{\lambda_{\text{gap}}(G)\bar{z}_2^2}_0 \prob(\|\xi_2\| \geq \bar{z}_2) \, dr\\
    &\leq \lambda_{\text{gap}}(G)^{-1} \int_{\lambda_{\text{gap}}(G)\bar{z}_2^2}^\infty \exp \bigg ( - \frac{(r - (n-1))^2}{8n} \bigg ) dr\\
    & \qquad + \lambda_{\text{gap}}(G)^{-1} \exp \bigg ( - \frac{(\lambda_{\text{gap}}(G) \bar{z}_2^2-(n-1))^2}{8n} \bigg ) \bar{z}^2\\
    &\leq \lambda_{\text{gap}}(G)^{-1} (\sqrt{8(n-1)\pi} + \lambda_{\text{gap}}(G) \bar{z}_2^2) \exp \bigg ( - \frac{(\lambda_{\text{gap}}(G) \bar{z}_2^2-(n-1))^2}{8(n-1)} \bigg )\\
    &\leq \lambda_{\text{gap}}(G)^{-1} \bigg ( \sqrt{8(n-1)\pi} \exp(-(n-1)/16)\\
    & \qquad + \lambda_{\text{gap}}(G) \bar{z}_2^2 \exp(-\bar{z}_2^4/64) \bigg ) \exp \bigg ( - \frac{(\lambda_{\text{gap}}(G) \bar{z}_2^2-(n-1))^2}{16(n-1)} \bigg )\\
    &\leq \kappa_0 \lambda_{\text{gap}}(G)^{-1} \exp \bigg ( - \frac{(\lambda_{\text{gap}}(G) \bar{z}_2^2-(n-1))^2}{16(n-1)} \bigg )\\
    &\leq \kappa_0 \lambda_{\text{gap}}(G)^{-1} \exp \bigg ( - \frac{n-1}{16} \bigg ),
\end{align*}
for some universal constant \(\kappa_0 \geq 1\) (independent of \(n\) and \(\bar{z}\)). Similarly, we have
\begin{align*}
    \E[\|\xi_1\|_{G^+}^2 \mathbbm{1}_{\|\xi_1\|_{G^+} \geq \bar{z}_1}] \leq \kappa_0 \lambda_{\text{gap}}(G)^{-1} \exp \bigg ( - \frac{(\lambda_{\text{gap}}(G)\bar{z}_1^2-1)^2}{16} \bigg ),
\end{align*}
for any \(\bar{z}_1 \geq \lambda_{\text{gap}}(G)^{-1/2}\). Therefore, we choose \(\bar{z}_1 = \frac{\lambda}{4} L_\sigma^{-1} \sqrt{\eta}\)We now return to \eqref{eq:contraction_w_err} using these bounds as well as the fact that \(\E[R'|Z'] = \hat{r}\). Defining the quantity,
\begin{equation*}
    A := \kappa_0^{1/2} \lambda_{\text{gap}}(G)^{-1/2} \exp(-(n-1)/32) + \kappa_0^{1/2} \lambda_{\text{gap}}(G)^{-1/2} \exp(-(\lambda_{\text{gap}}(G) \bar{z}_1^2 - 1)^2/32),
\end{equation*}
we obtain that for \(\eta \leq \min\{ \tilde{B}/2, d \tilde{B}_\sigma^2 / 4, 1/2L, 1/2 L_\sigma^2 A^2\}\)
\begin{align*}
    \E[f(R'') - f(R')] &= (r_2^{-1} e^{-ar_2} (\E[\hat{r}^2]^{1/2} + \eta \tilde{B} + \sqrt{\eta d} \tilde{B}_\sigma + \sqrt{\eta} L_\sigma r A) + 1)(\eta \tilde{B} + \sqrt{\eta d} \tilde{B}_\sigma + \sqrt{\eta} L_\sigma r A)\\
    &\leq (r_2^{-1} e^{-ar_2} (1 + \eta L + \sqrt{\eta} L_\sigma A) r + 1)\sqrt{\eta} L_\sigma r A + \frac{1}{r_2} e^{-a r_2} (\eta^2 \tilde{B}^2 + \eta d \tilde{B}_\sigma^2)\\
    & \qquad + (r_2^{-1} e^{-ar_2} (1 + \eta L + \sqrt{\eta} L_\sigma A) r + 1) (\eta \tilde{B} + \sqrt{\eta d} \tilde{B}_\sigma)\\
    & \qquad + r_2^{-1} e^{-ar_2} (\eta \tilde{B} + \sqrt{\eta d} \tilde{B}_\sigma) \sqrt{\eta} L_\sigma r^2 A\\
    &\leq (4 r_2^{-1} e^{-ar_2} r + 1)\sqrt{\eta} L_\sigma r A + \frac{3}{2 r_2} e^{-a r_2} (\eta^2 \tilde{B}^2 + \eta d \tilde{B}_\sigma^2).
\end{align*}
When \(r \leq r_2\), we have
\begin{align*}
    (r_2^{-1} e^{-ar_2} (1 + \eta L + 3 \sqrt{\eta} L_\sigma A) r + 1) r &\leq ( e^{-ar_2} (1 + \eta L + \sqrt{\eta} L_\sigma A) + 1) r\\
    &\leq ( e^{-ar_2} (1 + \eta L + \sqrt{\eta} L_\sigma A) + 1) (a^{-1}(1 - e^{-ar_2}))^{-1} a^{-1}(1 - e^{-ar})\\
    &\leq 4 (a^{-1}(1 - e^{-ar_2}))^{-1} f(r),
\end{align*}
where in the final line, we used \(\eta \leq L^{-1}\) and \(\sqrt{\eta} L_\sigma \kappa_0^{1/2} \leq 1\). When \(r > r_2\), we have
\begin{align*}
    (r_2^{-1} e^{-ar_2} (1 + \eta L + \sqrt{\eta} L_\sigma A) r + 1) r &\leq ( e^{-ar_2} (1 + \eta L + \sqrt{\eta} L_\sigma A) + 1) r_2^{-1} r^2\\
    &\leq 2 ( 2 + e^{ar_2}) f(r).
\end{align*}
Thus, we obtain,
\begin{align*}
    \E[f(R'')] &\leq \E[f(R')] + \sqrt{\eta} L_\sigma A \tfrac{6 \vee (4a)}{1 - e^{-ar_2}} f(r) + \frac{1}{2r_2} e^{-a r_2} (\eta^2 \tilde{B}^2 + \eta d \tilde{B}_\sigma^2)\\
    &\leq (1 - \eta c /2 + \sqrt{\eta} L_\sigma A \tfrac{6 \vee (4a)}{1 - e^{-ar_2}}) f(r) + \frac{3}{2r_2} e^{-a r_2} (\eta^2 \tilde{B}^2 + \eta d \tilde{B}_\sigma^2),
\end{align*}
where we used \(A L_\sigma \tfrac{6 \vee (4a)}{1 - e^{-ar_2}} \leq \sqrt{\eta} c/4\).

\section{Proofs for the stability of the noisy gradient estimator}

Using the Wasserstein contraction obtained in the previous section, we will now prove Proposition \ref{prop:time_indep_bounds}.

\pasteenv{proposition}{prop:time_indep_bounds}

The proof of Proposition follows from an application of Proposition \ref{prop:contraction} to the process in \eqref{eq:sgd_gaussian_approx}. Similar to the proof of Proposition \ref{prop:sgd_stability}, we obtain stability estimates by analysing the trajectories \(\theta_k\) and \(\tilde{\theta}_k\) trained on \(S\) and \(S^N\) with coupled minibatch indices. In particular, given a set of minibatch indices \(B \subset [N]\) with \(|B| = N_B\), if we set
\begin{gather*}
    b(\theta) := \E \Big [ \clip_C(G(\theta, (x_i)_{i \in B})) \Big | \theta, B, S \Big ], \qquad \tilde{b}(\theta) := \E \Big [ \clip_C(G(\theta, (\tilde{x}_i)_{i \in B})) \Big | \theta, B, S^N \Big ]\\
    \sigma(\theta) := \sqrt{\eta} \, \Sigma_S(\theta, B)^{1/2}, \qquad \tilde{\sigma}(\theta) := \sqrt{\eta} \, \Sigma_{S^N}(\theta, B)^{1/2},
\end{gather*}
where we use \((\tilde{x}_i)_{i=1}^N\) to denote the dataset \(S^N\) (i.e. \(\tilde{x}_i = x_i\) for all \(i \neq N\) and \(\tilde{x}_N = \tilde{x}\)), then the trajectories \(\theta_k\) and \(\tilde{\theta}_k\) are updated as in \eqref{eq:stoch_proc_1}, \eqref{eq:stoch_proc_2}. Using the shorthand, \(v_{i, j}(\theta) = w_{t_{(i, j)}} \nabla_\theta \|s_\theta(X_{(i, j)}, t_{(i, j)}) - \nabla \log p_{t_{(i, j)}|0}(X_{(i, j)}| x_i)\|^2\), we obtain the bound,
\begin{align*}
    \Sigma_S(\theta, B) &\succcurlyeq \cov \bigg ( \frac{1}{P N_B} \sum_{i \in B} \sum_{j=1}^P \clip_C(v_{i, j}(\theta)) \bigg | \theta, B, S \bigg )\\
    &\succcurlyeq \frac{1}{P} \frac{1}{N_B^2} \sum_{i \in B} \cov \big ( \clip_C(v_{i, j}(\theta)) \big | \theta, B, S \big )\\
    &\succcurlyeq\frac{1}{P N_B} \bar{\Sigma}.
\end{align*}
Therefore, we have \(\sigma(\theta) \succcurlyeq \sqrt{\eta/PN_B} \, \bar{\Sigma}^{1/2} =: G^{1/2}\), and similarly, \(\tilde{\sigma}(\theta) \succcurlyeq G^{1/2}\). The weighted norm \(\|\cdot\|_{G^+}\) satisfies the property,
\begin{equation*}
    \|\theta\|_{G^+} \leq \lambda_{\text{max}}(G^+)^{1/2} \|\theta\| \leq \sqrt{\frac{PN_B}{\eta \lambda_{\text{gap}}}} \|\theta\|
\end{equation*}
Therefore, due to the gradient clipping, we have \(\|b(\theta)\|_{G^+} \leq \sqrt{PN_B/\eta \lambda_{\text{gap}}} C =: B\). Furthermore, by Assumption \ref{ass:covariance_smoothness}, we apply the same argument used in the proof of Proposition \ref{prop:sgd_stability} to obtain
\begin{align*}
    \| b(\theta) - b(\theta')\|_{G^+} \leq (\overline{M}_4 B_\ell C_\tau^{1/2} + \overline{L}_4^2 ) \|\theta - \theta'\|_{G^+}
\end{align*}
so \(L_b = \overline{M}_4 B_\ell C_\tau^{1/2} + \overline{L}_4^2\). To obtain the Lipschitz constant for the volatility matrix, we first obtain,
\begin{align*}
    \sigma(\theta) - \sigma(\theta') &\preccurlyeq \sqrt{\eta} \cov \bigg ( \frac{1}{P N_B} \sum_{i \in B} \sum_{j=1}^P \Big ( (1 \vee (C \|v_{i, j}(\theta)\|^{-1}))v_{i, j}(\theta) - (1 \vee (C \|v_{i, j}(\theta')\|^{-1}))v_{i, j}(\theta') \Big ) \bigg | \theta, B, S \bigg )^{1/2}.
\end{align*}
From this, we deduce,
\begin{align*}
    \|\sigma(\theta) - \sigma(\theta')\|_{op, G^+} &\leq \sqrt{\eta} \sup_{\|v\|_{G^+}=1} \var \Big ( \Big \langle G^+ v, \frac{1}{P N_B} \sum_{i \in B} \sum_{j=1}^P \Big ( (1 \vee (C \|v_{i, j}(\theta)\|^{-1}))v_{i, j}(\theta) \\
    & \qquad \qquad - (1 \vee (C \|v_{i, j}(\theta')\|^{-1}))v_{i, j}(\theta') \Big ) \Big \rangle \Big | \theta, B, S \bigg )^{1/2}\\
    &\leq \sqrt{\eta} \bigg ( \frac{1}{P N_B^2} \sum_{i \in B} \var \Big ( \|v_{i, j}(\theta) - v_{i, j}(\theta')\|_{G^+}\Big | \theta, B, S \Big ) \bigg )^{1/2}.
\end{align*}
To control this further, we use the Lipschitz assumption on to show that \(\upsilon\) is Lipschitz also:
\begin{align*}
    \|v_{i, j}(\theta)& - v_{i, j}(\theta')\|_{G^+}\\
    &\leq 2 \|s_\theta(X_{(i, j)}, t_{(i, j)}) - s_{\theta'}(X_{(i, j)}, t_{(i, j)})\|_{G^+} \| \nabla_\theta s_\theta(X_{(i, j)}, t_{(i, j)})\|_{op, G^+}\\
    & \qquad + 2 \|s_\theta(X_{(i, j)}, t_{(i, j)}) - \nabla \log p_{t_{(i, j)}|0}(X_{(i, j)}| x_i)\|_{G^+}  \| \nabla_\theta s_\theta(X_{(i, j)}, t_{(i, j)}) - \nabla_\theta s_{\theta'}(X_{(i, j)}, t_{(i, j)})\|_{op, G^+}\\
    &\leq 2 L(X_{(i, j)}, t_{(i, j)})^2 \|\theta - \theta'\|_{G^+} + \frac{2 c^{1/2} \mu_t}{\sigma_t^2} M(X_{(i, j)}, t_{(i, j)}) \|\theta - \theta'\|_{G^+}.
\end{align*}
Computing the variance of this leads to the bound,
\begin{align*}
    \|\sigma(\theta) - \sigma(\theta')\|_{op, G^+} \leq 2 \sqrt{\frac{\eta}{P N_B \lambda_{\text{gap}}}} (\overline{M}_4 B_\ell C_\tau^{1/2} + \overline{L}_4^2) \|\theta - \theta'\|_{G^+} =: L_\sigma \|\theta - \theta'\|_{G^+}.
\end{align*}
Next, we use a similar argument to the proof of Proposition \ref{prop:sgd_stability} to obtain
\begin{equation*}
    \|b(\theta) - \tilde{b}(\theta)\|_{G^+} \leq \sqrt{\frac{PN_B}{\eta \lambda_{\text{gap}}}} \|b(\theta) - \tilde{b}(\theta)\| \leq \sqrt{\frac{PN_B}{\eta \lambda_{\text{gap}}}} \frac{2C}{N_B} \mathbbm{1}_{N \in B} =: \tilde{B}_b.
\end{equation*}

\begin{align*}
    \|\sigma(\theta) - \tilde{\sigma}(\theta')\|_{op, G^+} &\leq \sqrt{\eta} \bigg ( \frac{1}{P N_B^2} \sum_{i \in B} \var \Big ( \|\clip_C(v_{i, j}(\theta)) - \clip_C(v_{i, j}(\theta'))\|_{G^+}\Big | \theta, B, S \Big ) \bigg )^{1/2}\\
    &\leq \sqrt{\frac{\eta}{P N_B^2}} \var \Big ( \|\clip_C(v_{N, j}(\theta)) - \clip_C(v_{N, j}(\theta'))\|_{G^+}\Big | \theta, B, S \Big )^{1/2} \mathbbm{1}_{N \in B}
\end{align*}
Since \(0 \leq \|\clip_C(v_{i, N}(\theta)) - \clip_C(\tilde{v}_{i, N}(\theta'))\|_{G^+} \leq 2C \sqrt{\frac{PN_B}{\eta \lambda_{\text{gap}}}}\)
\begin{align*}
    \|\sigma(\theta) - \tilde{\sigma}(\theta')\|_{op, G^+} &\leq \sqrt{\frac{\eta}{P N_B^2}} \sqrt{\frac{PN_B}{\eta \lambda_{\text{gap}}}} C \mathbbm{1}_{N \in B}\\
    &\leq \sqrt{\frac{1}{N_B \lambda_{\text{gap}}}} C \mathbbm{1}_{N \in B}\\
    &=: \tilde{B}_\sigma.
\end{align*}
Therefore we have satisfied all assumptions of Proposition \ref{prop:contraction} aside from Assumption \ref{ass:contraction_l_sigma}. To satisfy this assumption we use that \(L_\sigma \sim \sqrt{\eta/P}\) and so if \(\eta\) is sufficiently small, or \(P\) is sufficiently large, this assumption is satisfied once \(n\) is sufficiently large also.

Using Proposition \ref{prop:contraction}, we obtain the contraction,
\begin{align*}
    \E[d(\theta_{k+1}, \tilde{\theta}_{k+1})|\theta_k, \tilde{\theta}_k, B_k] \leq (1 - \eta c/2) d(\theta_k, \tilde{\theta}_k) + \frac{3}{2 r_2} e^{-ar_2} \bigg ( \eta \frac{4 P C^2}{\lambda_{\text{gap}} N_B} \mathbbm{1}_{N \in B} + \frac{\eta n}{N_B \lambda_{\text{gap}}} C^2 \mathbbm{1}_{N \in B} \bigg ).
\end{align*}
Using the fact that \(\prob(N \in B_k) = N_B/N\), we obtain,
\begin{align*}
    \E[d(\theta_{k+1}, \tilde{\theta}_{k+1})] \leq (1 - \eta c/2) \E[d(\theta_k, \tilde{\theta}_k)] + \eta \frac{3}{2 r_2} e^{-ar_2} \bigg ( \frac{4 P C^2}{\lambda_{\text{gap}}} + \frac{n}{\lambda_{\text{gap}}} C^2 \bigg ) \frac{1}{N}.
\end{align*}
Thus, by comparison, we obtain the bound,
\begin{align*}
    \E[d(\theta_K, \tilde{\theta}_K)] &\leq \frac{3}{2 r_2} e^{-ar_2} \bigg ( \frac{4 P C^2}{\lambda_{\text{gap}}} + \frac{n}{\lambda_{\text{gap}}} C^2 \bigg ) \frac{1}{N} \eta \sum_{k=0}^{K-1} (1 - \eta c/2)^k\\
    &= \frac{3}{2 r_2} e^{-ar_2} \bigg ( \frac{4 P C^2}{\lambda_{\text{gap}}} + \frac{n}{\lambda_{\text{gap}}} C^2 \bigg ) \frac{1 - (1 - \eta c/2)^K}{N c / 2}\\
    &\leq \frac{3}{2 r_2} e^{-ar_2} (4 P + n) \frac{C^2}{\lambda_{\text{gap}}N} (\eta K \wedge 2/c).
\end{align*}
By the definition of \(f(r)\), we have that it dominates \(r^2\) up to a multiplicative constant:
\begin{align*}
    f(r) &\geq \bigg ( \bigg ( \frac{1}{a} (1 - e^{-ar_2}) \bigg ) \wedge \bigg ( \frac{1}{2 r_2} e^{-ar_2} \bigg ) \bigg ) r^2\\
    &\geq \frac{1}{2 r_2} e^{-ar_2} \bigg ( \bigg ( \frac{2 r_2}{a} (e^{ar_2} - 1) \bigg ) \wedge 1 \bigg ) r^2\\
    &\geq \frac{1}{2 r_2} e^{-ar_2} ( (2 r_2^2) \wedge 1 ) r^2.
\end{align*}
Thus, using assumption \ref{ass:covariance_smoothness}, it follows that
\begin{align*}
    \int \E[\|s_{\theta_K}(X_t, t) - s_{\tilde{\theta}_K}(X_t, t) \|^2|X_0=\tilde{x}, S] \tau(dt) &\leq \bar{L}^2 \E \Big [ \| \theta_K -\tilde{\theta}_K\|_{G^+}^2 \Big ]\\
    &\leq \bar{L}^2 \bigg ( \frac{1}{2 r_2} e^{-ar_2} ( (2 r_2^2) \wedge 1 ) \bigg )^{-1} \E \Big [ d(\theta_K, \tilde{\theta}_K) \Big ]\\
    &\leq 3 \bar{L}^2 ((2 r_2^2)^{-1} \vee 1) (4 P + n) \frac{C^2}{\lambda_{\text{gap}}N} (\eta K \wedge 2/c).
\end{align*}
We then use the fact that when \(\eta\) is sufficiently small, we obtain the estimate \(\eta_0 \gtrsim \lambda^{-1}\) and therefore,
\begin{equation*}
    r_1^2, r_2^2 \gtrsim \frac{P N_B C}{\eta \lambda_{\text{gap}} \lambda^2}, \qquad L \lesssim (\overline{M}_4 B_\ell C_\tau^{1/2} + \overline{L}_4^2) (P N_B \lambda_{\text{gap}})^{-1/2} \vee 1
\end{equation*}
and since \(L\) and \(r_1\) explode as \(\eta \to 0^+\), we also have,
\begin{align*}
    r_2^2 c &\gtrsim L^2 r_1^4 \exp(-6 L r_1^2 / c_0)\\
    &\gtrsim \exp(- 6 L r_1^2 / c_0).
\end{align*}

\end{document}